\documentclass{article}

\usepackage[round]{natbib}
\usepackage[utf8]{inputenc}
\usepackage{amsmath}
\usepackage{amssymb}
\usepackage{amsthm}
\usepackage{amsfonts}
\usepackage{dsfont}
\usepackage{bbm}
\usepackage{bm}
\usepackage{breqn}
\usepackage{mathtools}
\usepackage{fancyhdr}
\usepackage{graphicx}
\usepackage{caption}
\usepackage{subcaption}
\usepackage{placeins}
\usepackage{float}
\usepackage[margin=0.88in]{geometry}
\usepackage{parskip}
\usepackage{natbib}
\usepackage{tabularx}
\usepackage{booktabs}
\usepackage{placeins}
\usepackage{makecell}
\usepackage{algorithm}
\usepackage{algpseudocode}
\usepackage{natbib}
\usepackage{hyperref}

\usepackage[draft,inline,nomargin,index]{fixme}
\fxsetup{theme=color,mode=multiuser}
\FXRegisterAuthor{jg}{ajg}{\color{blue} JG}
\FXRegisterAuthor{hr}{ahr}{\color{red} HR}
\FXRegisterAuthor{sl}{asl}{\color{green} SL}

\newcommand{\ind}{\perp\!\!\!\!\perp} 
\newcommand{\zerone}{\omega}
\renewcommand{\cal}[1]{\mathcal{#1}}
\newcommand{\diff}{\mathrm{d}}
\newcommand{\one}{\mathds{1}}
\newcommand{\E}{\mathbb{E}}
\newcommand{\var}{\text{Var}}

\newcommand{\R}{\mathbb{R}}

\newcommand{\argmin}[1]{\underset{#1}{\text{argmin}}~}
\newcommand{\argmax}[1]{\underset{#1}{\text{argmax}}~}

\newcommand{\half}{\frac{1}{2}}
\newcommand{\prob}{\mathbb{P}}

\newcommand{\N}{\mathbb{N}}
\newcommand{\eps}{\varepsilon}
\newcommand{\NaN}{\varnothing}

\newcommand{\KL}{\mathrm{KL}}

\newcommand{\rsamp}{D}
\newcommand{\rsampm}{D'}
\newcommand{\rsampzero}{D_0}
\newcommand{\rsampone}{D_1}
\newcommand{\rsampmzero}{D'_{0}}
\newcommand{\rsampmone}{D'_{1}}
\newcommand{\effectiveSampleSize}{m_{\mathrm{eff}}}
\newcommand{\rspace}{\mathcal{G}}
\newcommand{\rrspace}{\bar{\mathcal{G}}}
\newcommand{\lrbr}[1]{\left(#1\right)}
\newcommand{\lrbrc}[1]{\left\{#1\right\}}

\newtheorem{theorem}{Theorem}

\newtheorem{remark}{Remark}
\newtheorem{lemma}[theorem]{Lemma}

\newtheorem{corollary}[theorem]{Corollary}

\newtheorem{prop}[theorem]{Proposition}

\title{Density Ratio Estimation and Neyman Pearson Classification with Missing Data}
\author{Josh Givens \and Song Liu \\ University of Bristol \and  Henry W J Reeve}
\date{}
\begin{document}

\maketitle

\begin{abstract}
Density Ratio Estimation (DRE) is an important machine learning technique with many downstream applications. We consider the challenge of DRE with missing not at random (MNAR) data. In this setting, we show that using standard DRE methods leads to biased results while our proposal (M-KLIEP), an adaptation of the popular DRE procedure KLIEP, restores consistency. Moreover, we provide finite sample estimation error bounds for M-KLIEP, which demonstrate minimax optimality with respect to both sample size and worst-case missingness. We then adapt an important downstream application of DRE, Neyman-Pearson (NP) classification, to this MNAR setting. Our procedure both controls Type I error and achieves high power, with high probability.
Finally, we demonstrate promising empirical performance both synthetic data and real-world data with simulated missingness.
\end{abstract}
\section{INTRODUCTION}
Density Ratio estimation (DRE) is the problem of estimating the ratio between two probability density functions (PDFs). 
DRE's ability to characterise the relationship between two PDFs naturally lends itself to many applications such as outlier detection \citep{Azmandian2012}, Generative Adversarial Networks (GANs) \citep{Nowozin2016}, and general binary classification. One issue with DRE is that it is especially sensitive to missing data due to the large impact a ``few key points" can have on the procedure \citep{Liu2017}. 
While there are a vast number of DRE approaches  \citep{Sugiyama2008,Sugiyama2010a,Kanamori2009,Choi2021}, none of them explicitly account for the case of missing data. Some work has been done regarding the impact of missing data on DRE \citep{Josse2019}, and while this work does explore a wide variety of missing data methods it does so while assuming the data to be missing at random (MAR). 
There are many applications in which such an assumption is unrealistic and the probability of an observation being missing depends in some way on the value of the observation itself leading to missing not at random data (MNAR). 
For example, many measuring instruments are more likely to err when attempting to measure more extreme values, while in questionnaires, participants are less likely to answer a question if they deem their answer to be embarrassing or unfavourable. 
In such a case, naively applying any classic DRE procedure leads to inconsistent estimations. 

Classification is an area of statistic with strong ties to density ratio estimation and is also an area where missing covariates are a common issue. 
In this paper, we focus on the links of DRE to a particular type of classification problem called Neyman Pearson (NP) classification.
NP classification allows the user to construct a classifier with a fixed Type I error rate with high probability \citep{Tong2013, Tong18}. 
This type of classification has many natural applications as there are often cases where missclassification in one direction is far more damaging than in the other (for example, disease diagnosis or fault detection). We would like to choose our classifier to control the error in said direction. 
Taking missing data into consideration during NP classification is vital.
Without it, our classifier may no longer satisfy our Type I error constraint with high probability.  
Classic NP classification works in two steps. First, we estimate the class density ratio between the two classes (hence motivating the use of DRE), then we define the classification boundary of our classifier to be a level set of our density ratio estimate, which leads to the satisfaction of our Type I error constraint with high probability \citep{Tong2013}. 
To our knowledge, no previous work has been done to make this second stage of NP classification robust to MNAR data. 

There is a large body of research into the phenomenon of MNAR data in more general cases \citep{Sportisse2020,YLiu2015, Li2013, Seaman2011}, however none of this work has explored its use within DRE specifically. Work exists on MNAR data in the covariates of logistic regression \citep{Ward2020}, a classification method with close ties to DRE. However, \citeauthor{Ward2020}'s approach requires more direct knowledge of the missing data distribution rather than just the probability of an observation being missing conditional on its value.


Our main contributions are:
\begin{itemize}
    \item  We adapt KLIEP \citep{Sugiyama2008}, a classic DRE procedure, to the MNAR setting using an inverse probability weighting approach before expanding this technique to any DRE whose objectives comprise expectations of two different data sets. We give finite sample convergence results for our method and show minimax rate optimality of this convergence when the conditional probability of being missing is known.
    \item  We provide an adaptation to the NP classification procedure, allowing it to control Type I error even with MNAR data and show finite sample power convergence results for our adaptation when the conditional probability of being missing is known.
\end{itemize}
We also provide extensions to our DRE approaches, which allow us to deal with both partial missingness across multiple dimensions and learning the missing pattern by querying a subset of the missing samples. Finally, we assess and compare the efficacy of our proposed methods by testing them on various simulated and real-world data with synthetic missingness.   
\section{MISSING DATA \& DENSITY RATIO ESTIMATION}
\subsection{ Motivation}
 Density Ratio Estimation (DRE) is a versatile tool with many downstream applications such as binary classification (including NP classification, which we will discuss in Section \ref{sec:NP_class}), GANs \citep{Nowozin2016,Ding2020}, and covariate shift \citep{Sugiyama2008,Tsuboi2009}.  
In DRE problems, we want to estimate $r^*\coloneqq p_1/p_0$, with $p_1,p_0$ the PDFs for two distributions. DRE approaches estimate $r^*$ using IID samples from these two distributions. Existing DRE approaches are designed to handle fully observed data sets. 
However, when observations are missing,
these approaches do not apply straightforwardly.

DRE requires two sets of observations.
When two samples are collected at different times under different contexts, they are especially susceptible to different missing patterns. An example is binary classification for diagnosing an illness. In this setting, it is natural to expect that samples collected from healthy individuals and patients will have different MNAR patterns due to privacy concerns. 

Within this setting, using both data sets by simply discarding all missing values will lead to a biased classification and degraded performance.

We aim to make DRE robust to Missing Not at Random (MNAR) datasets. In this setting, an observation's probability of being missing depends upon that observation's true value and the MNAR pattern is different between two samples. We formally introduce our problem set-up below.

\subsection{Missing Data Framework}\label{sec:notation}
Let $Z^1,Z^0$ be two RVs taking values in measurable space $(\cal Z,\cal B_{\cal Z})$ with densities (Radon Nikodym derivatives) $p_{1}$ and $p_{0}$ respectively w.r.t. (with respect to) a measure $\mu$ on $(\cal Z,\cal B_{\cal Z})$ and assume that $p_0$ is strictly positive. 
For $\zerone\in\{0,1\}$, let $X^\zerone$ be a RV taking values in the measurable space $(\cal X,\cal B_{\cal X})$ with $\cal X\coloneqq \cal Z\cup\{\NaN\},~\cal B_{\cal X}\coloneqq\sigma(\cal B_{\cal Z}\cup\{\{\NaN\}\})$ whose distribution is uniquely defined as follows:
\begin{align*}
    \prob(X^\zerone=\NaN|Z^\zerone)&\coloneqq\varphi^\zerone(Z)\\
    \prob(X^\zerone=x|Z^\zerone)&\coloneqq\one\{Z^\zerone=x\}(1-\varphi^\zerone(Z^\zerone)),
\end{align*}
where $\varphi^\zerone:\cal Z\rightarrow [0,1)$ is a measurable function. Here we take $X=\NaN$ to represent the observation being missing and thus $\varphi(z)$ represents the probability of an observation being missing given its ``true" value is $z$. Additionally, let $p'_\zerone$ be the density of $X^\zerone|X^\zerone\neq\NaN$ w.r.t $\mu$. Throughout, unless stated otherwise, we assume $\varphi$ to be known.

For $\zerone\in\{0,1\}$ and $n_\zerone\in\mathbb{N}$ we define $\{Z_i^{\zerone}\}_{i=1}^{n_\zerone}$ and $\{X_i^{\zerone}\}_{i=1}^{n_\zerone}$ to be IID copies of $Z^\zerone,X^\zerone$ respectively and abbreviate them by $\rsamp_{\zerone}\coloneqq\{Z_i^\zerone\}_{i=1}^{n_{\zerone}}$,  $\rsampm_{\zerone}\coloneqq\{X_i^\zerone\}_{i=1}^{n_{\zerone}}$. Finally we define,
$\rsamp\coloneqq(\rsampone,\rsampzero)$, $\rsampm\coloneqq(\rsampmone,\rsampmzero)$.

\subsection{Notations}
Throughout we shall adopt the following two conventions for the purpose of brevity. Firstly, if $X$, $Z$, or any other class-specific element is given without a sub/superscript specifying the class, the associated statement is assumed to hold for both classes where all the elements within the statement have the same class. Secondly, given a function $h:\cal Z\rightarrow\R$ we implicitly extend $h$ to a function $h: \mathcal{X} \rightarrow \R$ by taking $h(\NaN)\coloneqq 0$, unless stated otherwise.

We now introduce some additional notation. For $n\in\mathbb{N}$ we let $[n]\coloneqq\{1,\dotsc,n\}$. For a function $h:\cal Z\rightarrow\R$ we take $\|h\|_{\infty}\coloneqq\sup_{z\in\cal Z}|h(z)|$.  Finally, for a RV $W$ in a product space $\cal W^d$ and any $j\in[d]$ define $W^{(j)}$ to be the $j$\textsuperscript{th} coordinate of $W$ and $W^{(-j)}=\{W^{(j')}\}_{j'\neq j}$. The Kullback-Leibler (KL) between two density functions with $p_1$ absolutely continuous w.r.t. $p_0$ is defined as 
\begin{align*}
    \KL(p_1| p_0) := \int_{\mathcal{Z}} p_1(z) \log \frac{p_1(z)}{p_0(z)} \mu(\mathrm{d}z).
\end{align*}

\subsection{DRE}
The aim of DRE is to estimate $r^*:\cal Z\rightarrow [0,\infty)$ defined by
$
    r^*\coloneqq p_1/p_0.
$
We will let  $\cal H$ be the set of non-negative measurable functions from $\cal Z$ to $[0,\infty)$ and let $\cal G$ be some subset of $\cal H$ which we intend to use in our approximation of $r^*$. Note that while $r^*\in\cal H$, $\cal G$ may or may not contain the ``true" density ratio $r^*$. We say that $\cal G$ is correctly specified if $r^*\in\cal G$ and incorrectly specified otherwise. 
We now introduce a classic approach for estimating $r^*$ using $\rsamp$.
\subsubsection{Kullback-Leibler Importance Estimation Procedure}
First proposed in \citep{Sugiyama2008}, Kullback-Leibler Importance Estimation Procedure (KLIEP) is a popular DRE procedure that minimises the KL divergence between $p_1$ and $r \cdot p_0$. Specifically, it aims to solve the following optimisation problem:
\begin{align}\label{eq:kliep_obj}
    \tilde{r} := \arg\min_{r\in\cal G}&~\KL(p_1|r \cdot p_0) \\
    \text{subject to}:& \int_{\cal Z} r(z) p_0(z)\mu(\diff z)=1. \nonumber
\end{align}
As KL-divergence is strictly for use with probability densities, we need to include the additional constraint that $r\cdot p_0$ integrates out to $1$ over $\cal Z$. 

Provided $\cal G$ is closed under positive scalar multiplication, we can re-write $\tilde r$ in terms of an unconstrained optimisation problem. Specifically, we have $\tilde r\coloneqq N^{-1}\cdot r_0$ where
\begin{align}
\label{eq.kliep.obj}
r_0&\coloneqq\argmax{r\in\rrspace}\E[\log r(Z^1)]-\log\E[r(Z^0)]\\
    N&\coloneqq\E[r_0(Z^0)] .\notag
\end{align}
for some select $\rrspace \subseteq \rspace = \{ a \cdot g : a \in (0,\infty),~ g \in \rrspace\}$. More details on this alongside the lemma that makes this possible can be found in the Appendix \ref{app:KLIEP_rewrite}. Throughout we will be mostly interested in estimating $r_0$ and as such will refer to \ref{eq.kliep.obj} as the population KLIEP objective.
As we do not know the true distributions of $Z^1, Z^0$, we must approximate these expectations with samples from corresponding distributions giving us the sample KLIEP objective.

In the case of multi-dimensional real-valued data, $\cal Z=\R^p$, 
we can take $\rrspace=\{r_{\theta}:\R^p\rightarrow[0,\infty)|\theta\in\R^d\}$ where
\begin{align}\label{eq:KLIEP_paramform}
    r_{\theta}(z)\coloneqq \exp(\theta^\top f(z))
\end{align} for some $f:\R^p\rightarrow\R^d$ \citep{Tsuboi2009, Kanamori2010, Liu2017}. We refer to this form as the \emph{log-linear} form.
Under this form, our solution becomes $\hat r=\hat N^{-1}r_{\hat \theta}$ where
\begin{align}\label{eq:kliep_param_obj}
\begin{split}
    \hat\theta\coloneqq&\argmax{\theta\in\R^d}\frac{1}{n_1}\sum_{i=1}^{n_1}\theta^\top f(Z_i^1)-\log\frac{1}{n_0}\sum_{i=1}^{n_0}\exp(\theta^\top f(Z_i^0))
\end{split}\\
    \hat N\coloneqq&\frac{1}{n_0}\sum_{i=1}^{n_0}\exp(\theta^\top f(Z_i^0))\nonumber
\end{align}
which gives us a convex optimisation problem.

We now define $\tilde{\theta}$ as the minimiser of the population KLIEP objective under our log-linear model (i.e. the population analogue of (\ref{eq:kliep_param_obj})) allowing us to view $\hat\theta$ as an estimator for $\tilde\theta$. Indeed, with some mild restrictions on $f$ and for $n_\min:=\min\{n_0,n_1\}\geq C_0\log(1/\delta)$, we have that w.p. $1-\delta$ 
\begin{align}\label{eq:KLIEP_bound}
    \|\hat\theta-\tilde{\theta}\|\leq\sqrt{\frac{C_0\log(1/\delta)}{n_\min}}
\end{align}
where  $C_0$ is a constant depending on $f(Z_0),f(Z_1),\tilde{\theta},d$. Details of this result are found in Appendix \ref{app:KLIEP_bound}.

An additional bound for the accuracy of $\hat N$ is given in the Appendix section \ref{app:KLIEP_N_bound}. However, as we will see in applications such as NP classification, we only require the knowledge of $r^*$ up to strictly increasing transform and so it is often unnecessary to estimate $N$.

Within this paper, we mostly focus on KLIEP however for completeness we also present the general class of $f$-Divergence based density ratio estimators and describe some notable cases in Appendix \ref{app:Other_DRE}.
We now go on to adapt this procedure for use with MNAR data.
\section{PROPOSED METHODS}
In our setting we do not observe samples from $Z$ and so cannot use (\ref{eq:kliep_param_obj}) to approximate the population KLIEP objective. Instead, we only have samples from $X$. The following result relates expectations of $Z$ with expectations of $X$. 

\begin{lemma}\label{lemma:main_prob}
We have that $p'=C\cdot(1-\varphi)\cdot p$, for some $C\in\R$. Hence, for any measurable function $g:\cal Z\rightarrow \R$
\begin{align} \label{eq:imp_weight_exp}
    \E[g(Z)]=\E\left[\frac{\one\{X\neq\NaN\}}{1-\varphi(X)}g(X)\right].
\end{align}
\end{lemma}
For proof of this Lemma see Appendix \ref{app:proof_prob}). This importance weighting technique is an approach that is already used to tackle MNAR data contexts outside of DRE \citep{Li2013,Seaman2011}.
From this Lemma, we get that for some $C'\in\R$,
\begin{align*}
    r'\coloneqq \frac{p'_1}{p'_0}=C'\cdot \frac{1-\varphi^1}{1-\varphi^0}\cdot r^*
\end{align*} 
Consequently, $r'$ is not proportional to $r^*$ when $\varphi^1\neq\varphi^0$ with either non-constant. As a result, using KLIEP with only the observed values of our MNAR data will be biased due to the fact it estimates $r'$ rather than $r^*$.
We now use this Lemma to inform our adapted estimation procedures.
\subsection{KLIEP with Missing Data}
We can now approximate the population KLIEP objective by replacing expectations in \eqref{eq.kliep.obj} with the important weighted expectations in Lemma \ref{lemma:main_prob}. This replacement modifies the sample KLIEP objective to work with $\rsampm$ as follows: 
\begin{align}
\label{eq.kliep.missing}
    \begin{split}
        \hat r_0 \coloneqq&{}\argmax{r\in\rrspace}\frac{1}{n_1}\sum_{i=1}^{n_1}\frac{\one\{X_i^1\neq\NaN\}}{1-\varphi^1(X_i^1)}\log r(X_i^1)~-~\log\frac{1}{n_0}\sum_{i=1}^{n_0}\frac{\one\{X_i^0\neq\NaN\}}{1-\varphi^0(X_i^0)}r(X_i^0)
    \end{split}\\
    \hat N\coloneqq &{}\frac{1}{n_0}\sum_{i=1}^n\frac{\one\{X_i^0\neq\NaN\}}{1-\varphi^0(X_i^0)} \hat r_0(X_i^0)~. \nonumber
\end{align}

We refer to this approach as \emph{M-KLIEP} and to (\ref{eq.kliep.missing}) as the \emph{sample M-KLIEP objective}. We now take,
$\hat\theta'$ to be the estimated parameter when $r$ takes the log-linear form \eqref{eq:KLIEP_paramform}. 
To show the efficacy of M-KLIEP, we prove that 
$\hat\theta'$, converges to $\tilde\theta$ under mild conditions. Let $\effectiveSampleSize\equiv \effectiveSampleSize(n_0,n_1,\varphi^0,\varphi^1)$ denote the \emph{effective sample size} defined by 
\begin{align*}
\effectiveSampleSize:= \min\{n_0 \cdot (1-\|\varphi^0\|_\infty), n_1\cdot (1-\|\varphi^1\|_\infty) \}.
\end{align*}

\begin{theorem}\label{thm:MKLIEP_bound}
Let $\hat\theta'$ be defined as above. Assume that 
  $\|f\|_\infty\leq \infty$ and  $\sigma_\min :=\var(f(Z^0))>0$. Then, there exists some constant $C'_0 \geq 1$, depending only on $\|f\|_{\infty},~\sigma_\min,~\|\tilde{\theta}\|$, and $d$, such that for any $\delta \in (0,1/2]$ and $\effectiveSampleSize(n_0,n_1,\varphi^0,\varphi^1)>C'_0 \cdot \log(1/\delta)$ we have
\begin{align*}
\prob\left(\|\hat\theta-\tilde{\theta}\| >   \sqrt{\frac{C'_0\log(1/\delta)}{\effectiveSampleSize(n_0,n_1,\varphi^0,\varphi^1)}}\right) \leq \delta.
\end{align*}
\end{theorem}
Proof is given in Appendix \ref{app:MKLIEP_bound}.
This result shows that M-KLIEP, unlike applying KLIEP on only observed values, does not suffer from the inconsistent ratio estimation problem we mentioned earlier.  

The following result shows that this bound is minimax optimal w.r.t $\effectiveSampleSize$.

\begin{theorem}\label{thm:drMissLB} Given any estimator $\hat{\theta} \equiv \hat{\theta}(\rsampm)$ of $\tilde\theta$, and any $(w_0,w_1) \in [0,1]^2$,  $(n_0,n_1) \in \N^2$ and $\delta \in (0,1/4]$ there exists distributions $P_0$, $P_1$ on $\mathcal{X}$ which satisfies the conditions of Theorem \ref{thm:MKLIEP_bound}, with $\|\varphi^\omega\|_\infty \leq w_\omega$ such that if $\rsampm_\omega \sim P_\omega^{n_\omega}$ for $\omega \in \{0,1\}$ then 
\begin{align*}
\prob\left(\|\hat{\theta}-\theta\|> \frac{1}{2} \wedge \sqrt{\frac{10\log(1/(4\delta))}{\effectiveSampleSize(n_0,n_1,\varphi^0,\varphi^1)}}  \right) \geq  \delta.
\end{align*}
\end{theorem}


\subsection{General Extension} 
While we have described the adaptation for KLIEP, this approach can be applied to any expectation-based DRE procedure with the following form:
\begin{align*}
    \argmin{r\in\cal G}h(\E[g_1(Z^1)],\E[g_0(Z^0)])
\end{align*}
where $h:\R^{d_0}\times\R^{d_1}\rightarrow\R$, $g_\zerone:\cal Z\rightarrow \R^{d_\zerone}$. We can then approximate these expectations using Lemma \ref{lemma:main_prob} as we did for KLIEP. An example of this with $f$-Divergence based DRE is given in the Appendix section \ref{app:Other_DRE_miss}.

\subsection{Extensions to Partial Missingness across Multiple Dimensions}\label{sec:partial_miss}
We now extend our approaches to the case of partial missingness in multi-dimensional settings. Throughout this section let $d\in\N$ and replace our original space $\cal Z$ for $\cal Z^d$ (and similarly $\cal X$ with $\cal X^d$). We start off with the assumption on our data missingness structure that for all $j\in [d],$
\begin{align*}
    \prob(X^{(j)}&=\NaN|Z,X^{(-j)})=\varphi_j(Z^{(j)})\\
    \prob(X^{(j)}&=x^{(j)}|Z,X^{(-j)})=\one\{Z^{(j)}=x^{(j)}\}(1-\varphi_j(Z^{(j)}))
\end{align*}
As such, the components are missing independently from one another with probabilities only depending on their own true value. To further simplify proceedings, we assume a naive Bayes style condition. Namely we include the restriction that for all $j'\neq j$, $Z^{(j')}\ind Z^{(j)}$ so that $(Z^{(j)},X^{(j)})\perp(Z^{(j')},X^{(j')})$.
This implies that
\begin{align*}
    r^*\coloneqq\prod_{j=1}^d r^*_j\hspace{5mm}\text{ where }\hspace{5mm}
    r^*_j:=\frac{p_1^{(j)}}{p_0^{(j)}}
\end{align*}
with $p^{(j)}$ be the p.d.f. of $Z^{(j)}$.

This allows us to separately estimate the density ratio over each dimension and then take the product to get the joint density ratio (DR). Within each dimension, only data from that dimension is relevant (due to the aforementioned independence) and so we can estimate the density ratio on each dimension using exclusively the data from that dimension with our current methods.

\subsection{NP Classification with Missing data}\label{sec:NP_class}
We now apply our adapted procedures to the problem of NP classification. 
First, we introduce NP classification.

NP classification constructs a classifier that strictly controls the miss-classification of one class while minimising miss-classification in the other \citep{Cannon2002,Scott2005,Tong2013,Tong18}.
Specifically, in NP classification we aim to learn a classifier $\phi:\cal Z\rightarrow \{0,1\}$ which solves the following constrained optimisation problem
\begin{align}
    \min_{\phi}&~ \prob(\phi(Z^1)=0) \notag \\
    \text{subject to:}&~ \prob(\phi(Z^0)=1)\leq\alpha. \label{eq:type1er}
\end{align}
We refer to the classifier which solves the above problem as the \emph{NP oracle classifier} at level $\alpha$ and refer to class $0$ as the \emph{error controlled class}. 
From the Neyman-Pearson lemma, it can be shown that the classification boundary of the oracle classifier is a level set of $r^*$ \citep{Tong2013}.
This motivates the use of DRE to approximate $r^*$ and in turn the oracle classifier. 
In contrast to hypothesis testing, we do not know the distribution of $Z^0$. As such, there is no means of constructing a classifier that is guaranteed to satisfy (\ref{eq:type1er}). 
Instead, we create a classifier $\hat\phi_{\rsamp}$ from our data $\rsamp$. This classifier has a small pre-specified probability of violating our Type I error constraint. 
In other words for a given small $\delta>0$, and classification procedure $\hat\phi_{\rsamp}$,
\begin{align}\label{eq:type1er2}
    \prob\bigg(\prob\Big(\hat\phi_{\rsamp}(Z^0)=1\Big|\rsamp\Big)\leq\alpha\bigg)\geq 1-\delta.
\end{align}

A procedure satisfying (\ref{eq:type1er2}) is described in \cite{Tong18}. This procedure uses any density ratio estimate alongside additional data from the error-controlled class to choose a classification threshold. The estimated density ratio alongside the threshold is used for classification. As a result, different DRE approaches lead to different NP classification procedures. More information on the details of this classification procedure can be found in Appendix \ref{app:np_class}.

As we know the optimal classifier comes from the true density ratio, we can assess the accuracy of various DRE procedures by the efficacy of the associated NP classification procedure. We asses the efficacy of an NP classification procedure $\hat\phi_{\rsamp}$ by the expected power which is defined as
\begin{align*}
    &\E\bigg[\prob\Big(\hat\phi_{\rsamp}(Z^1)=1\Big|\rsamp\Big)\bigg]
    =\prob(\hat\phi_{\rsamp}(Z^1)=1).
\end{align*}

We can also assess the efficacy of a classifier by it's Type II error which we define to be $R_1(\phi)\coloneqq\prob(\phi(Z^0)=1)$.

\subsubsection{Adapting NP Classification to Missing Data}
Current NP classification approaches require direct samples from $Z^0$ in order to select the appropriate classification threshold. Therefore, we will need to adapt it for the MNAR setting. 

We can again use Lemma \ref{lemma:main_prob} to approximate $\prob(\hat r(Z^0)>C)$ using $\rsampm_0$ and choose a threshold based on this. 
Specifically, if for any $h: \cal Z\rightarrow \R$ we define
\begin{align*}
    W_i^{(C,h)}\coloneqq\frac{\one\{X^0_i\neq\NaN\}}{1-\varphi^0(X_i^0)}\one\{h(X^0_i)>C\}
\end{align*}
then $\E[W_i^{(C,h)}]=\prob(h(Z^0)>C)$.  We also introduce the effective sample size for class $0$ defined by $\effectiveSampleSize^0\coloneqq n_0(1-\|\varphi^0\|_{\infty})$. We now use this result to prove the key lemma that informs our adaptation to the threshold selection within our NP classification algorithm.
\begin{lemma}\label{lemma:NP_miss_typeI} For a given measurable $h:\cal Z\rightarrow \R$
let's choose \begin{align*}
\hat{C}_{\alpha,\delta, h}&\equiv \hat{C}_{\alpha,\delta, h}(\rsampm_0)\\
&:=\inf\left\lbrace C \in \R ~:~ \frac{1}{n}\sum_{i \in [n]} W_i^{(C,h)}\leq \alpha - \Delta_{\effectiveSampleSize^0,\delta}\right\rbrace \\
\text{where  }&\Delta_{\effectiveSampleSize^0,\delta}\coloneqq\sqrt{\frac{16\log(1/\delta)}{\effectiveSampleSize^0}}.
\end{align*}

Then, for any $\delta\in(0,1/2]$, 
\begin{align*}
    \mathbb{P}\bigg[\prob(h(Z^0)> \hat C_{\alpha,\delta, h}~|~\rsampm)> \alpha \bigg] \leq \delta.
\end{align*}
\end{lemma}
\begin{proof} 
Fix $C^*_{\alpha,h}:= \inf\left\lbrace C \in \R :\prob(h(Z^0)> C)\leq \alpha\right\rbrace$. As $\prob(h(Z^0)>C)$ is a decreasing right continuous function of $C$ we have that $\prob(h(Z^0)>C^*_{\alpha,h})\leq\alpha$. Hence,
\begin{align*}
&\mathbb{P}\left[\prob(h(Z^0)>\hat{C}_{\alpha,\delta,h}~|~\rsampmzero)> \alpha\right]\\
&=
 \prob(\hat{C}_{\alpha,\delta,h}  < C^*_{\alpha,h}) \\
 & \leq \prob\left(\frac{1}{n}\sum_{i \in [n]} W_i^{(C^*_{\alpha,h})} \leq \alpha - \Delta_{\effectiveSampleSize^0,\delta}\right) \leq\delta,
\end{align*}
where we used Lemma \ref{lemma:imp_weight_vec_bern} from Appendix \ref{app:MKLIEP_bound} in the final step.
\end{proof}

Therefore this gives rise to the following NP classification Algorithm \ref{alg:NP_classif_miss}. We will assume that we have access to an additional $n_0$ IID copies of $X_0$ which we will label as $X^0_{n_0+1},\dotsc,X^0_{2n_0}$ however this is just for notational convenience and the algorithm and associated theory adapts to any number of samples from $X_0$. We now let $g:\R\rightarrow\R$ be any strictly increasing function.
\begin{algorithm}[h]
\caption{Missing NP Classification Procedure}
\label{alg:NP_classif_miss}
\begin{algorithmic}[1]
    \vspace{0.1cm}
    \State Use $\{X_i^1\}_{i=1}^{n_1},\{X_i^0\}_{i=1}^{n_0}$ to estimate $g\circ r^*$ with $g\circ \hat r$ by any DRE procedure.
    \vspace{0.2mm}
    \State For $i\in\{1,\dotsc,n_0\}$ compute $\hat r_i\coloneqq g\circ \hat r(X_{i+n_0}^0)$ and $w_i\coloneqq\one\{X_{i+n_0}^0\in\cal Z\}(1-\varphi^0(X_{i+n_0}^0))^{-1}$ with $g\circ \hat r(\NaN)\coloneqq-\infty$.
    \vspace{0.2mm}
    \State Sort $\hat r_1,\dotsc,\hat r_{n_0}$ in increasing order to get $\hat r_{(1)},\dotsc,\hat r_{(n_0)}$ with $\hat r_{(i)}\leq \hat r_{(i+1)}$ and associated $w_{(1)},\dotsc,w_{(n_0)}$.
    \vspace{0.2cm}
    \State Set $i^*=\min\{i\in\{1,\dotsc,n_0\}|\frac{1}{n_0}\sum_{j=i}^{n_0}w_{(j)}\leq\alpha-\Delta_{\effectiveSampleSize^0,\delta}\}$.
    \vspace{0.2cm}
    \State Let $\hat C_{\alpha,\delta,g\circ \hat r}(\rsampm)\coloneqq \hat r_{(i^*)}$ and define $\hat\phi_{\rsampm}$ by 
    \begin{align*}
        \hat\phi_{\rsampm}(z)\coloneqq
        \one\{g\circ \hat r(z)>\hat C_{\alpha,\delta, g\circ \hat r}\}\quad\text{for all $z\in\cal Z$.}
    \end{align*}
\end{algorithmic}
\end{algorithm}

We can freely introduce $g$ within  Algorithm \ref{alg:NP_classif_miss} as the classifier produced is invariant to strictly increasing transformations. Crucially we can do this without knowing $g$ specifically and are only required to know $g\circ\hat r$. This added flexibility has many advantages, for example, we are now only required to learn $r^*$ up to a multiplicative constant. 

We now analyse the performance of Algorithm \ref{alg:NP_classif_miss} in conjunction with M-KLIEP.
To this end, define $\hat\phi_{\rsampm}$ as in Algorithm \ref{alg:NP_classif_miss} with $\hat r=r_{\hat\theta'}$ as defined in Theorem \ref{thm:MKLIEP_bound}. We now compare this with its population analogue $\tilde\phi$ defined by
$$\tilde\phi(z)\coloneqq\one\{g\circ r_{\tilde\theta}(z)>C^*_{\alpha,g\circ r_{\tilde\theta}}\}$$
for all $z\in\cal Z$,
where $C^*_{\alpha,h}$ is defined as in the proof of Lemma \ref{lemma:NP_miss_typeI}.
\begin{theorem}\label{thm:NP_pow_bound}
Assume the conditions of Theorem \ref{thm:MKLIEP_bound} hold and let $\hat\phi_{\rsampm},\tilde\phi$ be defined as above. 
Further assume that for all $\theta\in\R^d,z\in\cal Z$, $\|\nabla_\theta g \circ r_\theta(z)\|\leq L$.
Finally, suppose there exists positive constants $B_0,B_1,\gamma_0,\gamma_1, a$ s.t. $h$ satisfies the following conditions
\begin{align*}
 \prob\bigg(h(Z^0)\in[C^*_{\alpha,h}, C^*_{\alpha,h}+\zeta]\bigg)&\geq B_0^{-1}\zeta^{\gamma_0}\\
 \prob\bigg(h(Z^0)\in[C^*_{\alpha,h}-\zeta, C^*_{\alpha,h}+\zeta]\bigg)&\leq B_1\zeta^{\gamma_1}, 
\end{align*}
for all $\zeta\in(0,~ B_0 a^{1/\gamma_0}]$. Then there exists a constant  $C_2>0$  depending only on $C'_0, B_0,B_1,\gamma_0,\gamma_1,a$, and $L$ such that for all $\delta\in (0,\half]$, with $\effectiveSampleSize\geq C_2\log(1/\delta)$, we have
\begin{align*}
    \prob(|R_1(\hat\phi_{\rsampm})-R_1(\tilde\phi)|>\eps')&\leq 2\delta.
\end{align*}
where
\begin{align*}
    \eps'\coloneqq C_2\left[\left(\frac{\log(1/\delta)}{\effectiveSampleSize^0}\right)^{\frac{\gamma_1+1}{2\gamma_0}\wedge \half}+\left(\frac{\log(1/\delta)}{\effectiveSampleSize}\right)^{\frac{\gamma_1+1}{2}}\right].
\end{align*}

\end{theorem}
\begin{remark}
If $\cal G$ is correctly specified and $Z^1,Z^0$ are continuous then the NP lemma gives us that $\tilde\phi$ is the oracle classifier \citep{Tong2013}. Hence we have shown convergence in Type II error to the oracle classifier in this case.
\end{remark}
\begin{remark}
Taking $g=\log$ gives $L=b$ under the constraints imposed in Theorem \ref{thm:MKLIEP_bound}.
\end{remark}
We have now fully adapted NP classification to the MNAR setting. Namely, we have proposed an algorithm to construct an NP classifier using MNAR data and shown this classifier simultaneously control Type I error with high probability while also converging to the oracle classifier when $\cal G$ is correctly specified. We test our adaptations on some synthetic examples.
\section{SYNTHETIC DATA EXAMPLES}
Here we empirically evaluate our proposed methods alongside  CC-KLIEP (naive \textbf{C}omplete \textbf{C}ase \textbf{KLIEP}) which simply discards any missing value in the dataset and estimating the ratio  using only observed data.
Both M-KLIEP and CC-KLIEP are tested on simulated datasets. Details on the data generating processes can be found in Appendix \ref{app:synth_exp}.

\subsection{5-dimensional Correctly Specified Example}
In this example we take both classes to have multivariate Gaussian distributions with the same variance. We then induced MNAR missingness in class $0$ and no missingness in class $1$. We use the log-linear form for $r_\theta$ with $f(z)=z$ so that $\cal G$ is correctly specified. 
\sloppy100 simulations of the above data generation procedure were run for $n_0=n_1$ with $n_0$ ranging from $100$ to $1,500$ and $\tilde\theta$ estimated by M-KLIEP and CC-KLIEP. These simulations were then used to estimate the mean square (Euclidean) distance (MSD) between our estimate, $\hat\theta'$, and $\tilde\theta$ for both procedures. These estimates alongside 99\% C.I.s (confidence intervals) were calculated which are presented in Figure \ref{fig:msd_5dim}.
\begin{figure}[ht]
    \centering
    \includegraphics[width=0.45\textwidth]{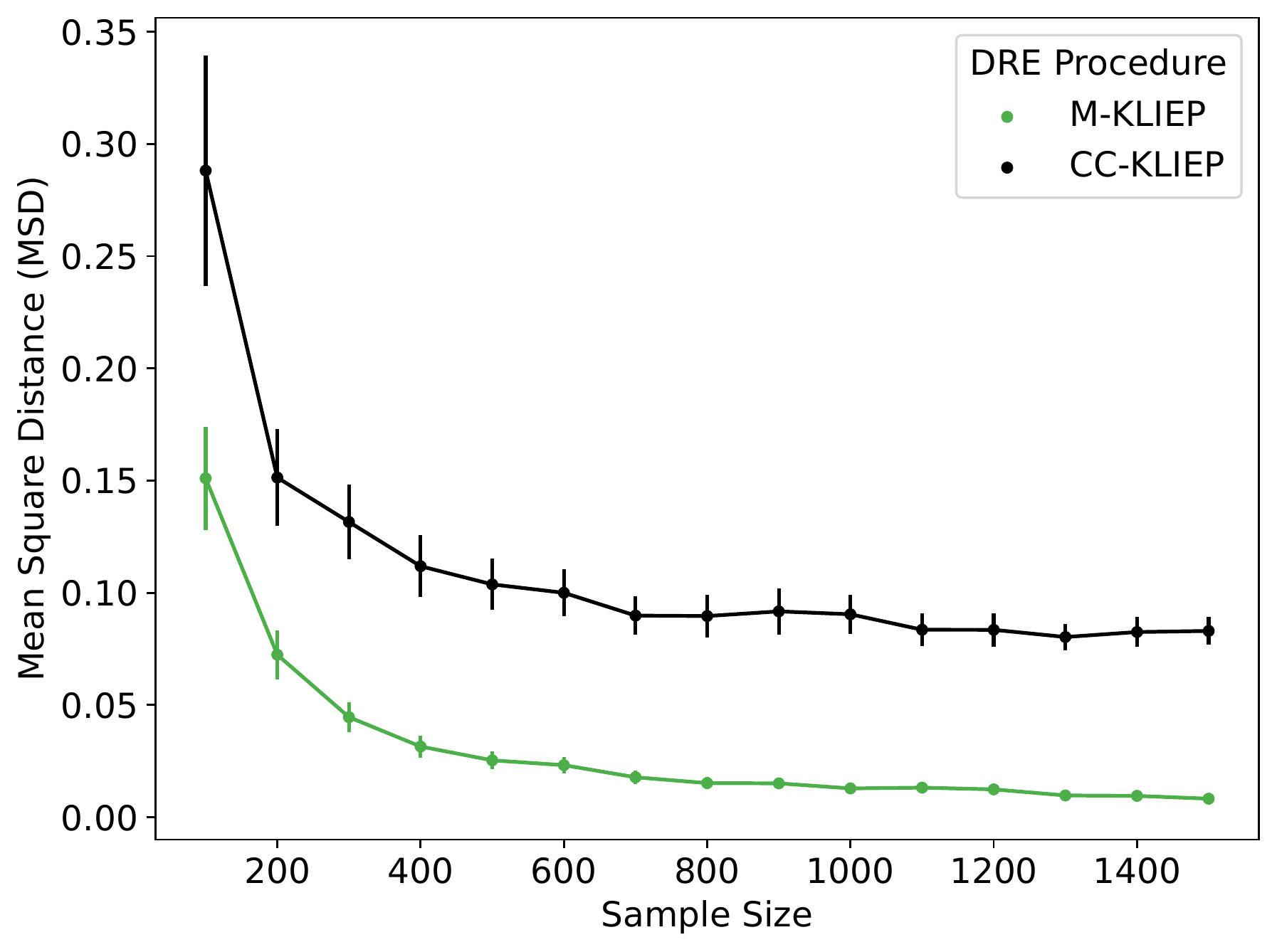}
    \caption{MSD between $\hat\theta'$ and $\tilde\theta$ for varying $n$ with 99\% C.I.s.}
    \label{fig:msd_5dim}
\end{figure}
In the plot we can see clear evidence of the complete case approach being asymptotically biased as the error plateaus around 0.08 while the error under M-KLIEP converges to 0 as $n_0,n_1$ increase.
We now go on to illustrate the affect of our DRE procedure on NP classification.
\subsection{Neyman Pearson Classification}
For this example we take both classes to be 2-dimensional Gaussian mixtures. We then induce MNAR missingness in class $0$ and no missingness in class $1$. M-KLIEP and CC-KLIEP are then used to estimate $r^*$ using the log-linear form with $f(x)=x$ again. NP classifiers were then fit using these estimators alongside the true $r^*$.

\begin{figure}[ht]
     \centering
         \includegraphics[width=0.45\textwidth]{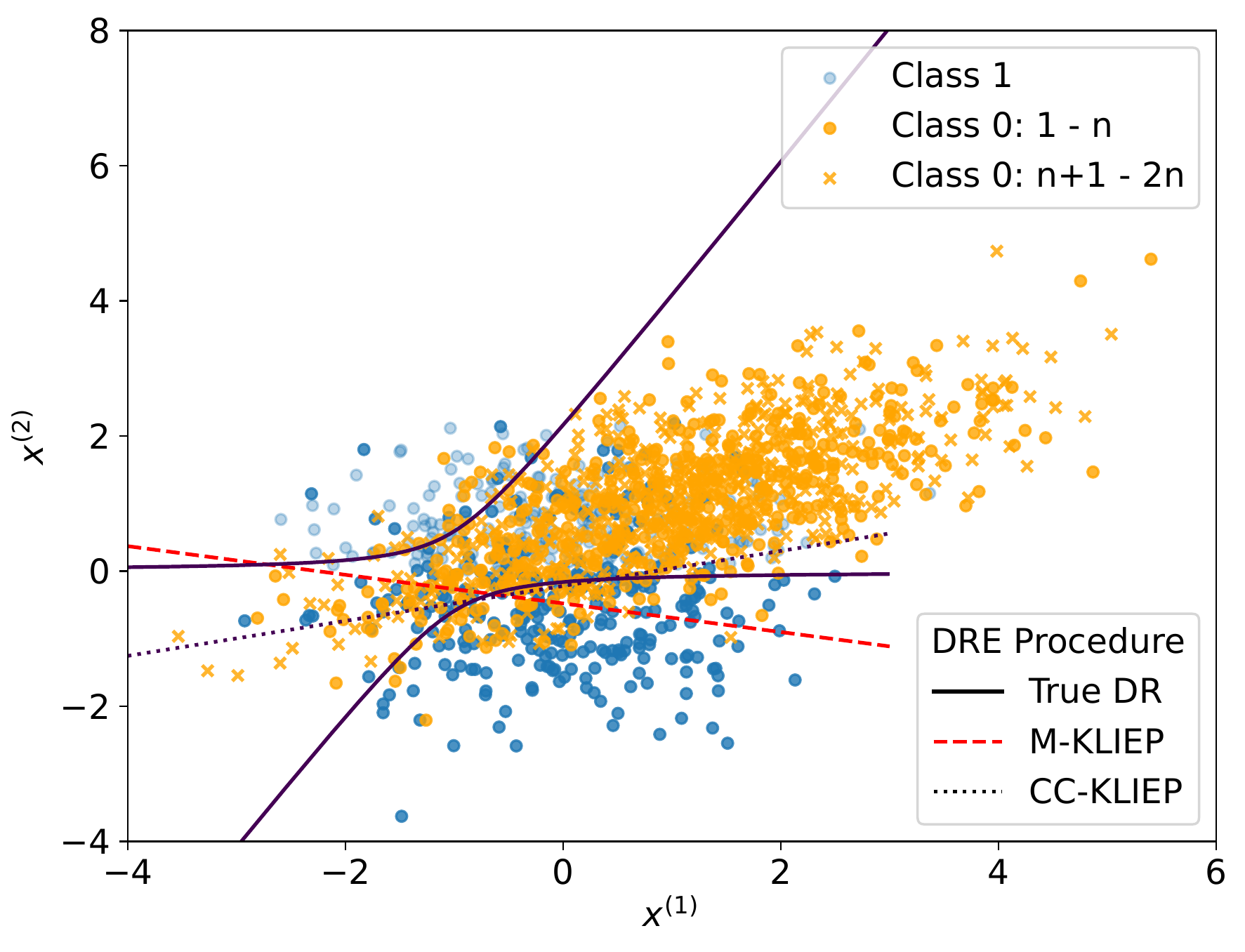}
        \caption{Scatter plot of $\rsamp$ alongside classification boundaries produced from corresponding $\rsampm$ via various procedures. If $X_i=\NaN$ then $Z_i$ is faded out.}
        \label{fig:np_class_bound}
\end{figure}
Figure \ref{fig:np_class_bound} shows one run of this experiment with $n_0=n_1=500$.
We see that M-KLIEP provides a good approximation to the classifier that uses $r^*$ despite the model being incorrectly specified while CC-KLIEP is biased due to the way it discards corrupted points.
\subsubsection*{Repeated Simulations}
We now run 100 simulations of the above experiment for various $n_1=n_0$ ranging from $100$ to $1500$. We then use these to estimate the expected power of the procedures alongside 99\% C.I.s . Figure \ref{fig:np_exp_power} shows the results of this experiment.
\begin{figure}
    \centering
    \includegraphics[width=0.45\textwidth]{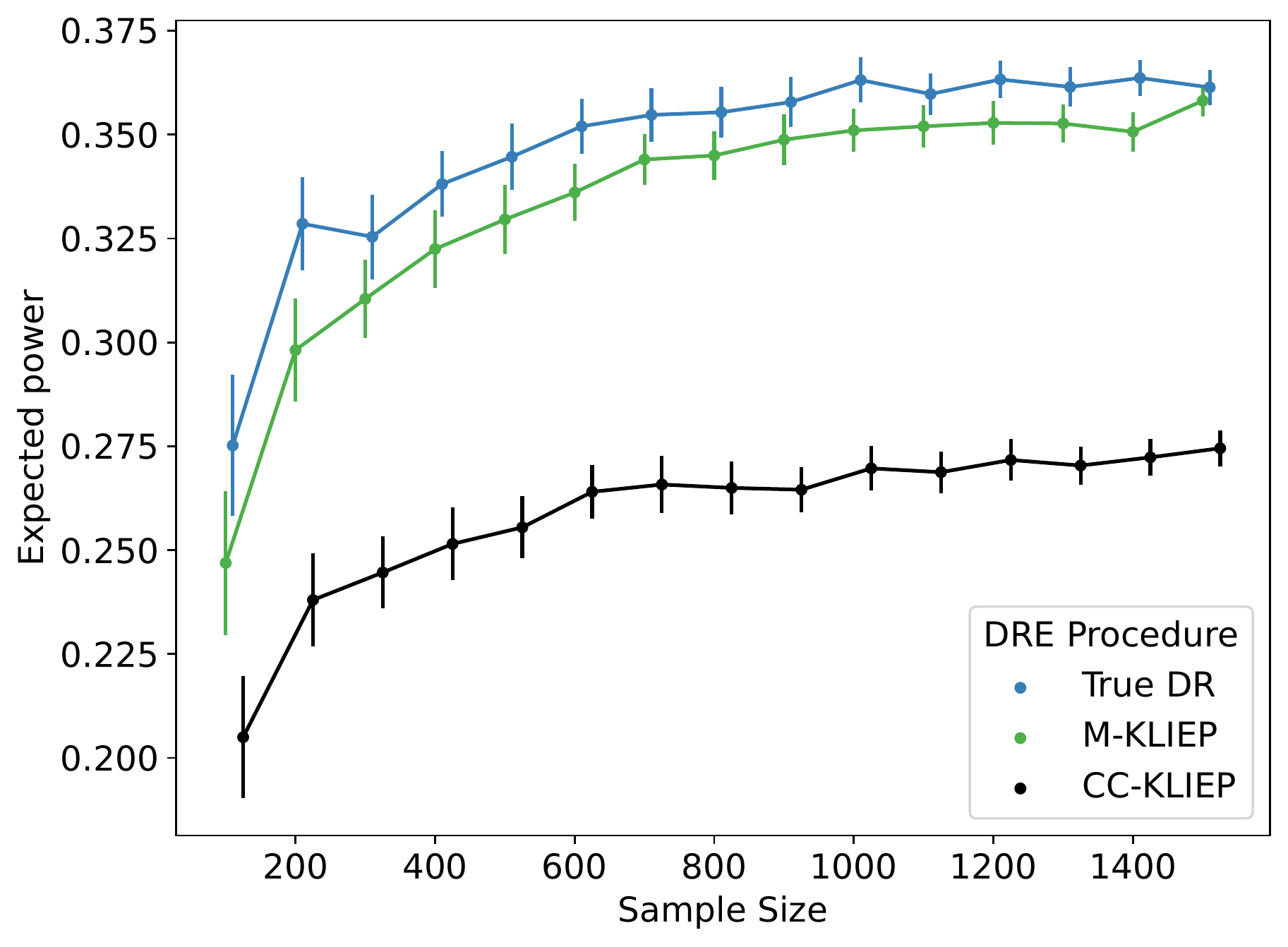}
    \caption{Expected Power for varying $n$ with 99\% C.I.s}
    \label{fig:np_exp_power}
\end{figure}

We see similar results to our correctly specified DRE experiment with the complete case approach performing poorly. As we can see, despite our model not being able to encompass the true density ratio, our estimate density ratio via M-KLIEP still gives a reasonable approximation to $r^*$ and therefore gives a good NP classification procedure. Due to $\cal G$ being incorrectly specified, M-KLIEP never converges to the expected power using the true DR. Additional synthetic experiments are given in Appendix \ref{app:synth_exp_add}

\section{REAL-WORLD DATASETS WITH SIMULATED MISSINGNESS}
We now go on to perform NP classification experiments with real-world data. To make this scenario more realistic we will assume the $\varphi$ to be unknown and aim to estimate them. We briefly tackle how to do this in the section below. 
\subsection{Learning the Missingness Function}\label{sec:miss_learn} Thus far we have assumed  $\varphi$ is known. We now consider the challenge of learning $\varphi$ by querying the latent values $Z_i$ for a subset of our missing data.


Here we are motivated by settings in which we can  send a small sub-sample of observations which erred (i.e. a subset of $\{i \in [n]: X_i=\NaN\}$) off for further investigation to obtain their respective true values $Z_i$.  Now we can learn the missingness function $\varphi$ by fitting a logistic regression with samples of $( Z_i,\one\{X_i=\NaN\})$ and treating $\one\{X_i=\NaN\}$ as the response variable.  


As we have not observed $Z_i$ for every $i\in [n]$, we cannot use a standard logistic regression to learn $\varphi$. However, we can use an adaption of logistic regression designed to deal with miss-representative class proportions presented in \cite{king2001}. More details on this can be found in the Appendix in section \ref{app:adj_logreg}.
We now try our methods on real-world data to assess their efficacy.

\subsection{Set-Up}
We choose datasets and classification problems where the NP classification is properly motivated: One would want to strictly control miss-classification for one class in these datasets. We have chosen 3 datasets which we refer to as ``Fire", ``CTG", and ``weather"; 
\begin{itemize}
    \item \textbf{CTG dataset}: This data set contains 11 different summaries of Cardiotocography (CTG) data for 2126 foetuses where we take each foetus to be an observation. Alongside this is a classification of each foetus as ``Healthy", ``Suspect", or ``Pathologic". We aim to predict whether the foetus is classified as ``Healthy" or not. We take the error controlled class to be ``Suspect" or ``Pathologic".
    \item \textbf{Fire dataset}: This data set contains 62,630 readings of 12 atmospheric measures such as temperature, humidity, and CO2 content. We aim to predict whether a fire is present. We take the presence of fire as the error controlled class.
    \item \textbf{Weather data}: This data set contains 142,193 observations of 62 dimensions giving various weather readings for a given day in a given location in Australia and whether it rained the following day. We aim to  predict whether it rains the following day. We take the occurrence of rain as the error controlled class. 
\end{itemize}

In all cases, we artificially induce missing observations separately across each dimension and only in the non error-controlled class. We construct each dimension-wise missing function $\varphi_j, j \in [d]$ as $\varphi_j(z)=(1+\exp\{\tau_j(a_{j,0}+a_{j,1}z)\})$
where $a_{j,0},a_{j,1}\in[0,\infty)$ and $\tau_j\in\{-1,1\}$. 

\begin{figure*}[t]
     \centering
     \begin{subfigure}{0.3\textwidth}
        \centering
        \includegraphics[width=1\textwidth]{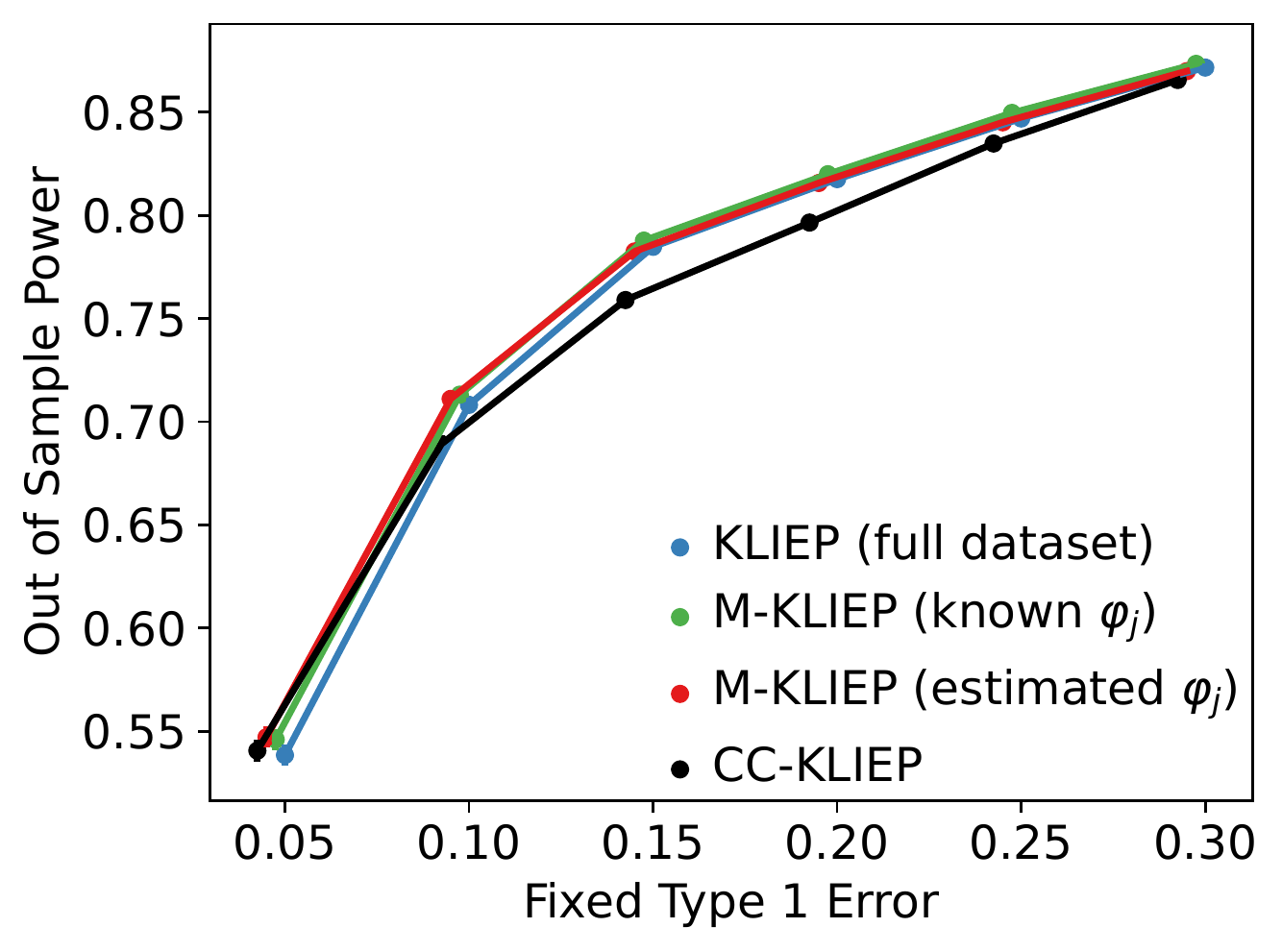}
        \caption{CTG Dataset.}
        \label{fig:np_RWE_CTG_varyalpha}
     \end{subfigure}
     \hfill
        \begin{subfigure}{0.3\textwidth}
        \centering
        \includegraphics[width=1\textwidth]{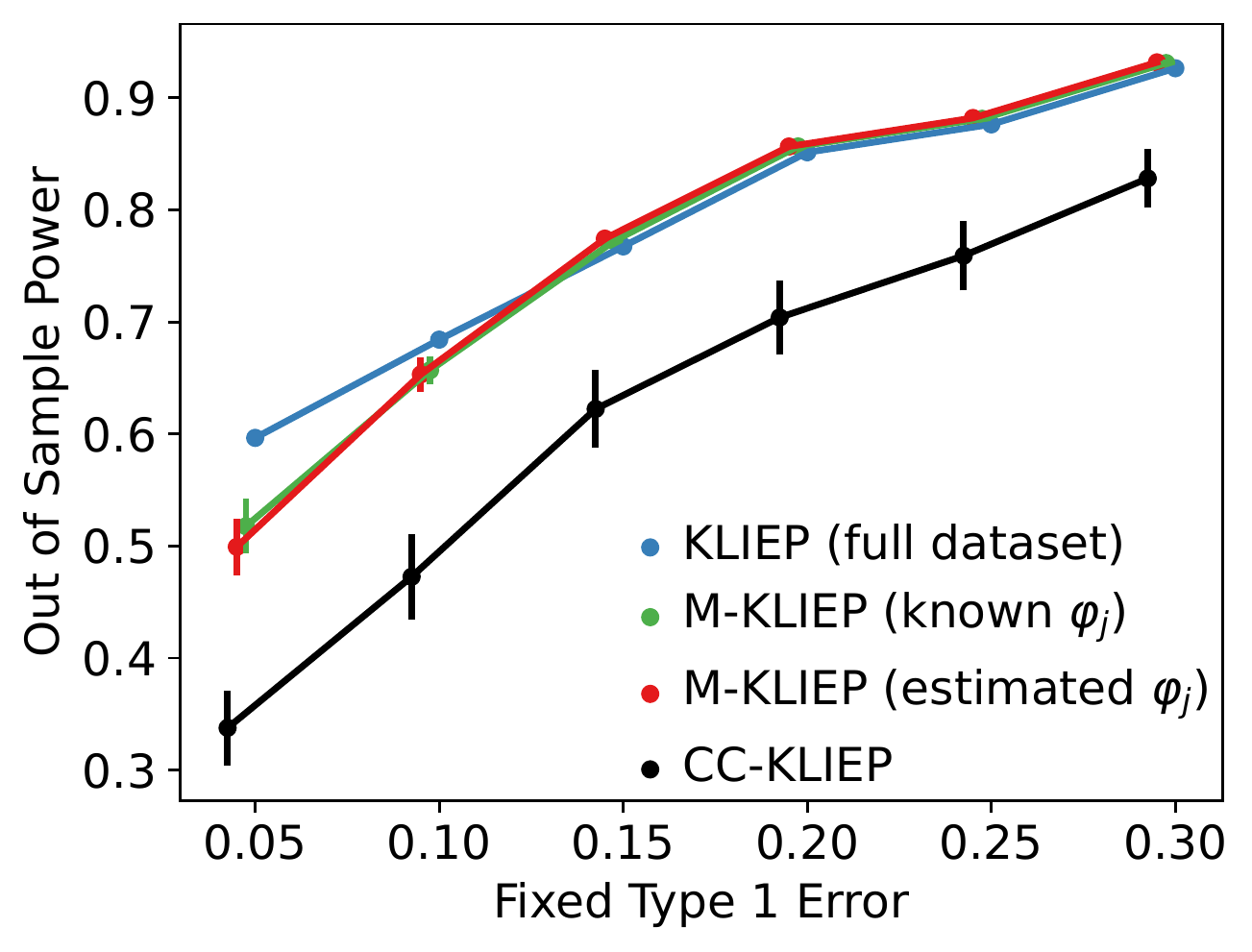}
        \caption{Fire Dataset.}
        \label{fig:np_RWE_smoke_varyalpha}
     \end{subfigure}
     \hfill
     \begin{subfigure}{0.3\textwidth}
        \centering
        \includegraphics[width=1\textwidth]{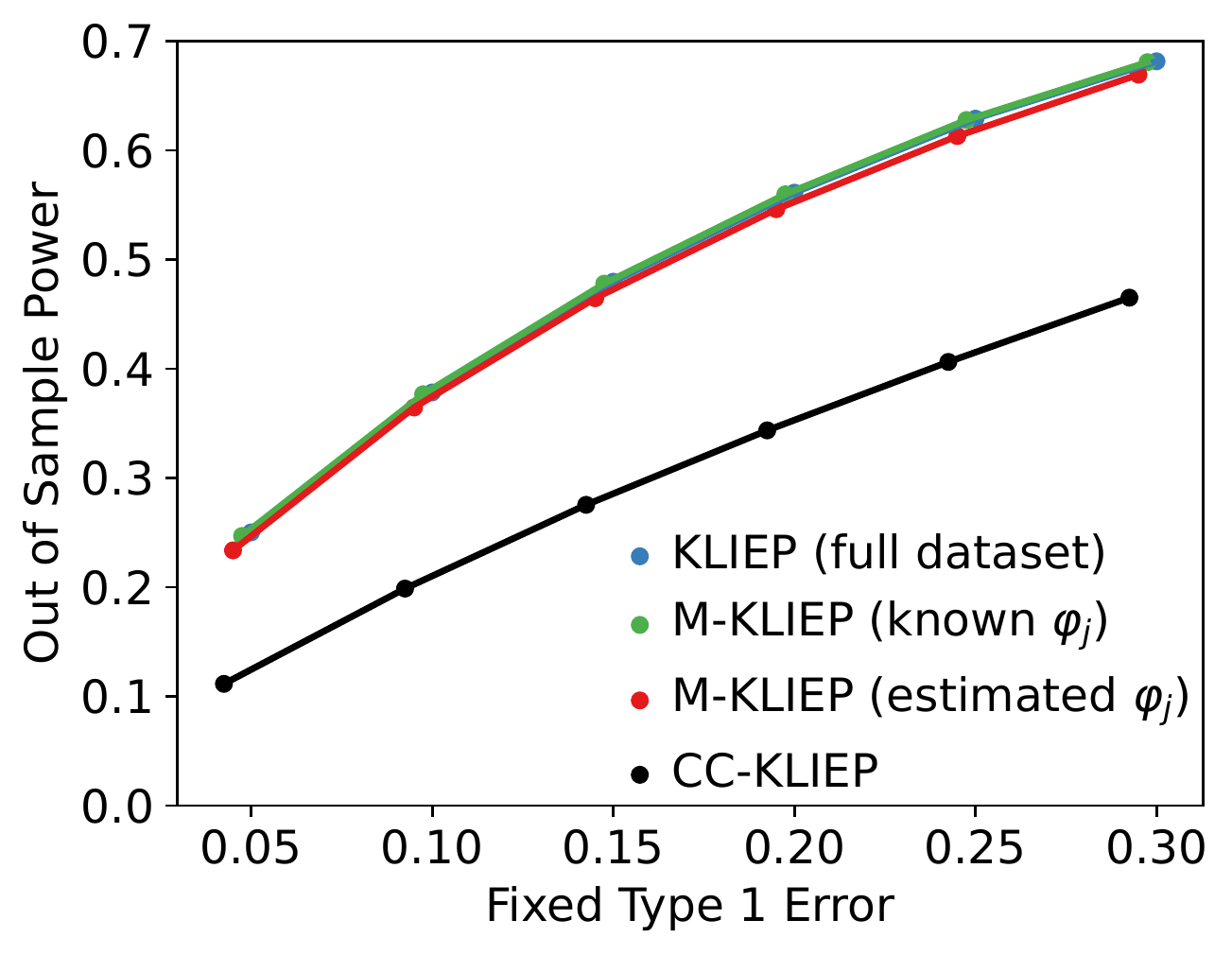}
        \caption{Weather Dataset.}
        \label{fig:np_RWE_weather_varyalpha}
     \end{subfigure}
     \caption{Out of sample power with pseudo 95\% C.I.s for various different target Type I errors.}
     \label{fig:np_RWE_vary_alpha}
    \hfill
\end{figure*}

\begin{figure*}[t]
     \centering
     \begin{subfigure}[b]{0.3\textwidth}
        \centering
        \includegraphics[width=\textwidth]{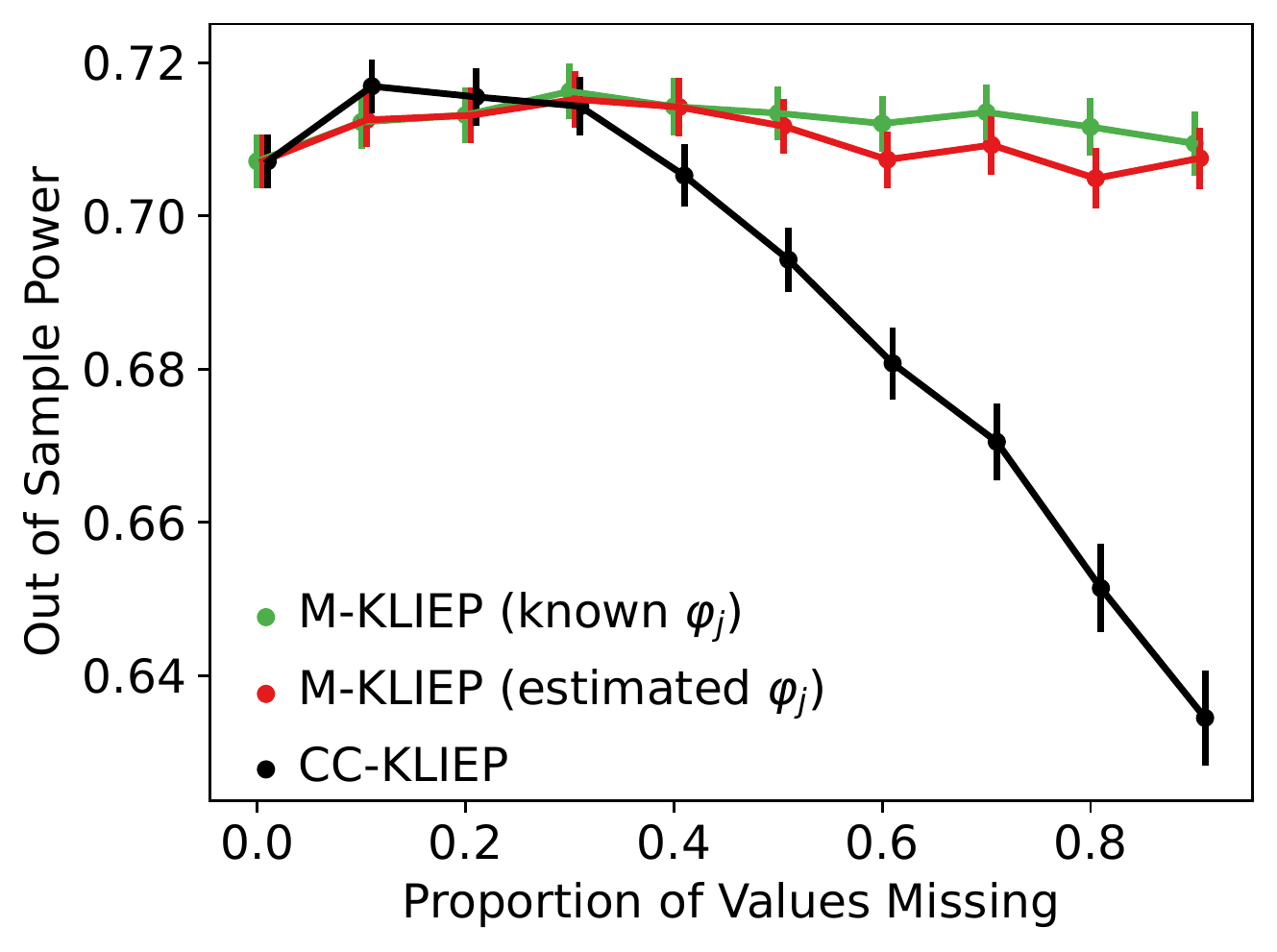}
        \caption{CTG Dataset.}
        \label{fig:np_RWE_CTG_varymiss}
     \end{subfigure}
     \hfill
     \begin{subfigure}[b]{0.3\textwidth}
        \centering
        \includegraphics[width=\textwidth]{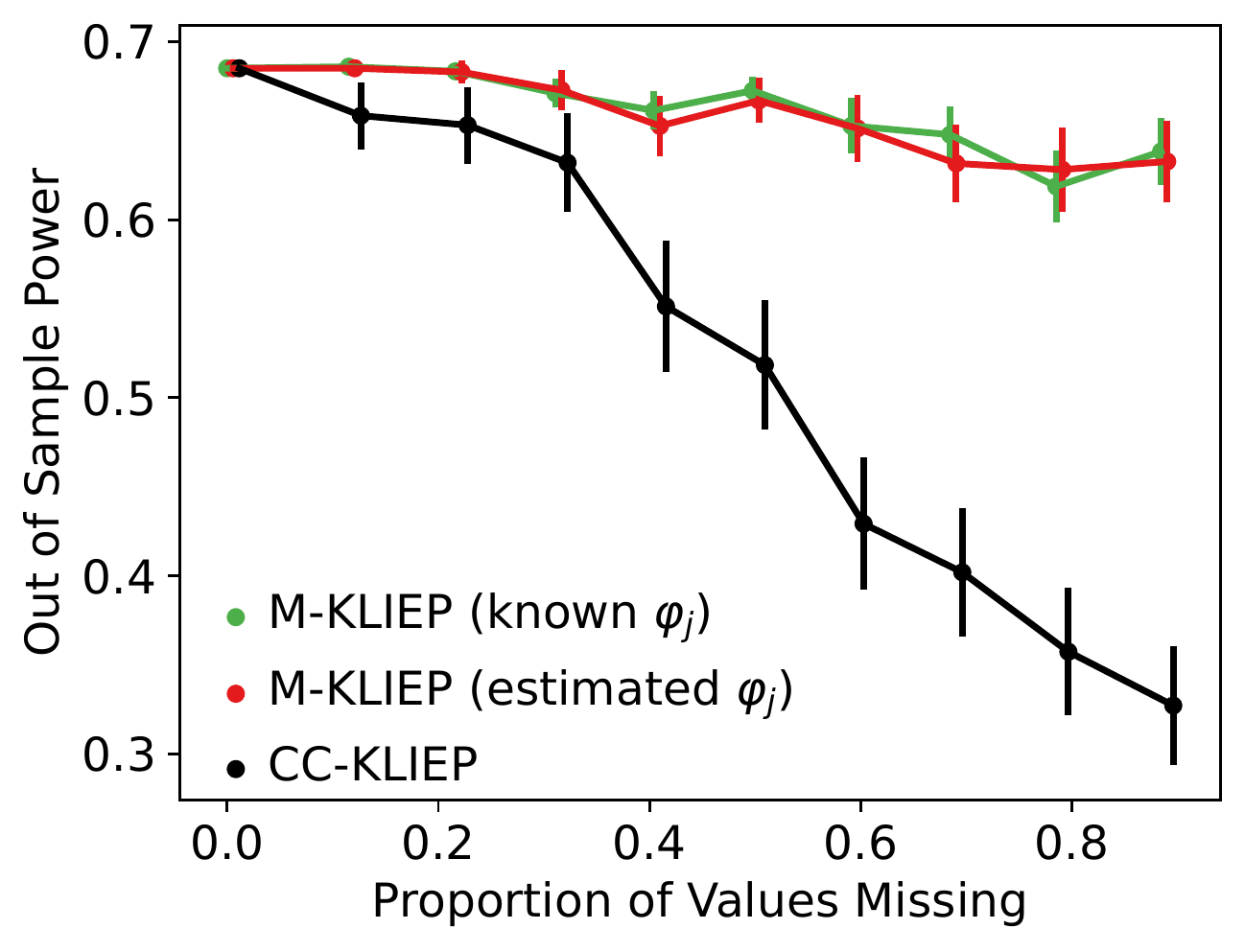}
        \caption{Fire Dataset.}
        \label{fig:np_RWE_smoke_varymiss}
     \end{subfigure}
     \hfill
      \begin{subfigure}[b]{0.3\textwidth}
        \centering
        \includegraphics[width=1\textwidth]{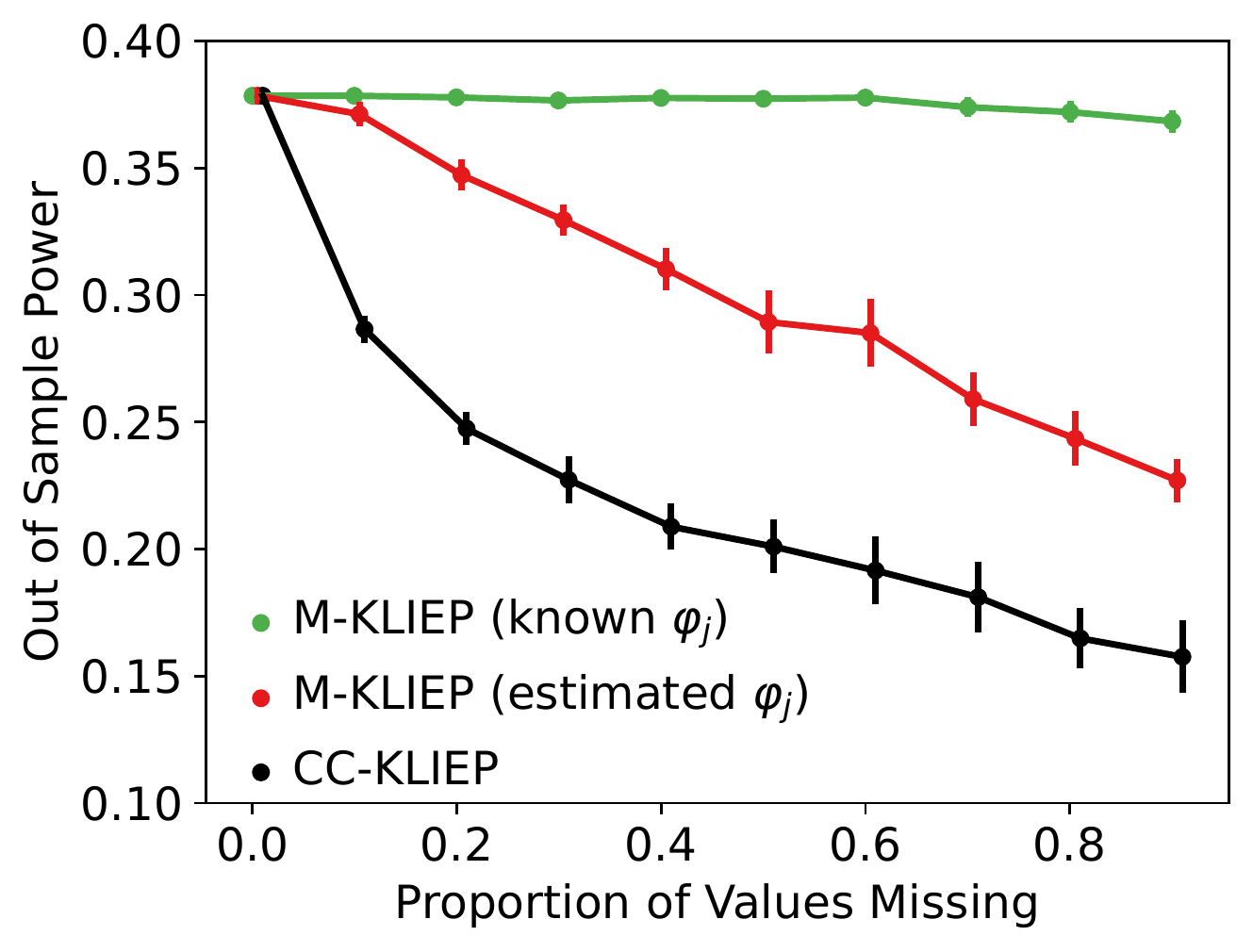}
        \caption{Weather Dataset}
        \label{fig:np_RWE_weather_varymiss}
     \end{subfigure}
     \caption{Out of sample power with pseudo 95\% C.I.s for various $\varphi$ with varying missing proportions.}
     \label{fig:np_RWE_varymiss}
\end{figure*}
We constructed NP classifiers using M-KLIEP and CC-KLIEP under our naive Bayes framework introduced in Section \ref{sec:partial_miss}. For M-KLIEP we test both cases where $\varphi_j$ is known and $\varphi_j$ is learned. Estimating $\varphi_j$ is done using the method described in Section \ref{sec:miss_learn} and each $\varphi_j$ is estimated separately using only data from that dimension. We also estimate the density ratio using the fully observed data via standard KLIEP as a benchmark.

The power of the NP classifiers produced with these DRE procedures was then calculated on fully observed testing data. We repeat this process multiple times with new random test/train splits and $\tau_j$ at each iteration to estimate the pseudo expected power and corresponding 95\% C.I.s.
We apply this technique at different target Type I errors and varying $a_{j,0}$ to construct datasets with different missing proportions.
Additional detail is given in Appendix \ref{app:real_exp_design}.

\subsection{Results}
In Figure \ref{fig:np_RWE_vary_alpha} we see that in both the Fire and Weather data, M-KLIEP significantly outperforms CC-KLIEP for all values of $\alpha$. For the CTG data we see significant out-performance for $\alpha\in[0.1,0.25]$ and comparable performance in the extremes. Surprisingly,  we see that M-KLIEP performs equally well when $\varphi$ is either learned or known. Further, M-KLIEP perform comparably with KLIEP run on fully observed datasets.
In Figure \ref{fig:np_RWE_varymiss} we see that M-KLIEP outperforms CC-KLIEP for larger missingness proportions. As we would expect, the more missing points, the worse CC-KLIEP performs. For the Fire and CTG data we see no loss in performance when learning $\varphi$. For the Weather data, an increase in the missing proportion leads to a decrease in performance of M-KLIEP using a learned $\varphi$. Remarkably, there is no performance loss when running M-KLIEP using the true $\varphi$ in all datasets. 

Results from additional experiments comparing with an iterative imputation approach are given in Appendix section \ref{app:real_exp_add}. 
\section{DISCUSSION \& CONCLUSIONS}
DRE is a widely used machine learning technique with a diverse range of applications. In this paper we have shown that, when data is MNAR, naively performing DRE by discarding all missing observations can lead to inaccurate estimates of the density ratio functions. We have proposed a novel procedure M-KLIEP as well as adaptations to a broad family of DRE procedures to account for the MNAR phenomenon when the missingness structure is known. For M-KLIEP we have presented finite sample bounds under a commonly used parametric form showing convergence at rate $\sqrt{\effectiveSampleSize}$ where $\effectiveSampleSize$ is the effective sample size. We have then extended these approaches in multiple ways to partial missingness across multiple dimensions under the naive Bayes framework. Finally, we have presented a technique to estimate missing patterns from data by querying the true values of a few missing observations.  

We then studied a downstream application of DRE in NP classification. 
We have adapted NP classification to MNAR data settings and shown that
our adaptation ensures satisfaction of our Type I error constraints with high probability.
We have also provided finite sample convergence results for the expected power of our NP classification procedure and hence shown that it converges to the optimal power under certain conditions.
 
We have shown all these adaptations to work well in practice on simulated data and applied them to real world data with  known and unknown synthetic missing patterns. 
In particular, our adapted NP classifier using M-KLIEP has been shown to negate losses in accuracy/power incurred by MNAR phenomenon in most cases.
We briefly explore the societal impact of our work in Appendix \ref{app:soc_imp}.


A natural extension of this work would be relaxing the naive Bayes assumption required for the handling partial missingness (Section \ref{sec:partial_miss}) to a more realistic set-up where dependencies are properly modelled and accounted for. This extension would enable applications in areas where the naive Bayes framework is unrealistic such as image classification. Another possible direction would be exploring the set-up where the missing pattern is known to only belong to some class as discussed in \cite{Sportisse2020}. 
This approach could be further expanded using a Bayesian framework by assigning of some prior belief on the missing pattern. Finally, our adaptation to NP classification could be extended further by adapting it to the case of partial missingness in both classes.

\subsubsection*{Acknowledgements}
Josh Givens was supported by a PhD studentship from the EPSRC Centre for Doctoral Training in Computational Statistics and Data Science (COMPASS).

\bibliography{ref}
\appendix
\section{ADDITIONAL PROOFS}
Here we prove key results from the paper however before we do we need to first introduce some additional notation. 
We let $p''_{\zerone}$ to refer to the density of $X^\zerone$ with w.r.t. the measure $\nu$ where $\nu$ is the unique measure on $\cal Z\cup\{\NaN\}$ s.t. (such that) $\nu(A)=\mu(A)$ for any $A\in\cal B_{\cal Z}$ and $\nu(\{\NaN\})=1$.
For square matrices $A,B\in\R^{d\times d}$ we take $A\succeq B$ to mean that $A-B$ is positive semi-definite and define $\lambda_\min(A)$ to be the smallest eigenvalue of $A$.
\vspace{0.1cm}

\subsection{KLIEP Objective Simplification}\label{app:KLIEP_rewrite}
\begin{lemma}\label{lemma:KLIEP_rewrite}
Let $\cal G$ be closed under positive scalar scalar multiplication so that for all $r\in\cal G,~a>0$, $a r\in\cal G$. Define $\tilde r$ to be the solution of the constrained optimisation problem
\begin{align*}
    \min_{r\in\cal G}&~\KL(p_1|r p_0)\\
    \text{subject to}:& \int_{\cal Z} r(z) p_0(z)\mu(\diff z)=1.
\end{align*}
Then we can re-write the solution to the above optimisation problem as $\tilde r\coloneqq N^{-1}\cdot r_0$ where
\begin{align*}
    r_0&\coloneqq\argmax{r\in\bar{\cal G}}\E[\log r(Z^1)]-\log\E[r(Z^0)]\\
    N&\coloneqq\E[r_0(Z^0)]
\end{align*}
where $\bar{\cal{G}}$ is $\bar{\cal G}\subseteq\cal G$ s.t. for any $r\in \bar{\cal G}$, there exists $r_0\in\bar{\cal G}, a>0$ s.t. $a\cdot r_0=r$. 
\end{lemma}

\begin{proof}
We have that
\begin{align*}
    \argmin{r\in\cal G} \KL(p_1|r p_0)&=\int_{\cal Z}p_1(z)\log\frac{p_1}{r p_0(z)}\\
    &=\argmin{r\in\cal G}\E[\log\frac{p_1(Z^1)}{r(Z^1)p_0(Z^1)}]\\
    &=\argmin{r\in\cal G}\E[\log p_1(Z^1)]-\E[\log p_0(Z^1)]-\E[\log r(Z^1)]\\
    &=\argmax{r\in\cal G}\E[\log r(Z^1)].
\end{align*}
As such we can re-write the optimisation problem as
\begin{align*}
    \min_{r\in\cal G}&~\E[\log r(Z^1)]\\
    \text{subject to}:& ~\E[r(Z^0)]=1 .
\end{align*}
To simplify this further we now show that 
\begin{align*}
    \{r\in\cal G| \E[r(Z^0)]=1\}=\{\E[r(Z^0)]^{-1}r|r\in\cal G\}.
\end{align*}
To this end let $r\in\cal G$, then as $G$ is closed under positive scalar multiplication, $\E[r(Z^0)]^{-1}r\in\cal G$. Hence, as $\E\left[\E[r(Z^0)]^{-1}r(Z^0)\right]=1$, we have that $\E[r(Z^0)]^{-1}r\in\{r\in\cal G| \E[r(Z^0)]=1\}$. As the other inclusion is trivial, we have shown equality between the two sets of functions.

From this we can now re-write the optimisation as
\begin{align*}
&\argmax{\E[r(Z^0)]^{-1}r\in\{\E[r(Z^0)]^{-1}r|r\in\cal G\}} \E[\log(\E[r(Z^0)]^{-1}r(Z^1))]\\
=&\argmax{\E[r(Z^0)]^{-1}r\in\{\E[r(Z^0)]^{-1}r|r\in\cal G\}}\E[\log r(Z^1)]-\log\E[r(Z^0)]
\end{align*}
which gives the desired result.
\end{proof}

\begin{remark}
We can take $\rrspace$ s.t. for any $r\in\rrspace$, $a'>0$, if $r,a 'r\in\cal G$ then $a'=1$. Hence for certain choices of $\cal G$ we can choose $\bar{\cal G}$ so that $r_0$ is unique
\end{remark}

\subsection{KLIEP Finite Sample Proofs}\label{app:KLIEP_bound}
Before we can prove Theorem \ref{thm:kliep_bound} there are a few simple results we need first.
The first is a generalisation of Bernstein bounds to random vectors presented in \cite{Kohler2017SubsampledCR}. 
\begin{lemma}[\cite{Kohler2017SubsampledCR}]\label{lemma:vec_bern}
Let $W_1,\dotsc,W_n$ be IID copies of a non-negative random vector $W$ over $\R^d$ with $\E[W]=\mu$.
Now let $b, \sigma>0$ be s.t.  $\|W\|\leq b$ a.s. and  $\E[\|W\|^2]\leq\sigma^2$. 
Then for any $\eps<\frac{\sigma^2}{b}$
\begin{align*}
    \prob\left(\left\|\frac{1}{n}\sum_{i=1}^n W_i -\mu\right\|\leq\eps\right)\geq1-e^{1/4}\exp\left\{-\frac{n\eps^2}{8\sigma^2}\right\}
\end{align*}
\end{lemma}

We also need a Lemma bounding true and sample covariance in terms of the eigen-space
\begin{lemma}
\label{lemma:cov_bound}
Let $W$ be a RV over $\R^d$ with $\|W\|\leq b$ a.s. and let $W_1\dotsc,W_n$ to be IID copies of $W$.
Define the sample covariance of $\{W_i\}_{i\in[n]}$ by
\begin{align*}
    \widehat\var(W;\{W_i\}_{i\in[n]})\coloneqq\frac{1}{n}\sum_{i=1}^n W_i W_i^\top-\left(\frac{1}{n}\sum_{i=1}^n W_i\right)\left(\frac{1}{n}\sum_{i=1}^n W_i\right)^\top
\end{align*}
then provided $\sigma_\min\leq 4 b^2$ 
\begin{align*}
    \prob\left(\lambda_\min(\widehat\var(W;\{W_i\}_{i\in[n]})\geq\half\lambda_\min(\var(W))\right)&
    \geq 1-(d+e^{1/4})\exp\left\{\frac{\sigma_\min n}{32b^2}\right\}
\end{align*}
where $\sigma_\min\coloneqq\lambda_\min(\var(W))$.
\end{lemma}
\begin{proof}
Define the centred RVs $Y_1,\dotsc, Y_n$ by $Y_i \coloneqq W_i-\E[W_i]$ and the random matrices $R_n, S_n$ by
\begin{align*}
    S_n&\coloneqq\frac{1}{n}\sum_{i=1}^n Y_i Y_i^\top \\
    R_n&\coloneqq\left(\frac{1}{n}\sum_{i=1}^n Y_i\right)\left(\frac{1}{n}\sum_{i=1}^n Y_i\right)^\top
\end{align*}
so that $R_n-S_n=\widehat\var(Y_1,\dotsc,Y_n)=\widehat\var(W_1,\dotsc,W_n)$.

Firstly we use the matrix Chernoff inequalities \citep{Tropp2015} to get that
\begin{align*}
        \prob\left(\lambda_\min(S_n)\geq\frac{3}{4}\sigma_\min\right)&\geq1-d\exp\left\{-\frac{a'\sigma_\min n}{b^2}\right\}\\
\end{align*}
where $a'\coloneqq3/4\log(3/4)+1/4$.

Now note that for any $x\in\R^d$ $\lambda_\max(xx^\top)=\|x\|^2$ so that $\lambda_\max(R_n)=\|\frac{1}{n}\sum_{i=1}^nY_i\|^2$. Additionally, by Lemma \ref{lemma:vec_bern}, provided $\sigma_\min\leq 4 b^2$,
\begin{align*}
   \prob\left(\lambda_\max(R_n)\leq\frac{1}{4}\sigma_\min\right)= \prob\left(\|\frac{1}{n}\sum_{i=1}^n Y_i\|<\half\sqrt{\sigma_\min}\right)&\geq1-e^{1/4}\exp\left\{-\frac{\sigma_\min n}{32b^2}\right\}\\
\end{align*}
Taking the intersection of these two events we get that w.p. at least $1-(d+e^{1/4})\exp\{\frac{-\sigma_\min n}{32b^2}\}$
\begin{align*}
    \lambda_\min(\widehat\var(Y_1,\dotsc,Y_n))&=\lambda_\min(S_n-R_n)\\
    &\geq\lambda_\min(S_n)-\lambda_\max(R_n)\\
    &\geq\frac{3}{4}\sigma_\min-\frac{1}{4}\sigma_\min\\
    &=\half\lambda_\min(\var(W))
\end{align*}
\end{proof}

We can now use this to state and prove our theorem giving finite sample bounds for the estimation error of KLIEP.
\begin{theorem}\label{thm:kliep_bound}
Define $r_\theta$ as in (\ref{eq:KLIEP_paramform}) and the population and empirical losses $L,~\hat L$ by
\begin{align*}
    L(\theta)&\coloneqq-\E[\log(r_\theta(Z^1)]+\log\E[r_\theta(Z^0)]\\
    \hat L(\theta,D)&\coloneqq-\frac{1}{n_1}\sum_{i=1}^{n_1}\log(r_\theta(Z_i^1))+
    \log\frac{1}{n_0}\sum_{i=1}^{n_0}r_\theta(Z_i^0).
\end{align*}
Let $\tilde{\theta},\hat\theta$ be the minimisers of these respective losses.
Now assume that $\|f\|_{\infty}<\infty$ and $\lambda_\min(\var(f(Z^0)))\coloneqq\sigma_\min>0$.
Then we have that for any $\delta<\half$ and $n_\min\geq C_0\log(1/\delta)$,
\begin{align*}
    \prob\left(\|\hat\theta'-\tilde{\theta}\|\leq\sqrt{\frac{C_0\log(1/\delta)}{n_\min}}\right)\geq 1-\delta
\end{align*}
with $C_0$ a constant depending upon $\|f\|_{\infty},\sigma_\min,d,\tilde\theta$.
\end{theorem}

\begin{proof}
We firstly show that $\hat L$ is a convex function for $\theta$ for any sample $\cal D\in\cal Z^{n_0,n_1}$.

The first term is linear in $\theta$ so clearly convex. For the second term, note that it can be written as $g\circ h_{\cal D^0}(\theta)$ with $g:\R^{n_0}\rightarrow\R$, $h:\R^d\rightarrow \R^{n_0}$ given by
\begin{align*}
    h_{\cal D^0}(\theta)&=\begin{pmatrix} 
        \theta^\top f(x_1^0)\\ \vdots\\ \theta^\top f(x_{n_0}^0)
    \end{pmatrix}\\
    g(w)&=\log\sum_{i=1}^{n_0}\exp\{w_i\}
\end{align*}
Now again for any $\cal D^0$, $h$ is linear in $\theta$ and therefore convex additionally $g$ is convex and non-decreasing therefore $g\circ h$ is convex.

Now we state some important bounds which we will use in our proof. If we let $b\coloneqq\|f\|_{\infty}$ $M\coloneqq\exp\{\|\tilde{\theta}\|b\}$ then for any $z\in\cal Z$,
\begin{itemize}
    \item $1/M \leq r_{\tilde{\theta}}(z)\leq M$
    \item $\|\nabla r_{\tilde{\theta}}(z)\|\leq b M$
    \item $\left\|\frac{\nabla r_{\tilde{\theta}}(z)}{r_{\tilde{\theta}}(x)}\right\|\leq b$
\end{itemize}

Finally we need to bound the minimum eigenvalue of the Hessian of the empirical loss function. To this end note that for any 
$\bar\theta\in B_\eta(\tilde{\theta})\coloneqq\{\theta\in\R^d|\|\tilde{\theta}-\theta\|\leq\eta\}$, 
\begin{align*}
    \nabla^2_\theta \hat L(\bar\theta;D)&\coloneqq\frac{\frac{1}{n_0^2}\sum_{i,j}r_{\bar\theta}(Z_i^0)r_{\bar\theta}(Z_i^0)[f(Z_i^0)-f(Z_j^0)][f(Z_i^0)-f(Z_j^0)]^\top}{\frac{1}{n_0^2}\sum_{i,j}r_{\bar\theta}(Z_i^0)r_{\bar\theta}(Z_i^0)\}}\\
    &\succeq\frac{\frac{1}{n_0^2}\sum_{i,j}M^{-2}2e^{-2\eta b}[f(Z_i^0)-f(Z_j^0)][f(Z_i^0)-f(Z_j^0)]^\top}{\frac{1}{n_0^2}\sum_{i,j}e^{2\eta b}r_{\tilde{\theta}}(Z_i^0)r_{\tilde{\theta}}(Z_j^0)}\\
    &\succeq\frac{1}{M^2e^{4\eta b}}\frac{\frac{1}{n_0^2}\sum_{i,j}[f(Z_i^0)-f(Z_j^0)][f(Z_i^0)-f(Z_j^0)]^\top}{\frac{1}{n_0^2}\sum_{i,j}r_{\tilde{\theta}}(Z_i^0)r_{\tilde{\theta}}(Z_j^0)}\\
    &=:\frac{2}{M^2e^{4\eta b}}\frac{\widehat\var(f(Z^0);\rsampzero)}{\widehat \E[r_{\tilde{\theta}}(Z^0);\rsampzero]^2}\\
\end{align*}
where, for square matrices $A,B$ we take $A\succeq B$, to mean $A-B$ is positives semi-definite.

Thus if we define $0<\sigma_\min\coloneqq\lambda_\min(\var[f(Z^0)])$ then Lemma \ref{lemma:cov_bound}
tells us that, provided $\frac{1}{4}\sigma_\min<1$,
\begin{align*}
    \prob\left(\lambda_\min(\widehat\var(f(Z^0);\rsampzero))\geq\half\sigma_\min\right)\geq1-(d+e^{1/4})\exp\left\{-\frac{1}{8}\left(\frac{\sigma_\min}{4b^2}\wedge1\right)n\right\}
\end{align*}
Additionally if we let $\mu_2\coloneqq\E[r_{\tilde{\theta}}(Z^0)]$ we can use Lemma \ref{lemma:vec_bern} to get that,
\begin{align*}
    \prob\left(\widehat\E[r_{\tilde{\theta}}(Z^0);\rsampzero]\leq\frac{3}{2}\mu_2\right)\geq 1-e^{1/4}\exp\left\{-\frac{1}{8}\left(\frac{\mu_2^2}{4M^2}\wedge 1\right)n_0\right\}
\end{align*}
Thus for any given $\bar\theta\in B_\eta(\tilde{\theta})$ w.p. 

\begin{align}\label{eq:hess_bound}
    \prob\left(\underbrace{\lambda_\min(\nabla^2_\theta\hat L(\bar\theta;\rsamp))\geq\frac{4\sigma_\min}{9M^2e^{4\eta b}\mu_2^2}}_{\coloneqq A}\right)\geq1-(d+2e^{1/4})\exp\left\{-\frac{1}{8}\left(\frac{\mu_2^2}{4b^2}\wedge\frac{\sigma_\min}{4b^2}\wedge 1\right) n_0\right\}
\end{align}

Now we have these bounds we can continue with the main body of the proof. 
For some $\eta>0$, define $\hat \theta '$ to be the solution to the following constrained optimisation problem
\begin{align*}
    \min_{\theta\in B_\eta(\tilde{\theta})}&~ \hat L(\theta;D)\\
\end{align*}
While, $\hat\theta'$ is a purely theoretical quantity, we know that when $\|\hat\theta'-\tilde{\theta}\|\leq1$, $\hat\theta=\hat\theta'$ and so we can use it to get finite sample bounds for $\hat\theta$.

By the necessary KKT conditions, we know that there exists $\mu^*>0$ s.t. 
\begin{align*}
    \nabla_{\theta} \hat L (\hat \theta')+\mu^*(\hat\theta'-\tilde{\theta})=0
\end{align*} 
in turn giving us that
\begin{align*}
    0&=\langle \hat \theta'-\tilde{\theta},\nabla_{\theta} \hat L (\hat \theta')+\mu^*(\hat\theta'-\tilde{\theta})\rangle\\
    &=\langle \hat \theta'-\tilde{\theta},\nabla_{\theta} \hat L (\hat \theta')\rangle+\mu^*\|(\hat\theta'-\tilde{\theta})\|^2\\
    &=\langle \hat \theta'-\tilde{\theta},\nabla_{\theta} \hat L (\tilde{\theta})+\nabla^2_{\theta}\hat L(\bar{\theta})(\hat\theta'-\tilde{\theta})\rangle+\mu^*\|(\hat\theta'-\tilde{\theta})\|^2.
    \end{align*}
Using multidimensional MVT over each coordinate with $\bar{\theta}$ some point in the hypercube with opposite corners given by $\tilde{\theta},\hat\theta'$ (which is a subset of $B_\eta(\tilde{\theta})$. We then get that
\begin{align*}
   \|\hat \theta'-\tilde{\theta}\|^2\lambda_{\min}(\nabla_{\theta}^2\hat L(\bar\theta))&\leq (\hat\theta'-\tilde{\theta})^\top\nabla^2_{\theta}\hat L(\bar \theta)(\hat\theta'-\tilde{\theta})+\mu^*\|(\hat\theta'-\tilde{\theta})\|^2\\
   &=\langle\tilde{\theta}-\hat\theta',\nabla_\theta\hat L(\tilde{\theta})\rangle\\
    &\leq\|\hat\theta'-\tilde{\theta}\|\|\nabla_\theta\hat L(\tilde{\theta})\|\\.
    \intertext{Therefore we have that}
    \|\hat \theta'-\tilde{\theta}\|&
    \leq\frac{1}{\lambda_{\min}(\nabla_{\theta}^2\hat L(\bar\theta))}
    \|\nabla_\theta\hat L(\tilde{\theta})\|.
\end{align*}
Hence if $A$ from (\ref{eq:hess_bound}) holds with our given $\bar\theta\in B_\eta(\tilde{\theta})$ we have that
\begin{align*}
    \|\hat \theta'-\tilde{\theta}\|&\leq\frac{9M^2e^{4\eta b} \mu_2^2} 
    {4\sigma_{\min}}\|\nabla_\theta\hat L(\tilde{\theta})\| \quad\text{a.s. .}
\end{align*}
From now on we will work on the event $A$ so that
bounding $\|\hat \theta-\tilde{\theta}\|$ simply requires us to bound $\|\nabla_\theta\hat L(\tilde{\theta})\|$.

As $\nabla_\theta L(\tilde{\theta})=0$ we have 
\begin{align*}
    \|\nabla_\theta\hat L(\tilde{\theta})\|&=\|\nabla_\theta\hat L(\tilde{\theta})-\nabla_\theta L(\tilde{\theta})\|\\
    &=\bigg\|\E\left[\frac{\nabla_\theta r_{\tilde{\theta}}(Z^1)}{r_{\tilde{\theta}}(Z^1)}\right]-\frac{\E[\nabla_\theta r_{\tilde{\theta}}(Z^0)]}{\mu_2}-\frac{1}{n_1}\sum_{i=1}^{n_1}\frac{\nabla_\theta r_{\tilde{\theta}}(Z_i^1)}{r_{\tilde{\theta}}(Z_i^1)}+\\
    &\quad\quad\quad\quad\frac{\frac{1}{n_0}\sum_{i=1}^{n_0}\nabla_\theta r_{\tilde{\theta}}(Z_i^0)}{\frac{1}{n_0}\sum_{i=1}^{n_0}r_{\tilde{\theta}}(Z_i^0)}\bigg\|\\
    &\leq\underbrace{\left\|\E\left[\frac{\nabla_{\theta}r_{\tilde{\theta}}(Z^1)}{r_{\tilde{\theta}}(Z^1)}\right]-\frac{1}{n_1}\sum_{i=1}^{n_1}\frac{\nabla_\theta r_{\tilde{\theta}}(Z_i^1)}{r_{\tilde{\theta}}(Z_i^1)}\right\|}_{\coloneqq I_1}+\\
    &\quad\quad\quad\quad\underbrace{\left\|\frac{\E[\nabla_\theta r_{\tilde{\theta}}(Z^0)]}{\E[r_{\tilde{\theta}(Z}^0)]}
    - \frac{\frac{1}{n_0}\sum_{i=1}^{n_0}\nabla_\theta r_{\tilde{\theta}}(Z_i^0)}{\frac{1}{n_0}\sum_{i=1}^{n_0}r_{\tilde{\theta}}(Z_i^0)}\right\|}_{\coloneqq I_2}
\end{align*}
We now go on to bound $I_1,I_2$, using using generalised Hoeffding bounds.

We first introduce some additional notation to allow us to proceed. Define RVs $W_1,W_2$ and constants $\mu_1,\mu_2$ as follows
\begin{align*}
    \mu_1&\coloneqq\E[\nabla_\theta r_{\theta}(Z^0)]& W_1&\coloneqq\frac{1}{n_0}\sum_{i=1}^{n_0}\nabla_\theta r_{\tilde{\theta}}(Z_i^0)\\
    \mu_2&\coloneqq\E[r_{\tilde{\theta}}(Z^0)] & W_2&\coloneqq\frac{1}{n_0}\sum_{i=1}^{n_0}r_{\tilde{\theta}}(Z_i^0)
\end{align*}
Then we have 
\begin{align*}
    I_2&\coloneqq\left\|\frac{\mu_1}{\mu_2}-\frac{W_1}{W_2}\right\|\\
    &=\frac{\left\|\mu_1 W_2-\mu_2 W_1\right\|}{\mu_2 W_2}\\
    &\leq \frac{\|\mu_1-W_1\|}{\mu_2}+\frac{|\mu_2-W_2|\|\mu_1\|}{\mu_2 W_2}\\
\end{align*}
Thus to have  
\begin{align*}
    I_1+I_2&<\frac{4\eps\sigma_{\min}}{9M^2e^{4\eta b}\mu_2^2},\\
    \intertext{it is sufficient to have}
    I_1&<\frac{2\eps\sigma_{\min}}{9M^2e^{4\eta b}\mu_2^2},\\
    \|\mu_1-W_1\|&<\frac{\eps\sigma_{\min}}{9M^2e^{4\eta b}\mu_2},\\
    |\mu_2-W_2|&<\frac{\eps\sigma_{\min}}{18M^2e^{4\eta b}\|\mu_1\|}\\
\end{align*}
provided $\eps<\frac{9M^2e^{4\eta b}\mu_2\|mu_1\|}{\sigma_\min}$.
Furthermore we can get probabilities on each of these events using Lemma \ref{lemma:vec_bern} once again. 

These probabilities are
\begin{align*}
    \prob\left(I_1<\frac{2\eps\sigma_{\min}}{9M^2e^{4\eta b}\mu_2^2}\right)&
    \geq 1-e^{1/4}\exp\left\{-\frac{1}{8}\left(\frac{4\eps^2\sigma_{\min}^2}{81b^2M^4e^{8\eta b}\mu_2^4}\wedge 1\right)n_1\right\}\\
    \prob\left(\|\mu_1-W_1\|\leq\frac{\eps\sigma_{\min}}{9M^2e^{4\eta b}\mu_2}\right)&
    \geq 1-e^{1/4}\exp\left\{-\frac{1}{8}\left(\frac{\eps^2\sigma_\min^2}{81b^2 M^6e^{8\eta b}\mu_2^2}\wedge 1\right)n_0\right\}\\
    \prob\left(|\mu_2-W_2|\leq\frac{\eps\sigma_\min}{18M^2e^{4\eta b}\|\mu_1\|}\right)&
    \geq 1-e^{1/4}\exp\left\{-\frac{1}{8}\left(\frac{\eps^2\sigma^2_\min}{324M^6e^{8\eta b}\|\mu_1\|^2}\wedge 1\right)n_0\right\}
\end{align*}
Therefore using unions bounds and taking $\eta=\|\tilde{\theta}\|$ we have that
\begin{align*}
    \prob(\|\hat\theta'-\tilde{\theta}\|\leq\eps)&\geq
    \prob\left(\|\nabla_{\theta}\hat L(\tilde{\theta})\|\leq\frac{\eps\sigma_\min}{M^4e^{4\|\tilde{\theta}\| b}}
    ,A\right)\\
    &\geq 1-\alpha\exp\left\{-C(\eps^2\wedge \gamma)n_\min\right\}
\end{align*}
where 
\begin{align*}
    \alpha&\coloneqq d+5e^{1/4}\\
    C&\coloneqq\frac{1}{8}\min\left\{\frac{4\sigma_{\min}^2}{81b^2M^4e^{8\|\tilde{\theta}\| b}\mu_2^4}, 
    \frac{\eps^2\sigma_\min^2}{81b^2 M^6e^{8\|\tilde{\theta}\| b}\mu_2^2},
    \frac{\eps^2\sigma^2_\min}{324M^6e^{8\|\tilde{\theta}\| b}\|\mu_1\|^2}
    \right\}\\
    \gamma&\coloneqq \frac{1}{8C}\min\left\{\frac{\mu_2^2}{4b^2},~\frac{\sigma_\min}{4b^2},~1\right\}\\
\end{align*}
As $\hat L$ is convex for any sample we know that if $\hat\theta'$ is in the interior of $B_1(\tilde{\theta})$ then $\hat\theta'=\hat\theta$. Therefore, for any $0<\eps<\eta$ the same result holds replacing $\hat\theta'$ with $\hat\theta$. 

hence now gives us that for $\frac{\log(\alpha/\delta)}{C n_\min}\leq\gamma\wedge\|\tilde\theta\|$ w.p. $1-\delta$
\begin{align*}
    \|\hat\theta-\tilde{\theta}\|\leq \sqrt{\frac{\log(\alpha/\delta)}{C n_\min}}
\end{align*}
If we assume $\delta\leq\half$ then $\log(\alpha/\delta)\leq(\log(\alpha)+1)\log(1/\delta)$. Therefore we have that provided $n_\min\geq\frac{(\log(\alpha)+1)\log(1/\delta)}{\|\tilde\theta\|\wedge\gamma}$ w.p. $1-\delta$,
\begin{align*}
    \|\hat\theta-\tilde{\theta}\|\leq \sqrt{\frac{(\log(\alpha)+1)\log(1/\delta)}{C n_\min}}.
\end{align*}
Finally, taking $C_0\coloneqq\frac{\log(\alpha)+1}{\min\{C,\|\tilde\theta\|,\gamma\}}$ gives our desired result.
\end{proof}

\subsubsection{Bound on KLIEP Normalising Constant}\label{app:KLIEP_N_bound}
We now give a finite sample bound for the normalising constant calculated in KLIEP
\begin{corollary}\label{cor:KLIEP_N_bound}
Define $N^*,\hat N$ to be
\begin{align*}
    N^*&\coloneqq\E[r_{\tilde{\theta}}(Z^0)]\\
    \hat N &\coloneqq \frac{1}{m_0}\sum_{i=1}^{m_0}r_{\hat\theta}(Z_i^0)
\end{align*}
with $\hat \theta$, $\tilde{\theta}$, $r_{\tilde{\theta}}$ defined as before. Then we have that for $\delta\in(0,1/2]$, provided $n_\min>C_N\log(1/\delta)$,
\begin{align*}
    \prob\left(\|N^*-\hat N\|\leq\sqrt{\frac{C_N\log(1/\delta)}{n_\min}}\right)&\geq 1-\delta
\end{align*}
with $C_N$ a constant depending upon $C_0,b,\|\tilde\theta\|$.

\end{corollary}
\begin{proof}
We condition upon the same events we condition upon in Theorem \ref{thm:kliep_bound}. Specifically for $\delta\in(0,1/2]$, provided $n_\min>2C_0\log(1/\delta)$, w.p. at least $1-\delta$ the following two conditions hold:
\begin{align*}
\|\hat\theta-\tilde\theta\|&\leq\sqrt{\frac{C_0\log(1/\delta)}{n_\min}}&
|\mu_2-W_2|&<\frac{\eps\sigma_{\min}}{18M^2e^{4\eta b}\|\mu_1\|}\\
\end{align*}
with $\eps=\sqrt{C_0\log(\alpha/\delta)}$.
The first conditions implies
\begin{align*} 
    |r_{\tilde\theta}(z)-r_{\hat\theta}(z)|&\leq e^{\|\tilde\theta\|b}\sqrt{\frac{2C_0\log(1/\delta)}{n_\min}}
\end{align*}

Therefore we have that
\begin{align*}
|N^*-\hat N|&\leq\left|\E[r_{\tilde{\theta}}(Z^0)]-\frac{1}{n_0}\sum_{i=1}^n r_{\tilde{\theta}}(Z^0_i)\right|+ \left|\frac{1}{n_0}\sum_{i=1}^{n_0}r_{\hat \theta}(Z^0_i)-\frac{1}{n_0}\sum_{i=1}^{n_0}r_{\tilde{\theta}}(Z^0_i)\right|\\
&\leq\sqrt{\frac{C_1\log(1/\delta)}{n_\min}}+ e^{\|\tilde\theta\|b}\sqrt{\frac{C_0\log(1/\delta)}{n_\min}}
\intertext{with}
    C_1&\coloneqq\frac{\sigma_\min\sqrt{C'_0}(\log(\alpha)+1)}{18M^2e^{4\|\tilde\theta\|b}\|\mu_1\|}
\end{align*}
Taking $C_N=(\sqrt{C_0}+e^{\|\tilde\theta\|b}\sqrt{C_1})^2$ gives our desired result.
\end{proof}
\subsection{Proof of Lemma \ref{lemma:main_prob}}
\label{app:proof_prob}
\begin{proof}
Throughout, unless stated otherwise, all integrals are taken w.r.t. $\mu$ (note that for any integral over some subset $A$ of $\cal Z$ this is equivalent to taking the integral w.r.t $\nu$.) 

First we prove that
\begin{align*}
    p''(z)&=(1-\varphi(x))p(z)
\end{align*}
for $\mu$ almost every $z\in\cal Z$.
by showing that for all $A\in \cal B_{\cal Z}$
\begin{align*}
    \prob(X\in A)&=\int_A (1-\varphi(x))p(x)\nu(\diff x)
    \intertext{Indeed}
    \int_A (1-\varphi(x))p(x)\diff x
    &=\E[\one\{Z\in A\}(1-\varphi(Z))]\\
    &=\E[\one\{Z\in A\}]-\E[\one\{Z\in A\}\one\{X=\NaN\}]\quad\text{by definition of $\varphi(Z)$}\\
    &=\E[\one\{Z\in A\}\one\{X\neq \NaN\}]\quad\text{by linearity of expectation}\\
    &=\prob(Z\in A,X\neq\NaN)\\
    &=\prob(Z\in A,X=Z)\\
    &=\prob(X\in A)
\end{align*}
We now have that for any $A\in\cal B_{\cal Z}$,
\begin{align*}
    \prob(X\in A|X\neq\NaN)&=\prob(X\neq\NaN)^{-1}\prob(X\in A)\\
    &=\prob(X\neq\NaN)^{-1}\int_A p''(x)\nu(\diff x)\\
    &=\prob(X\neq\NaN)^{-1}\int_A (1-\varphi(x)p(x)\mu(\diff x)
\end{align*}
and so $p'=\prob(X\neq\NaN)^{-1}\cdot (1-\varphi)\cdot p$.

We also have that $p$ can be extended to a density over $\cal X$ by taking $p(\NaN)=0$. This combined with the above result gives us that
\begin{align*}
    \frac{p(x)}{p''(x)}\coloneqq\frac{\one\{x\neq\NaN\}}{1-\varphi(x)}
\end{align*}
for $\nu$ almost every $x\in\cal X$.

Hence we can use importance weighting to get our desired result.
\end{proof}

\subsection{Proof M-KLIEP Finite Sample Bounds}\label{app:MKLIEP_bound}
\begin{lemma}\label{lemma:imp_weight_vec_bern}
Let $W,V$ be two RVs over $\R^d$ with p.d.f.s $p_W$ $p_V$ w.r.t. some measure $\nu$. 
Assume that $g\coloneqq\frac{p_W}{p_V}$ is well defined and $g(w)\leq\alpha$ for all $w\in\R^d$. Finally suppose that $\|W\|,\|V\|\leq b$ a.s. Then we have that for any $0<\eps<b$
\begin{align*}
    \prob\left(\left|\E[W]-\frac{1}{n}g(V)V\right|\leq\eps\right)\geq1-e^{1/4}\exp\left\{-\frac{\eps^2 n}{8\alpha b^2}\right\}
\end{align*}
\end{lemma}
\begin{proof}
This proof is a direct corollary of Lemma \ref{lemma:vec_bern} First we clearly have that $\|g(V)V\|\leq\alpha b$ a.s.. Second we have that 
\begin{align*}
    \E(\|g(V)V\|^2)&= \E[g(V)^2\|V\|^2]\\
    &=\E[g(W)\|W^2\|]\\
    &\leq \alpha \E[\|W\|^2]\\
    &\leq a b^2.
\end{align*}
Hence as $\E[g(V)V-\E[W]]=0$, we can use the Bernstein inequality to get
\begin{align*}
     \prob\left(\left|\E[W]-\frac{1}{n}g(V)V)\right|\leq\eps\right)\geq1-e^{1/4}\exp\left\{-\frac{\eps^2 n}{8\alpha b}\right\}
\end{align*}
\end{proof}
\begin{remark}
If we want to remove the requirement that $\eps<b$ we can re-write this as
\begin{align*}
        \prob\left(\left|\E[W]-\frac{1}{n}g(V)V\right|\leq\eps\right)\geq1-e^{1/4}\exp\left\{-\left(\frac{\eps^2}{b^2}\wedge 1\right)\frac{n}{8\alpha}\right\}
\end{align*}
\end{remark}

\begin{lemma}
\label{lemma:mcov_bound}
Let $W, V$ be two RVs over $\R^d$, with densities $p_W,p_V$ respectively s.t. $\|W\|,\|V\|<b$ a.s.. Assume that $g\coloneqq\frac{p_W}{p_V}:\R^d\rightarrow R^p$ is well defined with $g(w)<\alpha$ for all $w\in\R^d$. . Let $V_1,\dotsc,V_n$ be IID copies of $V$.
Now define the importance weighted sample covariance by of $W$ estimated by $\{V_i\}_{i\in[n]}$ 
\begin{align*}
    \widehat\var(W;\{V_i\}_{i\in[n]})\coloneqq\left(\frac{1}{n}\sum_{i=1}^n g(V_i) V_i V_i^\top\right)\left(\frac{1}{n}\sum_{i\in[n]}g(V_i)\right)-\left(\frac{1}{n}\sum_{i=1}^n g(V_i) V_i\right)\left(\frac{1}{n}\sum_{i=1}^n g(V_i) V_i\right)^\top
\end{align*}
then provided $\sigma_\min\leq 4 b^2$
\begin{align*}
    \prob\left(\lambda_\min(\widehat\var(W;\{V_i\}_{i\in[n]})\geq\half\sigma_\min \right)&
    \geq 1-(d+2e^{1/4})\exp\{\frac{a'\sigma_\min n}{\alpha b^2}\}
\end{align*}
where $a'\coloneqq\frac{(2-\sqrt{3})^2}{32}$ and $\sigma_\min\coloneqq\lambda_\min(\var(W))$.
\end{lemma}
\begin{proof}
Define the centred RVs $Y_1,\dotsc, Y_n$ by $Y_i \coloneqq V_i-\E[W_i]$ and the random matrices $R_n, S_n$ by
\begin{align*}
    S_n&\coloneqq\frac{1}{n}\sum_{i=1}^n g(V_i) Y_i Y_i^\top \\
    R_n&\coloneqq\left(\frac{1}{n}\sum_{i=1}^n g(V_i) Y_i\right)\left(\frac{1}{n}\sum_{i=1}^n g(V_i) Y_i\right)^\top\\
    a_n&\coloneqq\frac{1}{n}\sum_{i\in[n]}g(V_i)
\end{align*}
so that $a_n R_n-S_n=\widehat\var(W;\{Y_i\}_{i\in[n]})$. Simple algebraic manipulation gives us that 
$a_n R_n-S_n=\widehat\var(W;\{V_i\}_{i\in[n]})$.

Firstly we can use the matrix Chernoff inequalities \citep{Tropp2015}, alongside the fact that $\E[S_n]=\var(W)$ and $\lambda_\max(\frac{1}{n}g(V_i) Y_i Y_i^\top)\leq\frac{\alpha b^2}{n}$ to get
\begin{align*}
        \prob\left(\lambda_\min(S_n)\geq\frac{\sqrt{3}}{2}\sigma_\min\right)&\geq1-d\exp\left\{-\frac{a\sigma_\min n}{\alpha b^2}\right\}\\
\end{align*}

Now note that for any $x\in\R^d$ $\lambda_\max(xx^\top)=\|x\|^2$ so that $\lambda_\max(R_n)=\|\frac{1}{n}Y_i\|^2$. Hence, as $\E[g(V_i)Y_i]=0$, we can use vector Bernstein bounds to get that provided $\sigma_\min\leq 4b^2$
\begin{align*}
    \prob\left(\lambda_\max(R_n)\leq\frac{1}{4}\sigma_\min\right)&=\prob\left(\left\|\frac{1}{n}\sum_{i=1}^n g(V_i) Y_i\right\|<\half\sqrt{\sigma_\min}\right)\\
    &\geq1-e^{1/4}\exp\left\{-\frac{\sigma_\min n}{32\alpha b^2}\right\}.\\
\end{align*}
Finally by Lemma \ref{lemma:imp_weight_vec_bern},
\begin{align*}
    \prob\left(a_n\geq\frac{\sqrt{3}}{2}\right)\geq 1-e^{1/4}\exp\left\{\frac{(2-\sqrt{3})^2 n}{32\alpha}\right\}
\end{align*}
Taking the intersection of these three events we get that w.p. at least $1-(d+2e^{1/4})\exp\left\{-\left(\frac{a'\sigma_\min}{b^2}\wedge 1\right)\frac{n}{8\alpha}\right\}$
\begin{align*}
    \lambda_\min(\widehat\var(W;\{(V_i,g(V_i))\}_{i\in[n]})&=\lambda_\min(a_n S_n-R_n)\\
    &\geq\lambda_\min(a_nS_n)-\lambda_\max(R_n)\\
    &=a_n\lambda_\min(S_n)-\lambda_\max(R_n)\\
    &\geq\frac{\sqrt{3}}{2}\frac{\sqrt{3}}{2}\sigma_\min-\frac{1}{4}\sigma_\min\\
    &=\half\lambda_\min(\var(W))
\end{align*}
where $a'\coloneqq\frac{(2-\sqrt{3})^2}{4}$
\end{proof}

We can now state the Lemma which gives rise to Theorem \ref{thm:kliep_bound}.
\begin{lemma}
Let $r_\theta$ take the log-linear form and for
for $\theta\in\R^d$ define and $L, \hat L'$ by
\begin{align*}
    L(\theta)\coloneqq-\E[r_\theta(Z^1)]+\log\E[r_\theta(Z^0)]
    \hat L'(\theta;D')\coloneqq&-\frac{1}{n_1}\sum_{i=1}^{n_1}\frac{\one\{X_i^1\neq\NaN\}}{1-\varphi^1(X_i^1)}\log r_\theta(X_i^1)\\
    &{}+
    \log \frac{1}{n_0}\sum_{i=1}^{n_0}\frac{\one\{X_i^0\neq\NaN\}}{1-\varphi^0(X_i^0)}r_\theta(X_i^0).
\end{align*}
Finally, define the constant $\tilde \theta$ and the RV $\hat \theta$, by
\begin{align*}
    \tilde\theta&\coloneqq\argmax{\theta\in\R^d}L(\theta)\\
    \hat\theta&\coloneqq\argmax{\theta\in\R^d}\hat L'(\theta,D').
\end{align*}

Now suppose that 
 $\|\varphi^0\|_{\infty}<1, \|\varphi^1\|_{\infty}<1$, $\|f\|_{\infty}<\infty$ and additionally let $\sigma_\min\coloneqq\lambda_\min(\var(f(Z^0)))>0$

Then we have that for any $\delta<\half$ and $\effectiveSampleSize\min\geq C'_0\log(1/\delta)$,
\begin{align*}
    \prob\left(\|\hat\theta'-\tilde{\theta}\|\leq\sqrt{\frac{C'_0\log(1/\delta)}{\effectiveSampleSize}}\right)\geq 1-\delta
\end{align*}
where $C'_0$ a constant depending upon $\|f\|_{\infty},\sigma_\min,d,\|\tilde\theta\|$.
\end{lemma}
\begin{proof}
Firstly we define the population version of $\hat L'$ by $L'$ as follows
\begin{align*}
    L'(\theta)=\E\left[\frac{\one\{X^1\neq\NaN\}}{1-\varphi^1(X^1)}\log r(X^1)\right]-\log\E\left[\frac{\one\{x^0\neq\NaN\}}{1-\varphi^0(X^0)}\right].
\end{align*}
Then from Lemma \ref{lemma:main_prob} we have that $L(\theta)=L'(\theta)$.
As a result we can re-define $\tilde{\theta}$ to be 
\begin{align*}
    \tilde{\theta}\coloneqq\argmin{\theta\in\R^d}L'(\theta)
\end{align*}

For proof of convexity we note that again the term involving $\cal D^{'+}$ (our sample from $\{X_i^1\}_{i=1}^{n_1}$) is linear in $\theta$ and the second term is convex for the same reason as $\hat L$ is with the caveat that we replace $g:\R^{n_0}\rightarrow \R$ by $g_{\cal D^0}':\R^{n_0}\rightarrow\R$ defined by
\begin{align*}
    g(x)=\log\sum_{i=1}^n w_i \exp\{x_i\}\quad\text{where}\\
    w_i^0\coloneqq\frac{\one\{x_i^0\neq\NaN\}}{1-\varphi^0(x_i^0)}.
\end{align*}
As each $w_i>0$, this modified log-sum-exp is also convex by the same argument which makes the unmodified log-sum-exp function convex. That is, 
\begin{align*}
    \nabla^2_x g'_{\cal D^0}(z)=&\frac{\text{diag}(u)(1^\top u)-uu^\top}{(1^\top u)^2}
    \intertext{where $u_i=w_i^0\exp(z_i)$. If we then let $v\in\R^{n_0}$, we get}
    v^\top\nabla^2_x g'_{\cal D^0}(z)&=\frac{v^\top\text{diag}(u)v(1^\top u)-v^\top uu^\top v}{(1^\top u)^2}\\
    &=\frac{(\sum_{i=1}^{n_0}v_i^2 u_i)(\sum_{i=1}^{n_1}u_i)-(\sum_{i=1}^{n_0}v_iu_i)}{(1^\top u)}>0
\end{align*}
where the final inequality is given by the Cauchy-Schwartz inequality.

For ease of notation we define $b\coloneqq\|f\|_{\infty}$ and $M\coloneqq\exp{\|\tilde\theta\|b}$.
We now aim to bound $\lambda_\min(\nabla^2_\theta\hat L'(\bar\theta))$ from below. To this end we have
\begin{align*}
    \nabla^2_\theta \hat L(\theta;D)&\coloneqq\frac{\frac{1}{n_0^2}\sum_{i,j}w_i^0 w_j^0 r_{\bar\theta}(X^0_i)r_{\bar\theta}(X^0_j)[f(X_i^0)-f(X_j^0)][f(X_i^0)-f(X_j^0)]^\top} {\frac{1}{n^2}\sum_{i,j}w_i^0 w_j^0r_{\bar\theta}(X^0_i)r_{\bar\theta}(X^0_j)}\\
    &\succeq\frac{\frac{1}{n_0^2}\sum_{i,j}M^{-2}e^{-2\eta b} w_i^0 w_j^0 [f(X_i^0)-f(X_j^0)][f(X_i^0)-f(X_j^0)]^\top}{\frac{1}{n^2}\sum_{i,j}e^{2\eta b}w_i^0 w_j^0r_{\bar\theta}(X^0_i)r_{\bar\theta}(X^0_j)}\\
    &\succeq\frac{1}{M^2e^{4\eta b}}\frac{\frac{1}{n_0^2}\sum_{i,j} w_i^0 w_j^0 [f(X_i^0)-f(X_j^0)][f(X_i^0)-f(X_j^0)]^\top}{\frac{1}{n_0}\sum_{i,j}w_i^0 w_j^0 r_{\tilde{\theta}}(X_i^0)r_{\tilde{\theta}}(X_j^0)}\\
    &=:\frac{2}{M^4e^{4\eta b}}\frac{\widehat\var_0\bigg[f(Z^0);\rsampmone,\{w_i^0\}_{i=1}^{n_0}\bigg]}{\left(\widehat\E^'\bigg[r_{\tilde{\theta}}(Z^0);\rsampmone\bigg]\right)^2}
\end{align*}
where, for square matrices $A,B$ we take $A\succeq B$, to mean $A-B$ is positives semi-definite.
We can then bound $\widehat\var_0$ and $\widehat\E'$ similarly to before to get that

\begin{align*}
    \prob\left(\lambda_\min(\widehat\var(f(Z),\rsampmone))\geq\half \sigma_\min\right)&\geq
    1-(d+2e^{1/4})\exp\{-\left(\frac{a'\sigma_\min}{b^2}\wedge1\right)\frac{n_0(1-\varphi^0_\max)}{8}\}\\
    \prob\left(\E^'\bigg[r_{\tilde{\theta}}(Z^0);\rsampmone\bigg]
    \leq\frac{3}{2}\E[r_{\tilde{\theta}}(X^0)]\right)&\geq1-e^{1/4} \exp\left\{-\left(\frac{\mu_2}{4M^2}\wedge 1\right)\frac{n_0(1-\varphi^0_\max)}{8}\right\}
\end{align*}
where again $\sigma_\min\coloneqq\lambda_\min(\var(f(Z^0))$, $\mu_2\coloneqq\E[r_{\tilde{\theta}}(Z^0)]$.

Therefore we have that
\begin{align}\label{eq:m_hess_bound}
    \prob\left(\underbrace{\lambda_\min(\nabla^2_\theta\hat L';(\theta;\rsampm))\geq\frac{4\sigma_\min}{9M^2e^{4\eta b}\mu_2^2}}_{\coloneqq A'}\right)\geq1-(d+3e^{1/4})\exp\left\{-\left(\frac{\mu_2}{4M^2}\wedge\frac{a'\sigma_\min}{b^2}\wedge1\right)\frac{n_0(1-\varphi^0_\max)}{8}\right\}
\end{align}

We now have all the bounds required to continue with the main body of the proof.
Let $\hat\theta'$ be defined as the solution to the following constrained optimisation problem.
\begin{align*}
    \min_{\theta\in B_\eta(\tilde{\theta})}\hat L(\theta,D')
\end{align*}

By and identical argument to Theorem \ref{thm:kliep_bound} we then get that 
\begin{align*}
\|\hat\theta'-\tilde{\theta}\|\leq\frac{1}{\lambda_\min(\nabla^2_\theta \hat L'(\bar\theta))}\|\nabla_\theta \hat L'(\tilde{\theta})\|.
\end{align*}
for some $\bar\theta\in B_\eta(\tilde{\theta})$.

Hence if $A'$ from (\ref{eq:m_hess_bound}) holds with our given $\bar\theta\in B_\eta(\tilde{\theta})$ we have that
\begin{align*}
    \|\hat \theta'-\tilde{\theta}\|&\leq\frac{9M^2e^{4\eta b} \mu_2^2} 
    {4\sigma_{\min}}\|\nabla_\theta\hat L(\tilde{\theta})\| \quad\text{a.s. .}
\end{align*}
From now on we will work on the event $A$ so that
bounding $\|\hat \theta-\tilde{\theta}\|$ simply requires us to bound $\|\nabla_\theta\hat L(\tilde{\theta})\|$.

We further get that $\|\nabla_\theta \hat L'(\tilde{\theta})\|\leq I'_1+I'_2$ where
\begin{align*}
    I_1'&=\left\|\E\left[\frac{\nabla_{\theta}r_{\tilde{\theta}}(Z^1)}{r_{\tilde{\theta}}(Z^1)}\right]-\frac{1}{n_1}\sum_{i=1}^{n_1} \frac{\one\{X_i^1\neq\NaN\}}{1-\varphi^1(X_i^1)}\frac{\nabla_\theta r_{\tilde{\theta}}(X_i^1)}{r_{\tilde{\theta}}(X_i^1)}\right\| \\
    I_2'&=\left\|\frac{\E[\nabla_\theta r_{\tilde{\theta}}(X^0)]}{\E[r_{\tilde{\theta}}(X^0)]}-
    \frac{\frac{1}{n_0}\sum_{i=1}^{n_0}\frac{\one\{X_i^0\neq\NaN\}}{1-\varphi^0(X_i^0)}\nabla_\theta r_{\tilde{\theta}}(X_i^0)}{\frac{1}{n_0}\sum_{i=1}^{n_0}\frac{\one\{X_i^0\neq\NaN\}}{1-\varphi^0(X_i^0)}r_{\tilde{\theta}}(X_i^0)}\right\|.
\end{align*}
Define RVs $W_1',W_2'$ 
\begin{align*}
    W_1'&\coloneqq\frac{1}{n_0}\sum_{i\in[n_0]}\frac{\one\{X^0_i\neq\NaN\}} {1-\varphi^0(X_i^0)}\nabla_\theta r_{\tilde{\theta}}(X_i^0)\\
    W_2'&\coloneqq\frac{1}{n_0}\sum_{i\in[n_0]}\frac{\one\{X^0_i\neq\NaN\}} {1-\varphi^0(X_i^0)} r_{\tilde{\theta}}(X_i^0)
\end{align*}
and $\mu_1\coloneqq\E[\nabla_\theta r_{\tilde{\theta}}(Z^0)]$. We can then bound $I'_2$ to get
\begin{align*}
    I'_2\coloneqq \leq \frac{\|\mu_1-W'_1\|}{\mu_2}+\frac{|\mu_2-W'_2|\|\mu_1\|}{\mu_2 W'_2}\\
\end{align*}

Thus to have  
\begin{align*}
    I'_1+I'_2&<\frac{4\sigma_{\min}}{\eps 9M^2e^{4\eta b}\mu_2^2},\\
    \intertext{it is sufficient to have}
    I'_1&<\frac{2\eps\sigma_{\min}}{9M^2e^{4\eta b}\mu_2^2},\\
    \|\mu_1-W'_1\|&<\frac{\eps\sigma_{\min}}{9M^2e^{4\eta b}\mu_2},\\
    |\mu_2-W'_2|&<\frac{\eps\sigma_{\min}}{18M^2e^{4\eta b}\|\mu_1\|}.
\end{align*}
provided $\eps<\frac{9M^2e^{4\eta b}\mu_2\|mu_1\|}{\sigma_\min}$.

The probabilities for these events are
\begin{align*}
 \prob\left(I'_1<\frac{2\eps\sigma_{\min}}{9M^2e^{4\eta b}\mu_2^2}\right)&
    \geq 1-e^{1/4}\exp\left\{-\frac{1}{8}\left(\frac{4\eps^2\sigma_{\min}^2}{81b^2M^4e^{8\eta b}\mu_2^4}\wedge 1\right)n_1(1-\varphi^1_\max)\right\}\\
    \prob\left(\|\mu_1-W'_1\|\leq\frac{\eps\sigma_{\min}}{9M^2e^{4\eta b}\mu_2}\right)&
    \geq 1-e^{1/4}\exp\left\{-\frac{1}{8}\left(\frac{\eps^2\sigma_\min^2}{81b^2 M^6e^{8\eta b}\mu_2^2}\wedge 1\right)n_0(1-\varphi^0_\max)\right\}\\
    \prob\left(|\mu_2-W'_2|\leq\frac{\eps\sigma_\min}{18M^2e^{4\eta b}\|\mu_1\|}\right)&
    \geq 1-e^{1/4}\exp\left\{-\frac{1}{8}\left(\frac{\eps^2\sigma^2_\min}{324M^6e^{8\eta b}\|\mu_1\|^2}\wedge 1\right)n_0(1-\varphi^0_\max)\right\}
\end{align*}
Therefore by the same argument as Theorem \ref{thm:kliep_bound} we get
\begin{align*}
    \prob(\|\hat \theta-\tilde{\theta}\|\leq\eps)&\geq\prob\left(\|\nabla_{\theta}\hat L(\tilde{\theta})\|\leq\frac{\eps\sigma_{\min}}{M^4e^{4\eta b}},~A'\right)\\
    &\geq 1-\alpha'\exp\left\{-C'\left(\eps^2\wedge\gamma\right)\effectiveSampleSize\right\}.
\end{align*}
where
\begin{align*}
    \alpha'&\coloneqq d+6e^{1/4}\\
    C'&\coloneqq\frac{1}{8}\min\left\{\frac{4\sigma_{\min}^2}{81b^2M^4e^{8\|\tilde{\theta}\| b}\mu_2^4}, 
    \frac{\eps^2\sigma_\min^2}{81b^2 M^6e^{8\|\tilde{\theta}\| b}\mu_2^2},
    \frac{\eps^2\sigma^2_\min}{324M^6e^{8\|\tilde{\theta}\| b}\|\mu_1\|^2}
    \right\}\\
    \gamma'&\coloneqq \frac{1}{8C}\min\left\{\frac{a'\mu_2^2}{M^2},~\frac{\sigma_\min}{4b^2},~1\right\}\\
\end{align*}

By an identical argument to the proof of Theorem \ref{thm:kliep_bound} in Section \ref{app:KLIEP_bound} this gives us that for
us that for $\delta\leq\half$, provided $\effectiveSampleSize\geq\frac{\log(\alpha')\log(1/\delta)}{\gamma'\wedge\|\tilde\theta\|}$ then
\begin{align*}
    \prob\left(\|\tilde\theta-\hat\theta\|\leq\sqrt{\frac{\log(\alpha')\log(1/\delta)}{C' \effectiveSampleSize}}\right)\geq1-\delta.
\end{align*}
Now taking $C'_0\coloneqq\frac{\log(\alpha')}{\min\{C',\|\tilde\theta\|,\gamma'\}}$ gives our desired result.
\end{proof}

\subsubsection{Bound on Normalisation term in M-KLIEP}
\label{app:MKLIEP_N_bound}
\begin{corollary}\label{cor:MKLIEP_N_bound}
Define $N^*,\hat N'$ to be
\begin{align*}
    N^*&\coloneqq\E[r_{\tilde{\theta}}(Z^0)]\\
    \hat N &\coloneqq \frac{1}{m_0}\sum_{i=1}^{m_0}\frac{\one\{X_i^0\neq\NaN\}}{1-\varphi^0(x_i^0)}r_{\hat\theta'}(X_i^0)
\end{align*}
with $\hat \theta'$, $\tilde{\theta}$, $r_{\tilde{\theta}}$ defined as before. Then we have that for $\delta\in(0,1/2]$, provided $n_\min>C'_N\log(1/\delta)$,
\begin{align*}
    \prob\left(\|N^*-\hat N\|\leq\frac{1}{1-\|\varphi^0\|_\infty}\sqrt{\frac{C'_N\log(1/\delta)}{\effectiveSampleSize}}\right)&\geq 1-\delta
\end{align*}
\end{corollary}
\begin{proof}
We work assuming the same events hold as in Theorem \ref{thm:MKLIEP_bound}.
Crucially we have that for $\delta\in(0,1/2]$, provided $\effectiveSampleSize\geq C'_0\log(1/\delta)$, w.p. at least $1-\delta$ the following two events hold
\begin{align*}
 \|\hat\theta-\tilde\theta\|&\leq\sqrt{\frac{2C'_0\log(1/\delta)}{n_\min}} &
 |\mu_2-W'_2|&<\frac{\eps\sigma_{\min}}{18M^2e^{4\eta b}\|\mu_1\|}\\
\end{align*}
We will refer to these inequalities as $A$ and $B$ respectively. $A$ then implies,
\begin{align*}
    \left|\frac{\one\{x\neq\NaN\}}{1-\varphi^0(x)}r_{\tilde\theta}(x)-\frac{\one\{x\neq\NaN\}}{1-\varphi^0(x)}r_{\hat\theta}(z)\right|\leq \frac{1}{1-\|\varphi^0\|_{\infty}}e^{\|\tilde\theta\|b}\sqrt{\frac{C'_0\log(1/\delta)}{n_\min}}
\end{align*}
Therefore A and B together give
\begin{align*}
|N^*-\hat N'|&\leq\left|\E[r_{\tilde{\theta}}(Z^0)]-\frac{1}{n_0}\sum_{i=1}^n r_{\tilde{\theta}}(Z^0_i)\right|+ \left|\frac{1}{n_0}\sum_{i=1}^{n_0}r_{\hat \theta'}(Z^0_i)-\frac{1}{n_0}\sum_{i=1}^{n_0}r_{\tilde{\theta}}(Z^0_i)\right|\\
&\leq \sqrt{C'_1\frac{\log(1/\delta)}{\effectiveSampleSize}} +\frac{1}{1-\|\varphi^0\|_{\infty}}e^{\|\tilde\theta\|b}\sqrt{\frac{C'_0\log(1/\delta)}{\effectiveSampleSize}}
\intertext{with}
    C_1&\coloneqq\frac{\sigma_\min\sqrt{C'_0}(\log(\alpha)+1)}{18M^2e^{4\|\tilde\theta\|b}\|\mu_1\|}.
\end{align*}
Therefore, taking $C'_N=(\sqrt{C'_0}+\exp\{\|\tilde\theta\|b\}\sqrt{C'_1})^2$ gives our desired result.
\end{proof}
\subsection{M-KLIEP Lower Bound}\label{app:KLIEP_L_bound}
For each $\zeta \in (0,\infty)$ let's define a density $f_\zeta$ of a measure supported on $[0,1]$ by
\begin{align*}
f_\zeta(y)=\frac{\zeta e^{\zeta y}}{e^\zeta-1},
\end{align*}
for $y \in [0,1]$, and $f_\zeta(y)=0$ for $y \notin [0,1]$. In addition, let $f_0$ denote the uniform density on $[0,1]$.

\begin{lemma}\label{ref:klBound} For all $\zeta \in [0,2]$ we have 
\begin{align*}
\mathrm{KL}(f_0,f_\zeta) \leq \frac{\zeta^2}{20}.
\end{align*}
\end{lemma}
\begin{proof} Given any $\zeta \in (0,2]$ we have
\begin{align*}
\mathrm{KL}(f_0,f_\zeta) & = \int_0^1 \log\left(\frac{e^\zeta-1}{\zeta e^{\zeta y}}\right)dy = \log\left\lbrace \frac{\sinh(\zeta/2)}{(\zeta/2)}\right\rbrace \leq \log\left( 1+\frac{\zeta^2}{20}\right) \leq \frac{\zeta^2}{20},
\end{align*}
where the penultimate inequality uses the bound \cite[Lemma 3.3 (i)]{klen2010jordan}. The case where $\zeta=0$ is immediate.
\end{proof}

Let $P_0$ be the uniform distribution on $A_d:=[0,1/\sqrt{2}]\times [0,1/\sqrt{2(d-1)}]^{d-1}$, take $X^0=Z^0 \sim P_0$ and $C_d:=\sqrt{2^d\cdot (d-1)^{d-1}}$. Moreover, given $\zeta \in [0,2]$, $w_1 \in [0,1]$ and $\tau  \in \{0,1\}$ we write $P_1(\zeta,w_1,\tau)$ for the distribution on $X^1=(X^1_j)_{j\in [d]}$ constructed by choosing $Z^1$ with density $(z_j)_{j \in [d]}\mapsto C_d \cdot f_{\tau \cdot \zeta}(\sqrt{2}z_1)$ on $A_d$, and choosing $\varphi^1$ so that $\varphi^1(z) = w_1$ for all $z \in [0,1]^d$. We write $\overline{P}_{n_0,n_1}(\zeta,w_1,\tau)$ for the joint distribution on $\rsampm:=(\rsampm_0,\rsampm_1)$ where $\rsampm_0 \sim P_0^{n_0}$ and $\rsampm_1 \sim P_1(\zeta,w_1,\tau)^{n_1}$.

\begin{lemma}\label{ref:klBoundProduct} Given $\zeta \in [0,2]$, $n_1 \in \N$ and $w_1 \in [0,1]$ we have
\begin{align*}
\mathrm{KL}&\left\lbrace \overline{P}_{n_0,n_1}(\zeta,w_1,0),\overline{P}_{n_0,n_1}(\zeta,w_1,0)\right\rbrace  \leq \frac{n_1 (1-w_1)\zeta^2}{20}.
\end{align*}
\end{lemma}
\begin{proof} We have 
\begin{align*}
\mathrm{KL}&\left\lbrace P_1(\zeta,w_1,0),P_1(\zeta,w_1,1)\right\rbrace  = \int_{\mathcal{X}} \log\left( \frac{dP_1(\zeta,w_1,0)}{dP_1(\zeta,w_1,1)} \right) dP_1(\zeta,w_1,0)\\
& =(1-w_1) \cdot \int_{A_d} \log\left( \frac{C_d \cdot f_{0}(\sqrt{2} \cdot z_{1})}{ C_d \cdot f_{\zeta}(\sqrt{2} \cdot z_{1})} \right) \left(C_d \cdot f_{0}(\sqrt{2} \cdot z_{1})\right) dz_{1}\\
& =(1-w_1) \cdot \int_{[0,1]} \log\left( \frac{ f_{0}(y)}{ f_{\zeta}(y)} \right)  f_{0}(y) dy\\
&= (1-w_1) \cdot \mathrm{KL}(f_0,f_\zeta) \leq \frac{(1-w_1) \cdot \zeta^2}{20},
\end{align*}
where we used Lemma \ref{ref:klBound} in the final step. Hence, by the product rule for Kullback-Leibler divergence \cite[Chapter 2]{tsybakov2009introductiona} we have
\begin{align*}
\mathrm{KL}&\left\lbrace \overline{P}_{n_0,n_1}(\zeta,w_1,0),\overline{P}_{n_0,n_1}(\zeta,w_1,0)\right\rbrace \\ &= n_0 \cdot \mathrm{KL}\left(P_0,P_0\right) + n_1 \cdot \mathrm{KL}\left\lbrace P_1(\zeta,w_1,0),P_1(\zeta,w_1,1)\right\rbrace  \leq \frac{n_1(1-w_1)\zeta^2}{20}.
\end{align*}

\end{proof}

We are now ready to complete the proof of Theorem \ref{thm:drMissLB}.

\begin{proof}[Proof of Theorem \ref{thm:drMissLB}] We assume, without loss of generality, that $n_1w_1 \leq n_0 w_0$. Suppose $Z^0 \sim P_0$ and $Z^1$ has density $(z_j)_{j \in [d]}\mapsto \sqrt{2} \cdot f_{\tau \cdot \zeta}(\sqrt{2} \cdot z_1)$ for some $\zeta \in [0,1/\sqrt{2}]$, $\tau  \in \{0,1\}$ then the density ratio between the densities of $Z^1$ and $Z^0$ is proportional to $e^{\theta(\zeta,\tau)^\top z}$ on $[0,1]^d$ where $\theta(\zeta,\tau)=(\sqrt{2} \cdot \tau \cdot \zeta ,0,\ldots,0)^\top \in \R^d$, and hence $\|\theta(\zeta,\tau)\| \leq 1$. Next, we convert $\hat{\theta}=(\hat{\theta}_j)_{j \in [d]}$ into an estimator $\hat{\tau}$ by $\hat{\tau}=\one\{\hat{\theta}_j>\zeta/\sqrt{2}\}$. Consequently, for $\tau \in \{0,1\}$ we have
\begin{align*}
\|\hat{\theta}-\theta(\zeta,\tau)\|_2&\geq |\hat{\theta}_1-\zeta \tau| \geq \frac{\zeta}{\sqrt{2}}\cdot \one\{\hat{\tau}\neq \tau\}.
\end{align*}
Hence, by \cite[Theorem 2.2(iii)]{tsybakov2009introductiona} for at least one $\tau \in \{0,1\}$ we have
\begin{align*}
\min_{\tau \in \{0,1\}}&\prob_{\rsampm \sim \overline{P}_{n_0,n_1}(\zeta,w_1,\tau)}\left\lbrace \|\hat{\theta}-\theta(\zeta,\tau)\|_2 \geq \frac{\zeta}{\sqrt{2}}  \right\rbrace \\& \geq \min_{\tau \in \{0,1\}}\prob_{\rsampm \sim \overline{P}_{n_0,n_1}(\zeta,w_1,\tau)}\left( \hat{\tau}\neq \tau \right) \geq \frac{1}{4}\cdot \exp\left(-\frac{n_1 (1-w_1)\zeta^2}{20}\right).
\end{align*}
To complete the proof we take 
\begin{align*}
\zeta:=\min\left\lbrace  \sqrt{\frac{20\log(1/(4\delta))}{\min\{n_0(1-w_0),n_1(1-w_1)\}}} ,\frac{1}{\sqrt{2}}\right\rbrace.
\end{align*}
\end{proof}

\subsection{Proof of Missing NP Classifier Finite Sample Bounds}
Before we can proof this result we need to state a (relatively trivial) result on conditional probabilities.

\begin{lemma}\label{lem:cond_prob}
Let $X,Y$ be RVs on measurable space $(\cal Z,\cal B_{\cal Z})$ and define events $F\in\sigma(X)$, and $~E_0,E_1\in\cal B_{\cal Z}$. Suppose that $F\cap E_0\subseteq E_1$ then
\begin{align*}
    \prob(F\cap \{\E[\one_{E_0} Y|X]>\E[\one_{E_1}Y|X]\})=0.
\end{align*}
\end{lemma}
\begin{proof}
Define the event $C\coloneqq\{\E[\one_{E_0}Y|X]>\E[\one_{E_1}Y|X]\}$ then by construction, as $F\cap C\in\sigma(X)$, this gives
\begin{align*}
    \E[\one_{E_0\cap F\cap C} Y]&=\E[\E[\one_{E_0}Y|X]\one_{F\cap C}]\\
    &\geq\E[\E[\one_{E_1}Y|X]\one_{F\cap C}]\\
    &=\E[\one_{E_1\cap F\cap C}Y]
\end{align*}
where equality holds throughout if and only if $\prob(F\cap C)=0$.
Alternatively, as $F\cap E_0\subseteq E_1$ we have that $F\cap E_0\cap C\subseteq F\cap E_1\cap C$
so that
\begin{align*}
    \E[\one_{F\cap C\cap E_0}Y]&\leq\E[\one_{F\cap C\cap E_1}Y].
\end{align*}
Therefore equality holds and $\prob(F\cap C)=0$.
\end{proof}

We can now go on to prove Theorem \ref{thm:NP_pow_bound}
\begin{proof}[Proof of Theorem \ref{thm:NP_pow_bound}] This proof is adapted from \cite{Tong2013}.

For notational simplicity we define functions $\hat h,\tilde h:\cal Z\rightarrow\R$ by
\begin{align*}
    \hat h(z)&\coloneqq g\circ r_{\hat\theta} \\
    \tilde h(z)&\coloneqq g\circ r_{\tilde\theta}
\end{align*}
so that $R_1(\hat \phi),~R_1(\tilde\phi)$ can be re-written as
\begin{align*}
    \hat\phi=\prob(\hat h(Z^1)\leq\hat C_{\alpha,\delta,\hat h})\\
    \tilde\phi=\prob(\tilde h(Z^1)\leq C^*_{\alpha,\tilde h})
\end{align*}
Define events $G_0,G_1$ by
\begin{align*}
    G_0\coloneqq\{\tilde h(Z^0)>C^*_{\alpha,\tilde h},~\hat h(Z^0)\leq\hat C_{\alpha,\delta,\hat h}\}\\
    G_1\coloneqq\{\tilde h(Z^0)\leq C^*_{\alpha,\tilde h},~\hat h(Z^0)>\hat C_{\alpha,\delta,\hat h}\}\\
\end{align*}
We now have that
\begin{align*}
    R_1(\hat\phi)-R_1(\tilde\phi)=\E\left[\one G_0 \cup G_1 
    \left| \tilde{h(Z^0)} - C^*_{\alpha,\tilde{h}}\right|\right]
    +2C^*_{\alpha,\tilde h}(\prob( \tilde h(Z^0)\geq C^*_{\alpha,\tilde h})-\prob(\hat h(Z^0)\geq\hat C_{\alpha,\delta,\hat h}))
\end{align*}

For ease of notation we will define $\Delta_{\effectiveSampleSize^0,\delta}$ as in Lemma \ref{lemma:NP_miss_typeI}, that is,
\begin{align*}
    \Delta_{\effectiveSampleSize^0,\delta}&\coloneqq\sqrt{\frac{16\log(1/\delta)}{\effectiveSampleSize^0}}
\end{align*}
Now define 2 events $E_0,~E_1$ as follows
\begin{align*}
    E_0&\coloneqq\{\|\hat \theta-\tilde{\theta}\|\leq\eps''\}\\
    E_1&\coloneqq\{\prob(\hat h(X^0))\geq\hat C_{\alpha,\delta,\hat h}|\rsampm)\geq\alpha-2\Delta_{\effectiveSampleSize^0,\delta}\}
    \intertext{where}
    \eps''&\coloneqq\sqrt{\frac{C'_0\log(1/\delta)}{\effectiveSampleSize}}.
\end{align*}

Theorem \ref{thm:MKLIEP_bound} and Lemma \ref{lemma:NP_miss_typeI} give us that both these events occur w.p. $1-\delta$ so that $\prob(E_0\cap E_1)\geq1-2\delta$. We now aim to show that, 
\begin{align*}
    E_0\cap E_1 \subseteq \{R_1(\hat\phi)-R_1(\tilde\phi)<\eps'\}.
\end{align*}

We immediately have that 
\begin{align*}
    E_1\subseteq\{\prob( \tilde h(Z^0)\geq C^*_{\alpha,\tilde h})-\prob(\hat h(Z^0)\geq\hat C_{\alpha,\delta,\hat h}|\rsampm)\leq2\Delta_{\delta,\effectiveSampleSize}\}
\end{align*}

Additionally, we note that as $\|\nabla_\theta g\circ r_{\theta}(z)\|\leq L$ for all $z,\theta$, so that
\begin{align*}
E_0\subseteq\{\sup_{z\in\cal Z}\|\tilde h(z)-\hat h(z)\|^2\leq\eps''L\}.
\end{align*}

We now aim to bound $\hat C_{\alpha,\delta,g}$ above.
Taking the intersection of these two events we get that
\begin{align*}
    E_0\cap E_1\subseteq\{\prob(\tilde h(Z^0)\geq\hat C_{\alpha,\delta,\hat h}-\eps''L|\hat C_{\alpha,\delta,\hat h})\geq\alpha-2\Delta_{\effectiveSampleSize^0,\delta}\}
\end{align*}
On the other hand, the lower bound on our condition gives that $2\Delta_{\effectiveSampleSize^0,\delta}>a$,
\begin{align*}
    2\Delta_{\effectiveSampleSize^0,\delta}&\leq\prob[C^*_{\alpha,\tilde h}<\tilde h(Z^0)\leq C^*_{\alpha,\tilde h}+(2B_0\Delta_{\effectiveSampleSize^0,\delta})^{1/\gamma_0}]\\
    &=\prob[\tilde h(X^0)\leq C^*_{\alpha,\tilde h}+(2B_0\Delta_{\effectiveSampleSize^0,\delta})^{1/\gamma_0}]-
    \prob[\tilde h(X^0)\leq C^*_{\alpha,\tilde h}]\\
    &=\prob[\tilde h(X^0)\leq C^*_{\alpha,\tilde h}+(2B_0\Delta_{\effectiveSampleSize^0,\delta})^{1/\gamma_0}]-(1-\alpha). 
\end{align*}
Combining these two results we get that
\begin{align*}
    E_0\cap E_1\subseteq\{\prob[\tilde h(X^0)\geq C^*_{\alpha,\tilde h}+(2B_0\Delta_{\effectiveSampleSize^0,\delta})^{1/\gamma_0}]\leq \alpha-2\Delta_{\effectiveSampleSize^0,\delta} \leq\prob[\tilde h(X^0)\geq \hat C_{\alpha,\delta,\tilde h}-\eps'L|\hat C_{\alpha,\delta,\hat h}]\},
\end{align*}
and hence, $E_0\cap E_1\subseteq\{ \hat C_{\alpha,\delta,g}<C^*_{\alpha,\tilde h}+\eps''L+(2B_0\Delta_{\effectiveSampleSize^0,\delta})^{1/\gamma_0}\}$ which implies that
\begin{align*}
    E_0\cap E_1\subseteq\{C^*_{\alpha,\tilde h}\leq \tilde h(X^0)\leq C^*_{\alpha,\tilde h}+2\eps''L+(2B_0\Delta_{\effectiveSampleSize^0,\delta})^{1/\gamma_0}\}=:G'_0.
\end{align*}

We can now use this to bound the expectation. Indeed using Lemma \ref{lem:cond_prob} we have that
\begin{align*}
    E_0\cap E_1\subseteq\{\E[\one\{G_0\} |\tilde h(Z^0)-C^*_{\alpha,\tilde h}|~|\rsampm]&\leq\E[\one\{G'_0\}|\tilde h(Z^0)-C^*_{\alpha,\tilde h}| |\rsampm]\}\\
\end{align*}
We we also have
\begin{align*}
    \E[\one\{G'_0\}|\tilde h(Z^0)-C^*_{\alpha,\tilde h}| |\rsampm]
    &\leq \left(2\eps''L+(2B_0\Delta_{\effectiveSampleSize^0,\delta})^{1/\gamma_0}\right)\prob(G'_0)\\
    &\leq B_2(2\eps''L+(2B_0\Delta_{\effectiveSampleSize^0,\delta})^{1/\gamma_0})^{\gamma_1+1}
    \intertext{with}
    B_2&\coloneqq \min\left\{\frac{1}{B_0a^{\gamma_1/\gamma_0}}, B_1\right\}
\end{align*}
where the final step uses our probability lower bound. 

By an identical argument again 
\begin{align*}
E_0\cap E_1\subseteq\{\E[\one_{G_1} |\tilde h(Z^0)-C^*_{\alpha,\tilde h}|~|\rsampm] \leq B_2(2\eps''L+(2B_0\Delta_{\effectiveSampleSize^0,\delta})^{1/\gamma_0})^{\gamma_1+1}\}
\end{align*}
and so we get 
\begin{align*}
    E_0\cap E_1\subseteq\{R_1(\hat\phi)-R_1(\tilde\phi)\leq 2B_2(2\eps''L+(2B_0\Delta_{\effectiveSampleSize^0,\delta})^{1/\gamma_0})^{\gamma_1+1}+2C^*_{\alpha,\tilde h}\Delta_{\delta,\effectiveSampleSize^0}\}.
\end{align*}
Taking $C_2\coloneqq4B_2(L\vee B_0)(\gamma_1+1)+2C^*_{\alpha,\tilde h}$ therefore gives 
$E_0\cap E_1\subseteq\{R_1(\hat\phi)-R_1(\tilde\phi)\leq\eps'\}$ and hence,
\begin{align*}
    \prob(R_1(\hat\phi)-R_1(\tilde\phi)\leq\eps)&\leq\prob(R_1(\hat\phi)-R_1(\tilde\phi)\leq\eps,~E_0\cap E_1)\\
    &=\prob(E_0\cap E_1)\\
    &=1-2\delta.
\end{align*}
\end{proof}

\section{SUPPLEMENTARY METHODS}
\subsection{\texorpdfstring{$f$}{f}-Divergence based DRE}\label{app:Other_DRE}
We now present another example DRE which relates to the $f$-Divergence. First we state the following theorem on which these approaches are built
\begin{theorem}
Let $f:(0,\infty)\rightarrow \R$ be convex and lower-semicontinuous function and define $f':\R_+\rightarrow\R$ to be the derivative of $f$. Firstly, there exists unique $f^*:A\rightarrow \R$ with $A\subseteq\R$ by
\begin{align*}
    f^*(t)\coloneqq\sup_{u\in (0,\infty)}{u t - f(u)}
\end{align*}
Secondly, 
\begin{align*}
    \argmin{T\in\cal T}\E[T(Z^1)]-\E[f^*(T(Z^0))]=T^* 
\end{align*}
where, $T^*\coloneqq f'\circ r^*$ and $\cal T$ is some set of non-negative functions containing $T^*$.
\end{theorem}

As a direct result of this, for any set of positive real values functions $\cal G$ containing $r^*$ we have that
\begin{align*}
    r^*&\coloneqq\argmax{r\in\cal G}\E[f'(r(Z^1))]-\E[f^*(f'(r(Z^0))_]
\end{align*}

We can then approximate this using $\rsamp$ to get
\begin{align*}
    \frac{1}{n_1}\sum_{i=1}^{n_1}f'(r(Z^1_i))-\frac{1}{n_0}\sum_{i=1}^{n_0}f^*(f'(r(Z^0_i)))
\end{align*}
Note that while we have theoretical guarantees of consistency in the case of correctly specified $\cal G$, we have much less justification for this approach in the case of incorrectly specified $\cal G$. 
This is simply because from a heuristic perspective, we have no idea if  
\begin{align*}
    \tilde r&\coloneqq\argmax{r\in\cal G}\E[f'(r(Z^1))]-\E[f^*(f'(r(Z^0)))]
\end{align*} is sensible approximation of $r^*$ when $\cal G$ is miss-specified. 
As a result, one could argue that KLIEP is a more principled approach as in the objective we are approximating we choose some notion of the ``closest" $r\in\cal G$ to $r^*$.  

We now go on to expand this process in the case of 2 popular choice of $f$
\subsubsection*{JS-Divergence}
If we take $f(u)\coloneqq u\log(u)-(u + 1)\log\left(\frac{1+u}{2}\right)$ then the associated $f$-divergence is the JS-Divergence. 
The corresponding $f':(0,\infty)\rightarrow \R,~f^*:(\log(2),\infty)\rightarrow\R$ are then given by

\begin{align*}
    f'(t)&\coloneqq \log\left(\frac{2t}{1+t}\right) & f^*(t)&\coloneqq-\log(2-\exp(t))
\end{align*} giving us that
\begin{align*}
    r^*&\coloneqq\argmax{r\in\cal G}\E[f'(r(Z^1))]-\E[f^*(f'(r(Z^0)))]\\
    &=\argmax{r\in\cal G}\E\left[\log\left(\frac{2r(Z^1)}{1+r(Z^1)}\right)\right]+\E\left[\log\left(\frac{2}{1+r(Z^0)}\right)\right]\\
    &=\argmax{r\in\cal G}\E\left[\log\left(\frac{r(Z^1)}{1+r(Z^1)}\right)\right]+\E\left[\log\left(\frac{1}{1+r(Z^0)}\right)\right]\\
\end{align*}
for correctly specified $\cal G$. The form that our approximation for this take is highly similar to the form of the logistic regression approach and is in fact identical when $n_1=n_0$. The only difference being that this approach adjusts for the class imbalance directly in the estimator while the logistic regression approach adjusts for it after the fact.

\subsubsection*{KL-Divergence}
If we take $f(t)\coloneqq t\log(t)$ then the corresponding $f$-divergence is the KL-divergence.
This choice of $f$ has corresponding $f^*:\R\rightarrow\R,~ f':(0,\infty)\rightarrow R$ defined by 
\begin{align*}
    f^*(t)&\coloneqq\exp\{t-1\}& f'(t)&\coloneqq1+\log(t).
\end{align*}

Thus we get that
\begin{align*}
    r^*&\coloneqq\argmax{r\in\cal G}\E[f'(r(X^1))]-\E[f^*(f'(r(X^0)))]\\
    &=\argmax{r\in\cal G}\E[\log(r(X^1))]-\E[r(X^0)].
\end{align*}
Note the objective of this estimator differs from KLIEP as $\E[r(X^0)]$ is not logged.

\subsubsection{Adaptations to MNAR Data} \label{app:Other_DRE_miss}
Adapting this to work with $\rsampm$ simply equates to modifying the objective to be
\begin{align*}
    \frac{1}{n_1}\sum_{i=1}^{n_0}\frac{\one\{X_i^1\neq\NaN\}}{1-\varphi^1(X_i^1)}f'(X_i^1)- \frac{1}{n_0}\sum_{i=1}^{n_0}\frac{\one\{X_i^0\neq\NaN\}}{1-\varphi^0(X_i^0)}f^*(f'(X_i^0)).
\end{align*}

\subsection{Original Neyman-Pearson Classification}\label{app:np_class}
We now describe the Neyman-Pearson classification procedure laid out in \citep{Tong18}. This procedure is describe in Algorithm \ref{alg:np} and constructs a classifier from our data. We assume we have $n_0$ additional copies of $Z^0$ which we label $Z^0_{n_0+1},\dotsc,Z^0_{2n_0}$ however the algorithm adapts to any number of samples from either distribution. We also let $g:\R\rightarrow\R$ be any strictly increasing function.
\begin{algorithm}[ht]
\caption{Neyman-Pearson Classification Procedure} \label{alg:np}
\begin{algorithmic}[1]
\vspace{0.1cm}
\State Use $\{Z_i^1\}_{i=1}^{n_1},\{Z_i^0\}_{i=1}^{n_0}$ to estimate $r^*$ with $\hat r$ by any DRE procedure.
\vspace{0.2cm}
\State Set $i^*=\min\{i\in\{1,\dotsc,n_0\}|\prob(W\geq i)\leq\delta\}$ where $W\sim \text{Binomial}(n_0,1-\alpha)$
\vspace{0.2cm}
\State For $i\in\{1,\dotsc,n_0\}$ compute $\hat r_i\coloneqq g\circ \hat r(Z_{n_0+i}^0)$
\vspace{0.2cm}
\State Sort $\hat r_1,\dotsc,\hat r_{n_0}$ in increasing order to get $\hat r_{(1)},\dotsc,\hat r_{(n_0)}$ with $\hat r_{(i)}\leq \hat r_{(i+1)}$ 
\vspace{0.2cm}
\State Define $\hat C_{\alpha,\delta,g\circ \hat r}\coloneqq \hat r_{(i^*)}$
\vspace{0.2cm}
\State Define the classifier $\hat\phi_{\rsamp}$ by 
\begin{align*}
    \hat\phi_{\rsamp}(z)=\one\{g\circ\hat r(z)>C_{\alpha,\delta,g\circ \hat r}\} \quad\text{ for all $z\in\cal Z$.}
\end{align*}
\end{algorithmic}
\end{algorithm}
Note that the classifier produced by the algorithm does not depend on $g$.

\begin{prop}
Algorithm \ref{alg:np} constructs a Neyman-Pearson classifier satisfying Equation (\ref{eq:type1er2}).
\end{prop}
\begin{proof}
As the algorithm does not depend on $g$ we take $g$ as the identity in the proof.
We adapt the proof given in \citep{Tong18}. Let $\hat r\in\cal H$ and define $A$ to be the event
\begin{align*}
    A\coloneqq\{\prob(\hat\phi_{\rsamp}(Z^0)=+|\rsamp)>\alpha\}.
\end{align*}
We then aim to show that $\prob(A|\hat r)\leq\delta$. To this end let $C^*_{\alpha,\hat r}\coloneqq\inf\{C\in\R|\prob(\hat r(Z^0)>C|\hat r)\leq\alpha\}$. 

Now as $\prob(\hat r(Z^0)>C|\hat r)$ is a right continuous decreasing function we have that $\prob(\hat r(Z^0)>C_{\alpha,\hat r}|\hat r)\leq\alpha$.

Thus we can re-write $A$
\begin{align*}
    A&=\{\hat C < C^*\}\\
    &=\{\hat r_{(i^*)} < C^*\}\\
    &=\{i^* \text{ or more of } \hat \hat r_{1},\dotsc,\hat r_{n_0}\text{ are less than } C^*\}\\
    &=\{\sum_{i=1}^{n_0}B_i\geq i^*\}\quad\text{where }B_i\coloneqq\one\{\hat r_{i}< C^*\}.\\
    \intertext{Hence we have}
    \prob(A|\hat r)&=\prob(\sum_{i=1}^{n_0}B_i\geq i^*|\hat r).
\end{align*}
As $Z_{n_0+1}^0,\dotsc,Z_{2n_0}^0$ are independent, we have that $B_1,\dotsc,B_{n_0}$ are independent given $\hat r$. Furthermore, we have that $q\coloneqq\prob(B_i=1)\leq 1-\alpha$. If we now define $W'\sim\text{Binomial}(n_0,q),~W\sim\text{Binomial}(n_0,1-\alpha)$ we get that.
\begin{align*}
    \prob(A|\hat r)&=\prob(W'\geq i^*)\\
    &\leq\prob(W\geq i^*)\quad\text{as }q\leq 1-\alpha\\
    &\leq\delta \quad\text{by definition of } i^*.
\end{align*}

Finally, as $\prob(A|\hat r)\leq\delta$ for all non-negative measurable $\hat r$ we have that $\prob(A)\leq\delta$.
\end{proof}

\subsection{Logistic Regression Re-adjustment}\label{app:adj_logreg}
Here we describe how to use the adjusted logistic regression described in \cite{king2001} to learn the $\varphi$ under the scenario presented in Section \ref{sec:miss_learn}.

First define $I=\{i\in[n]|z_i \text{ is observed}\}$ so that for $i\in I$ either $X_i\neq\NaN$ or we have queried $Z_i$. We then refer to $\{(W_i,Z_i)\}_{i\in I}$ as our set of fully observed samples. 
These fully observed samples will have disproportionally few $W_i=1$ however $\{Z_i|W_i=1,i\in I\}$, $\{Z_i|W_i=0, i \in I\}$ still give exact samples from $Z|W=1,~Z|W=0$ respectively. 
\cite{king2001} propose a re-adjustment to logistic regression specifically for this case in which the true proportion each response are different to the proportion in your given data. 

Let $\hat \beta_0, \hat \beta_1$ be the logistic regression estimators for the coefficients of the intercept and slope respectively from our fully observed samples. Now let $\hat \tau$ be some consistent estimate of $\prob(W=1)$ and $\hat \tau_I=\frac{\sum_{i\in I}\one\{W_i=1\}}{|I|}$. Now define
\begin{align*}
    \hat \beta_0'=\hat \beta_0-\log\left\{\frac{(1-\hat\tau)\hat\tau_{I}}{\hat\tau(1-\hat\tau_{I}}\right\}
\end{align*}
Then $\hat \beta_0'$ and $\hat \beta_1$ are consistent estimates of the logistic regression parameters. In our case we will simply take $\hat \tau=\frac{\sum_{i\in[n]}w_i=1}{n}$. This allows us to simplify and get
\begin{align*}
    \hat \beta_0'=\hat \beta_0-\log\left\{m/n_1\right\}
\end{align*}
where $n_1=\sum_{i\in[n]}\one\{w_i=1\}$ and $m=\sum_{i\in I}\one\{w_i=1\}$ that is, the number of observations we have queried.

\section{SYNTHETIC EXPERIMENT DETAILS} \label{app:synth_exp}
\subsection{5-dimensional Correctly Specified Case}
For setting we take $Z_1,Z_0$ to have PDFs defined by
\begin{align*}
    p_1(z)=N(z;\mu_1,I)\\
    p_0(z)=N(z;\mu_0,I)
\end{align*} with $\mu_0=0,~\mu_1=(0.1,0.1,0.1,0.1,0.1)^\top$ and where $N(z;\mu,\Sigma)$ is the PDF of a multivariate normal distribution with mean $\mu$ and variance $\Sigma$ evaluated at $z$.
We take $\varphi^1(x)=\half\one\{x^\top\bm a>0\}$, $\varphi^0=0$ where $\bm a=(1,1,1,1,1)^\top$. 

To estimate the density ratio, we use $r_\theta(x)\coloneqq\exp\{\theta^\top x\}$
making the model correctly specified with ''true" parameter $\tilde\theta\coloneqq -\mu_1$. 
for given $n\in\mathbb{N}$, $n$ IID samples are drawn from both $X^0$ and $X^1$. M-KLIEP and CC-KLIEP are then fit using gradient simple gradient descent to obtain parameter estimates. This process is repeated $100$ times for each $n$ and with $n\in\{100,200,300,\dotsc,1500\}$.

\subsection{Mixed Gaussian Neyman-Pearson Case}
For this experiment we take $Z^1$ and $Z^0$ to be 2-dimensional Gaussian mixtures with the following densities
\begin{align*}
    p_1(z)&=\half N\left(z;\begin{pmatrix}0 \\ 0\end{pmatrix},I \right)+
    \half N\left(z;\begin{pmatrix}-1 \\ 4\end{pmatrix},I \right)\\
    p_0(z)&=\half N\left(z;\begin{pmatrix}1 \\ 0\end{pmatrix},I \right)+
    \half N\left(z;\begin{pmatrix}0 \\ 4\end{pmatrix}, I \right)
\end{align*}
Now let $\varphi^1=0.9\one_{Z^\top\bm a>2}(Z)$ where $\bm a=(0,1)^\top$, $\varphi^0=0$ and define $X^1,X^0$ on $\R^2\cup\{\NaN\}$ as before. For given $n\in\mathbb{N}$ we generate $n$ samples from $X^0,X^1$ and use these to estimate $\tilde r$ by both M-KLIEP and CC-KLIEP. An additional $Z^0=X^0$ and these used to produce a classifier via Algorithm \ref{alg:np} with $\alpha=\delta=0.1$. These samples are also used to produce a classifier via Algorithm \ref{alg:np} with $\hat r$ replaced by $r^*$, the true density ratio.
One instance of this procedure with $n=500$ is presented in Figure \ref{fig:np_class_bound}

To estimate the power of the produced classifiers, 1,000,000 samples from $Z^1$ are produced and the proportion classified as $1$ recorded. This process is then repeated for 100 times for each $n\in\{100,200,300,\dotsc,1500\}$. The estimated powers of the classifiers produced from these iterations is then used to estimate the expected power of the procedures alongside 99\% confidence intervals

\section{ADDITIONAL SYNTHETIC EXPERIMENTS} \label{app:synth_exp_add}
\subsection{Naive Bayes Assumption Test}
We wanted to test the effect of our Naive Bayes assumption on our ability to estimate the density ratio. For this experiment we take $Z_0,~Z_1$ to be distributed as follows:
\begin{align*}
    Z_1&\sim N\lrbr{ 
    \begin{pmatrix}  0 \\ 0 \end{pmatrix},~\begin{pmatrix}
        1 & \rho \\ \rho & 1
    \end{pmatrix}}\\
    Z_0&\sim N\lrbr{ 
    \begin{pmatrix}  1 \\ 2 \end{pmatrix},~\begin{pmatrix}
        1 & \rho \\ \rho & 1
    \end{pmatrix}}\\
\end{align*}
for varying $\rho\in[0,1]$. To test our Naive Bayes assumption alongside our M-KLIEP approach we also induced non-uniform missingness separately in the features of the class $0$ data. The missingness functions used were $\varphi^1_1(x)=0.8\cdot\one\lrbrc{x>0}$, $\varphi^1_2(x)=0.8\cdot\one\lrbrc{x<0}$ and then no missingness for class 0 i.e. $\varphi^0_1\equiv\varphi^0_2\equiv 0.$ 

100 samples from $X_0,X_1$ were then generated, a density ratio fit through the Naive Bayes assumption, and then 100 more samples from $X_0$ (equivalent to $Z_0$) were used to produce an NP classifier. Figure \ref{fig:NP_power_naivebayes} shows you the average power of the classifier alongside 95\% C.I.s from 100 montecarlo simulations.

\begin{figure}
    \centering
    \includegraphics{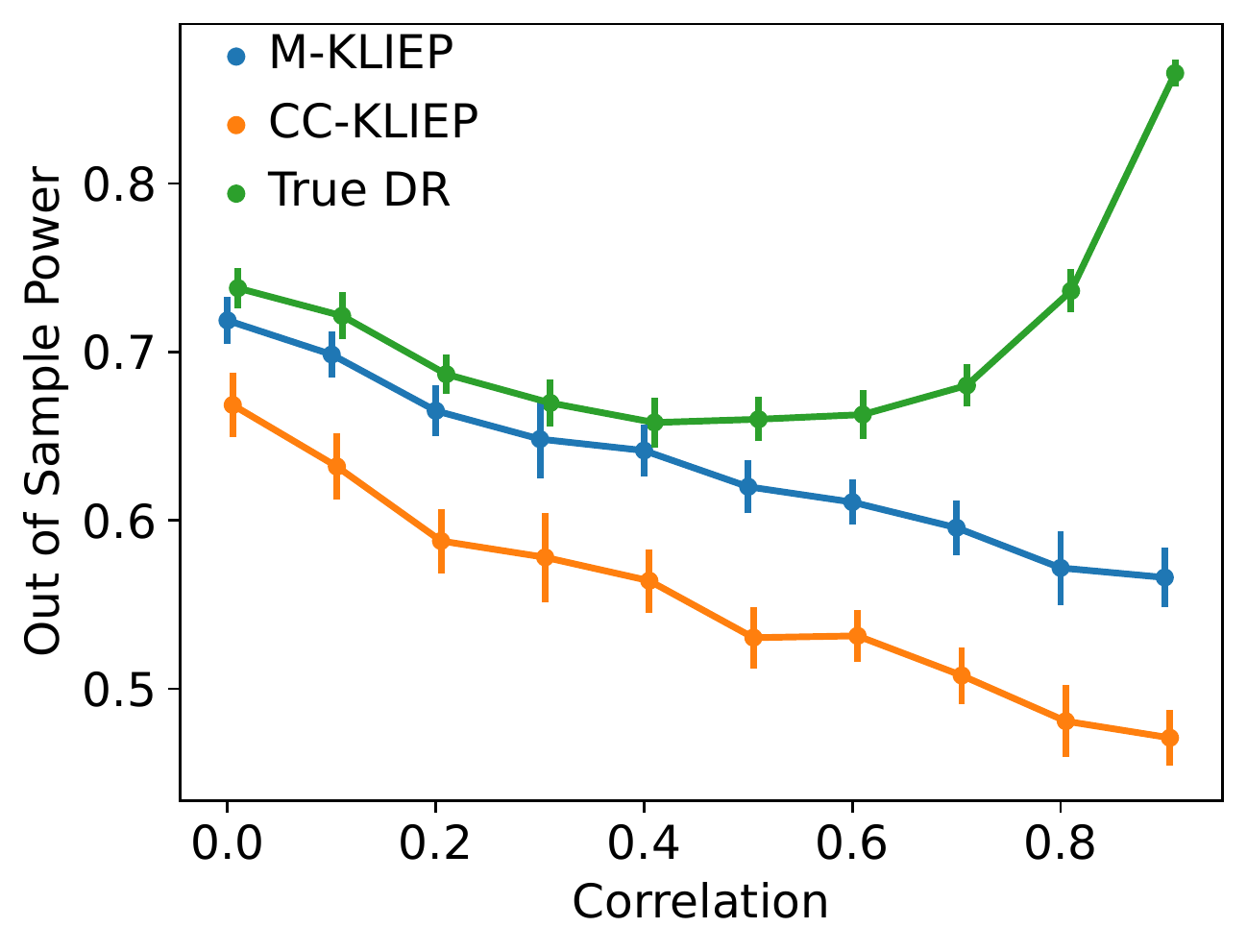}
    \caption{The expected power of the Naive Bayes NP classifier for various levels of correlation between features.}
    \label{fig:NP_power_naivebayes}
\end{figure}

As we can see the Naive Bayes approaches only start seriously deviating from the performance of the true density at around $\rho=0.5$ We also reassuringly see that M-KLIEP consistently outperforms CC-KLIEP.

\subsection{Differing Variance Misspecified Test}
We have tried an additional misspecified case set up as follows
\begin{align*}
    Z_1&\sim N\lrbr{ 
    \begin{pmatrix}  0 \\ 0 \end{pmatrix},~\begin{pmatrix}
        1 & 0 \\ 0 & 1
    \end{pmatrix}}\\
    Z_0&\sim N\lrbr{ 
    \begin{pmatrix}  1 \\ 1 \end{pmatrix},~\begin{pmatrix}
        1 & 0 \\ 0 & 2
    \end{pmatrix}}\\
\end{align*}
We then introduced complete missingness in class 1 with $\varphi^1(x)=0.8\cdot\one\lrbrc{x_1>0}$ and no missingness in class 0 ($\varphi^0\equiv$.) 
We then used the log-linear form for our density ratio estimate with $f(x)=x$ which leads to an incorrectly specified model ($f(x)=(x^\top,x^{2^\top})^\top$ would lead to a correctly specified model.) 

For various values of $n$, we then generate $n$ samples from $X_0,X_1$ and a density ratio was then fit using M-KLIEP and CC-KLIEP. An additional $n$ samples from $X_0$ (equivalent to $Z_0$) were used to produce an NP classifier. 

\begin{figure}
    \centering
    \includegraphics{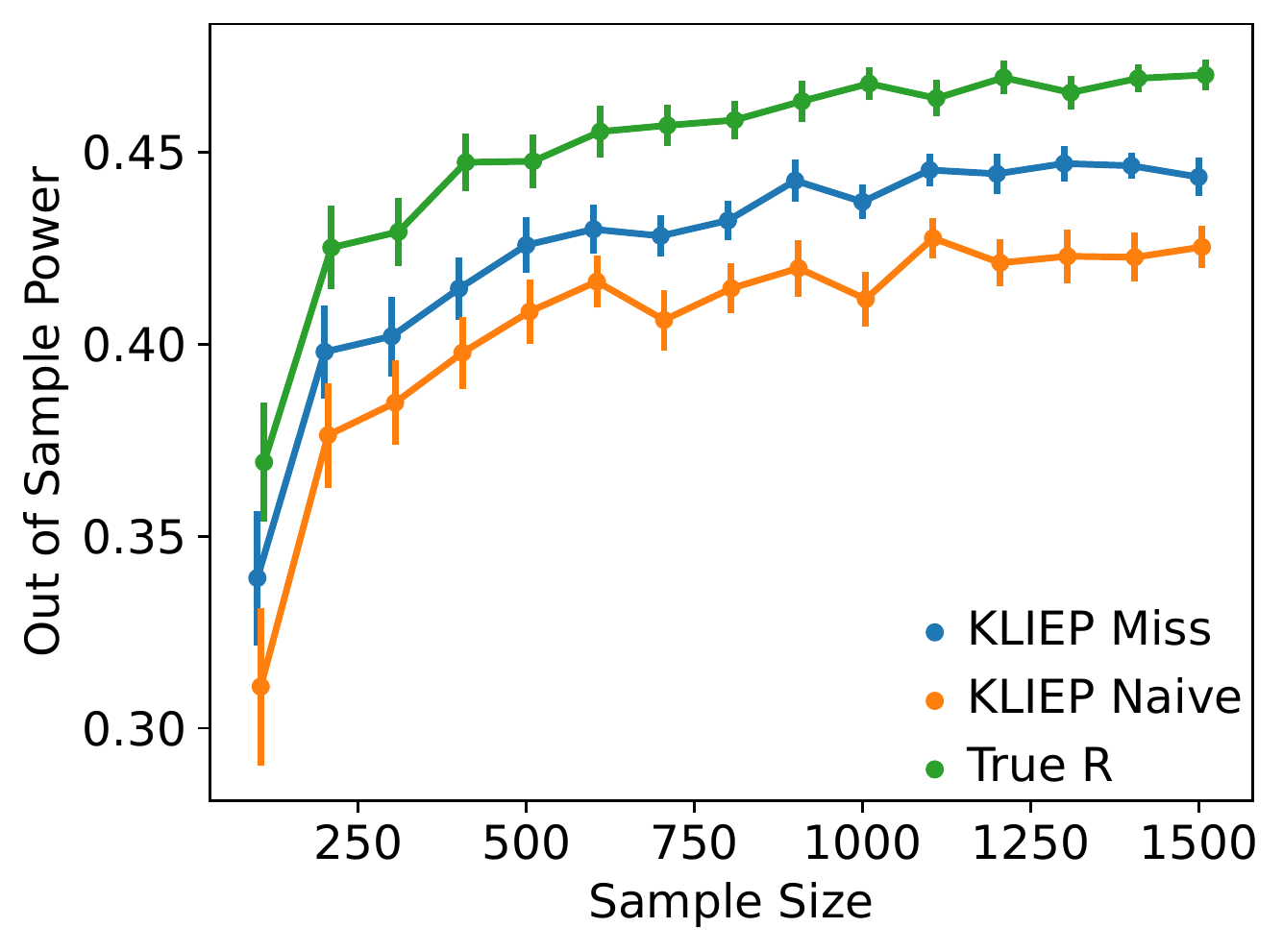}
    \caption{Expected power of NP classifier when used in conjunction with multiple DRE approaches for varying sample size. In this case our model is misspecified.}
    \label{fig:NP_power_diffvar}
\end{figure}

Figure \ref{fig:NP_power_diffvar} shows one simulation in this case alongside the classification boundaries produced. We can see that with this misspecified parametric boundary our model will always approximate the true boundary relatively crudely however we can clearly see that the M-KLIEP boundary is a more sensible approximation than the CC-KLIEP boundary

\begin{figure}
    \centering
    \includegraphics[width=\textwidth]{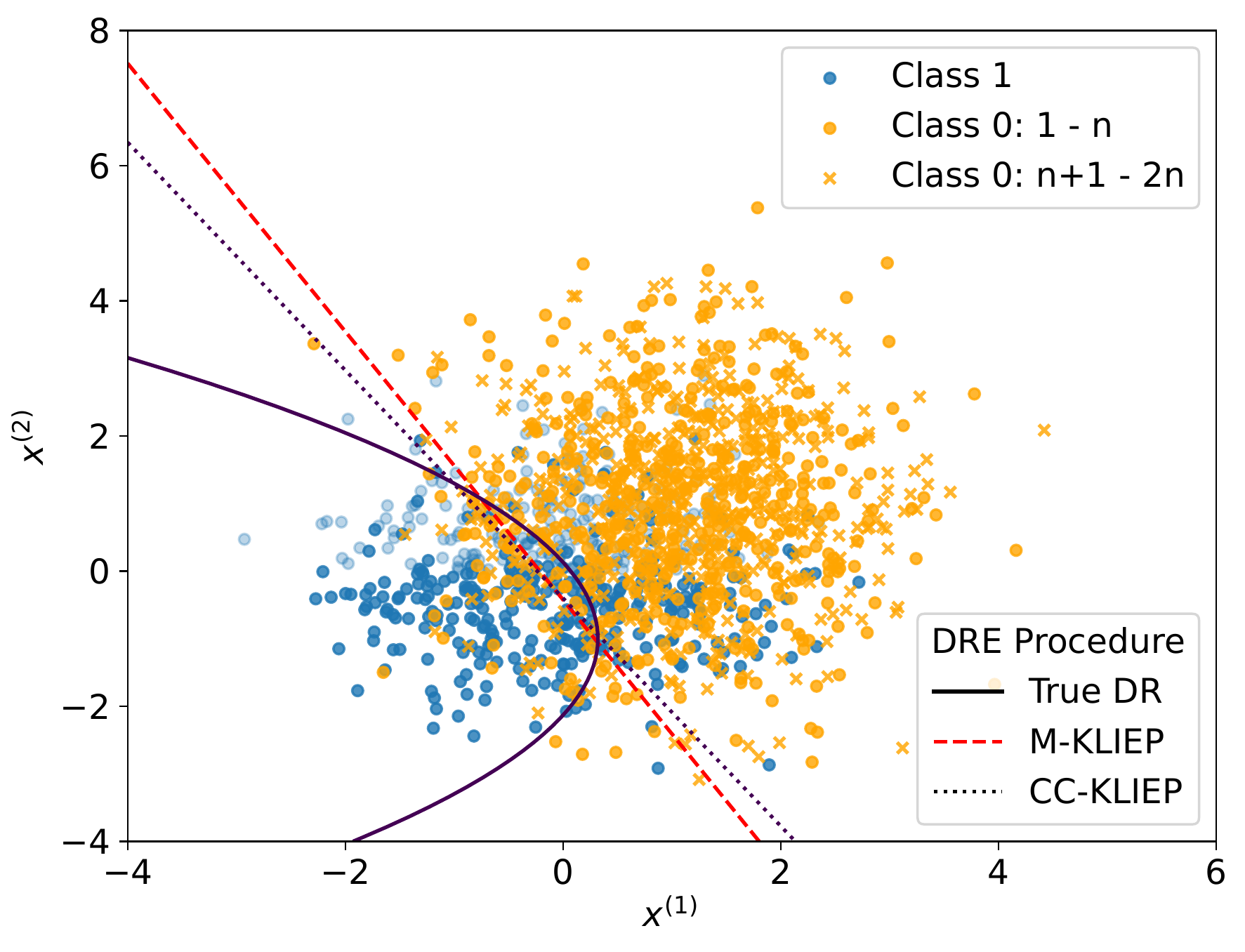}
    \caption{Scatter plot of $\rsamp$ alongside classification boundaries produced from corresponding $\rsampm$ via various procedures. If $X_i=\NaN$ then $Z_i$ is faded out.}
    \label{fig:NP_sample_diffvar}
\end{figure}

Figure \ref{fig:NP_sample_diffvar} shows you the estimated expected power of the classifiers produced alongside 95\% C.I.s from 100 montecarlo simulations (where 1,000,000 additional samples from $Z_1$ were used to estimate the power of the classifier in each simulation.) This definitively shows M-KLIEP has better performance than CC-KLIEP. While M-KLIEP does perform worse than the true classification boundary this decrease is not drastic. 

\subsection{Varying Misspecificaiton Level}
In this experiment we wanted to examine the effect of the level of misspecification on our model. To do this we let $Z_1, Z_0$ have pdfs defined as follows
\begin{align*}
    p_1(z)&=(1-\rho) N\left(z;\begin{pmatrix}0 \\ 0\end{pmatrix},I \right)+
    \rho N\left(z;\begin{pmatrix}2 \\ 0\end{pmatrix},I \right)\\
    p_0(z)&=\half N\left(z;\begin{pmatrix}1 \\ 0\end{pmatrix},I \right)
\end{align*}
for various $\rho\in[0,0.5]$.
We then introduced complete missingness in class 1 with $\varphi^1(x)=0.8\cdot\one\lrbrc{x_1>0}$ and no missingness in class 0 ($\varphi^0\equiv$.) 
We then used the log-linear form for our density ratio estimate with $f(x)=x$ which leads to an incorrectly specified model for $\rho>0$ with larger values of $\rho$ leading to greater levels of misspecification.

We then generate 100 samples from $X_0,X_1$ and a density ratio was fit using M-KLIEP and CC-KLIEP. An additional 100 samples from $X_0$ (equivalent to $Z_0$) were used to produce an NP classifier. Figure \ref{fig:NP_power_varymisspec} shows you the average power of the classifier alongside 95\% C.I.s from 100 montecarlo simulations (where 1,000,000 additional samples from $Z_1$ were used to estimate the power of the classifier in each simulation.)

\begin{figure}
    \centering
    \includegraphics{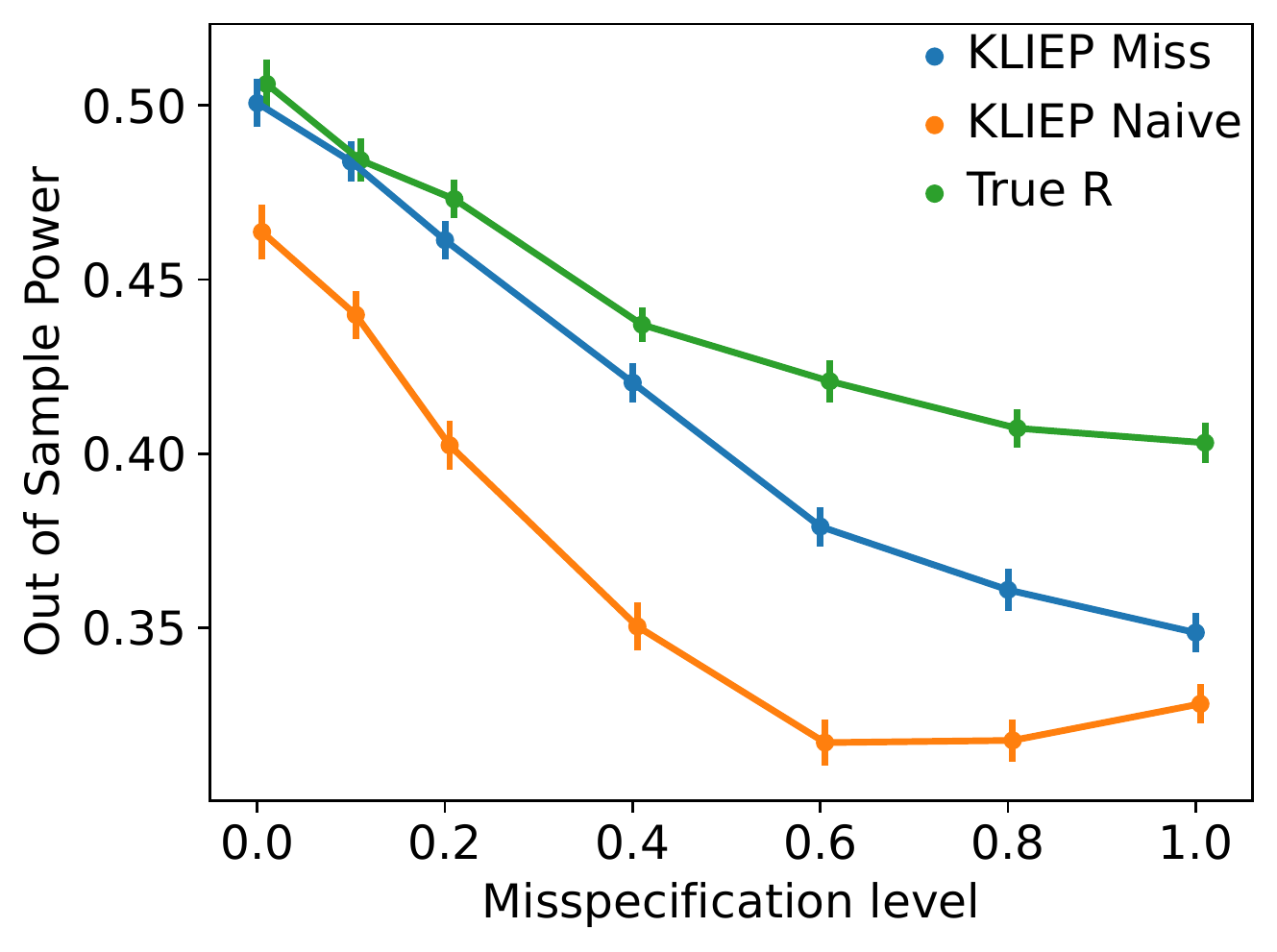}
    \caption{Expected power of NP classifier produced via various DRE techniques for varying levels of misspecification. We take misspecification to be $2*\rho$.}
    \label{fig:NP_power_varymisspec}
\end{figure}

\section{REAL WORLD EXPERIMENT DETAILS}\label{app:real_exp}
Here we give additional information on the experiments performed on the real world data.
We start off by describing each of the three datasets.
\subsection{Datasets}\label{app:real_data}
\subsubsection{CTG}
Foetal Cardiotocograms (CTGs) measure the babies heart function during labour. These are then used by doctors and nurses and then used as a diagnostic tool to assess the foetus' health. In our data-set we have numerical summaries of foetal CTGs which alongside the diagnoses of one of ``normal", ``suspect", or ``pathologic" that doctors ascribed to associated foetuses. We aim to create an automated diagnosis procedure which will determine from the summaries whether the foetus is ``normal" (class $1$) and therefore needs no additional attention/ intervention, or one of ``suspect" or ``pathologic" (Class $0$) to and hence needing further follow up. It is clear that NP classification as a well suited classification procedure to this paradigm as we would like to strictly control the probability of miss-classifying a ``suspect" or ``pathologic" baby as ``normal".

The data is taken from: https://archive.ics.uci.edu/ml/datasets/cardiotocography and was first presented in \cite{AyresdeCampos2000}.

This data contains observations of 2129 foetuses, 1655 from class $1$ and $474$ from class $0$. From the data we select 10 features which we describe below

We split the data as follows, $237$ \& $1555$ samples from Class $0$ \& $1$ to fit $\hat r$, $237$ samples from Class $0$ are used to fit the threshold of the classifier, and $100$ samples from Class $1$ are used to estimate the power of the classifier. When $\varphi$ are learned we query $10$ missing samples from each feature.

\subsubsection{Fire}
For multiple different fires and non-fires, various different environmental readings were taken such as Temperature, Humidity and CO2 concentration. The aim of this is to be able to classify whether or not a fire is present to create a sort of IOT (internet of things) smoke detector to detect the presence of fire. Again NP classification is clearly a good fit as falsely detecting a fire is far less damaging than missing a fire. As such, we take the presence of fire as our error controlled class (Class $0$) as we want to strictly control the probability of not detecting a fire which is present. This data is taken from: https://www.kaggle.com/datasets/deepcontractor/smoke-detection-dataset.

This dataset consists of 62,630 observations of 12 features. The observations consist of 17,873 observations from Class $1$ and 44,757 from Class $0$. 
Before carrying out any of our procedures we perform minimal feature manipulation by trimming extreme values for some of the features as described below:
\begin{itemize}
    \item \textbf{TVOC(ppb), eCO2(ppm)}: Values are trimmed to not exceed $1000$.
    \item \textbf{PM1.0, PM2.5}: Values are trimmed to not exceed $1$.
    \item \textbf{NC0.5, NC2.5}: Values are trimmed to not exceed $5$.
\end{itemize}
\subsubsection{Weather}
A dataset containing various weather reading for different days in Australia collected with the aim of using the previous days weather to predict whether there is a chance on the following day. We want to be able to choose what we mean by a ``chance" in terms of the probability of our classifier correctly predicting rain when it is present therefore motivating the use of NP classification. As such we take the event of rain the following day to be our error controlled class (Class $0$.) This data is taken from: https://www.kaggle.com/datasets/jsphyg/weather-dataset-rattle-package.

This dataset consists of 142,193 observations of 20 features. The observations consist of 110,316 observations from Class $1$ and 31,877 from Class $0$.

Before carrying out any of our procedures we perform minimally feature manipulation by trimming extreme values for some of the features. Factor variables are split into indicator variables for each of the possible outcomes, with one outcome having no indicator to avoid redundancy. Continuous features are trimmed to be within 5 IQR (inter-quartile range) of the median. This data manipulation lead to a total of $62$ dimensions in our final data.

\subsection{Experimental Designs}\label{app:real_exp_design}
For each of the datasets, the same overall experimental design is used. Firstly data is randomly split into train (Class $0$ \& $1$)/calibrate (Class $0$)/test (Class $1$) datasets of sizes given below. Any missing values originally in the data set are then imputed using simple mean imputation. Next the data is normalised to reduce the risk of numerical precision issues within our algorithms. The Class $0$ training data is then corrupted along each dimension using $\varphi_j$ of the form $\varphi_j(z)\coloneqq(1+\exp\{\tau_j(a_{0,j}+a_{1,j}z)\})^{-1}$. This $\varphi_j$ is then estimated by querying points as described in Section \ref{sec:miss_learn}. The density ratio is then fit by M-KLIEP, CC-KLIEP on the training data. Alongside this, we also fit the density ratio using M-KLIEP with the true $\varphi_j$ and KLIEP with the original non-corrupted training data as benchmarks. A classifier is then constructed from these estimated density ratios using the calibration data by the procedure described in Algorithm \ref{alg:np} with given $\alpha,\delta$. The testing data is then used to estimate the power of these classifiers.

Multiple iterations of this process are performed with random train/calibrate/test splits (of fixed size) and random $\tau_j\in\{-1,1\}$ for each iteration. The estimated powers from these iterations are then used to calculate the pseudo expected power and 95\% C.I.s

For experiment 1, we vary $\alpha\in\{0.05,0.1,0.15,0.2,0.25,0.3\}$ with $\delta=0.05$, $a_{j,0}=-\mu_j/\sigma_j,~a_{j,1}=1/\sigma_j$ where $\mu_j$, and $\sigma_j$ are the sample mean and variance of feature $j$.

For experiment 2 this process is repeated with varying $a_{j,1}$ with $\alpha=0.1,\delta=0.05, a_{j,1}=\sigma_j$. The $a_{j,0}$ are chosen to produce the following proportions of missing values for each feature: $\{0.1,0.2,0.3,\dotsc,0.9\}$.

Dataset specific experimental design information can be found below
\subsubsection{Dataset specific experimental design}
\paragraph{CTG}
The size of the data splits are as follows, Class 0 train = $237$, Class 1 train = $1555$, Class $0$ calibrate = $237$, and Class $1$ test $100$.We query $10$ missing samples from each feature when learning $\varphi_j$. For both experiments 1,000 iterations are run.

\paragraph{Fire}
The sample sizes of the data splits are as follows, Class 0 train = $20,000$, Class 1 train = $12,873$, Class $0$ calibrate = $24,757$, and Class $1$ test = $5,000$. We query $50$ missing samples from each feature when learning $\varphi_j$.
For both experiments, 100 iterations are run.

\paragraph{Weather}
The size of the data splits are as follows, Class 0 train = $15,000$, Class 1 train = $100,316$, Class $0$ calibrate = $16,877$, and Class $1$ test = $10,000$. We query $50$ missing samples from each feature when learning $\varphi_j$.
For experiment 1, 100 iterations are run; for experiment 2, 30 iterations are run.

\subsection{Additional Missing Data Approaches}
We also compared our approached to an iterative imputation approach implemented using the \verb|sklearn.impute.IterativeImputer| module in Python using the default estimator \verb|BayesianRidge()|. For fair comparison, this imputed data was then used to perform Naive Bayes DRE to make it comparable to our approach. Results of these additional experiments are given in figures \ref{fig:np_RWE_vary_alpha_mice} \& \ref{fig:np_RWE_varymiss_mice} and described below.

\subsubsection{Additional Results}\label{app:real_exp_add}
Here we present the same results as the paper but with the iterative imputation method included. Figure \ref{fig:np_RWE_vary_alpha_mice} shows the out of sample power of various NP classifiers for varying Type I error we can see that the additional iterative imputation method seems to perform comparably or worse than our method for each data set

Figure \ref{fig:np_RWE_varymiss_mice} shows the out of sample power of various NP classifiers for varying levels of missingness. Here we see that as the level of missingness increases, the iterative imputation method degrades significantly leading it to perform significantly worse than our approach. It does however consistently perform better than the complete case approach

\begin{figure*}
     \centering
     \begin{subfigure}{0.3\textwidth}
        \centering
        \includegraphics[width=1\textwidth]{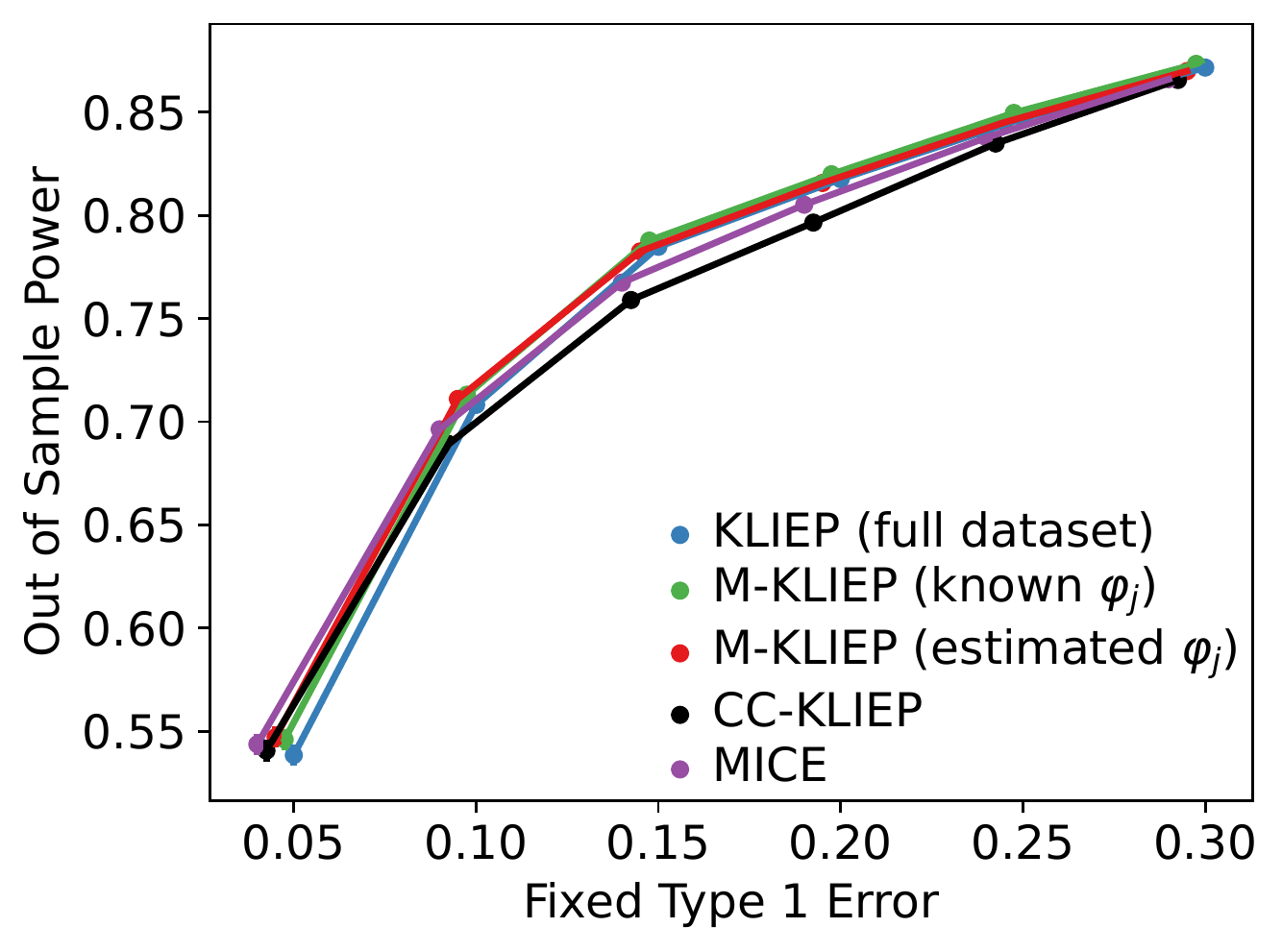}
        \caption{CTG Dataset.}
        \label{fig:np_RWE_CTG_varyalpha_mice}
     \end{subfigure}
     \hfill
        \begin{subfigure}{0.3\textwidth}
        \centering
        \includegraphics[width=1\textwidth]{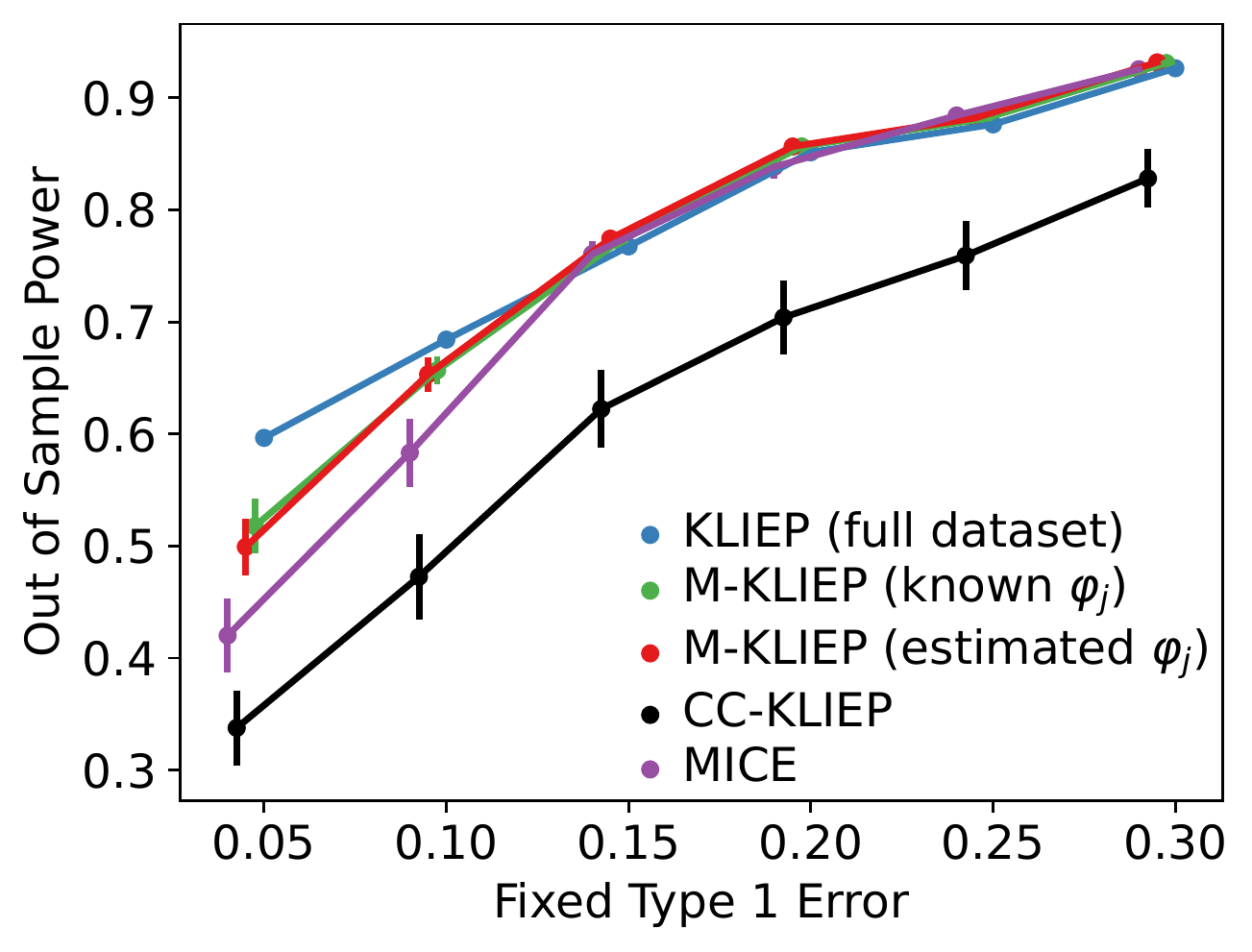}
        \caption{Fire Dataset.}
        \label{fig:np_RWE_smoke_varyalpha_mice}
     \end{subfigure}
     \hfill
     \begin{subfigure}{0.3\textwidth}
        \centering
        \includegraphics[width=1\textwidth]{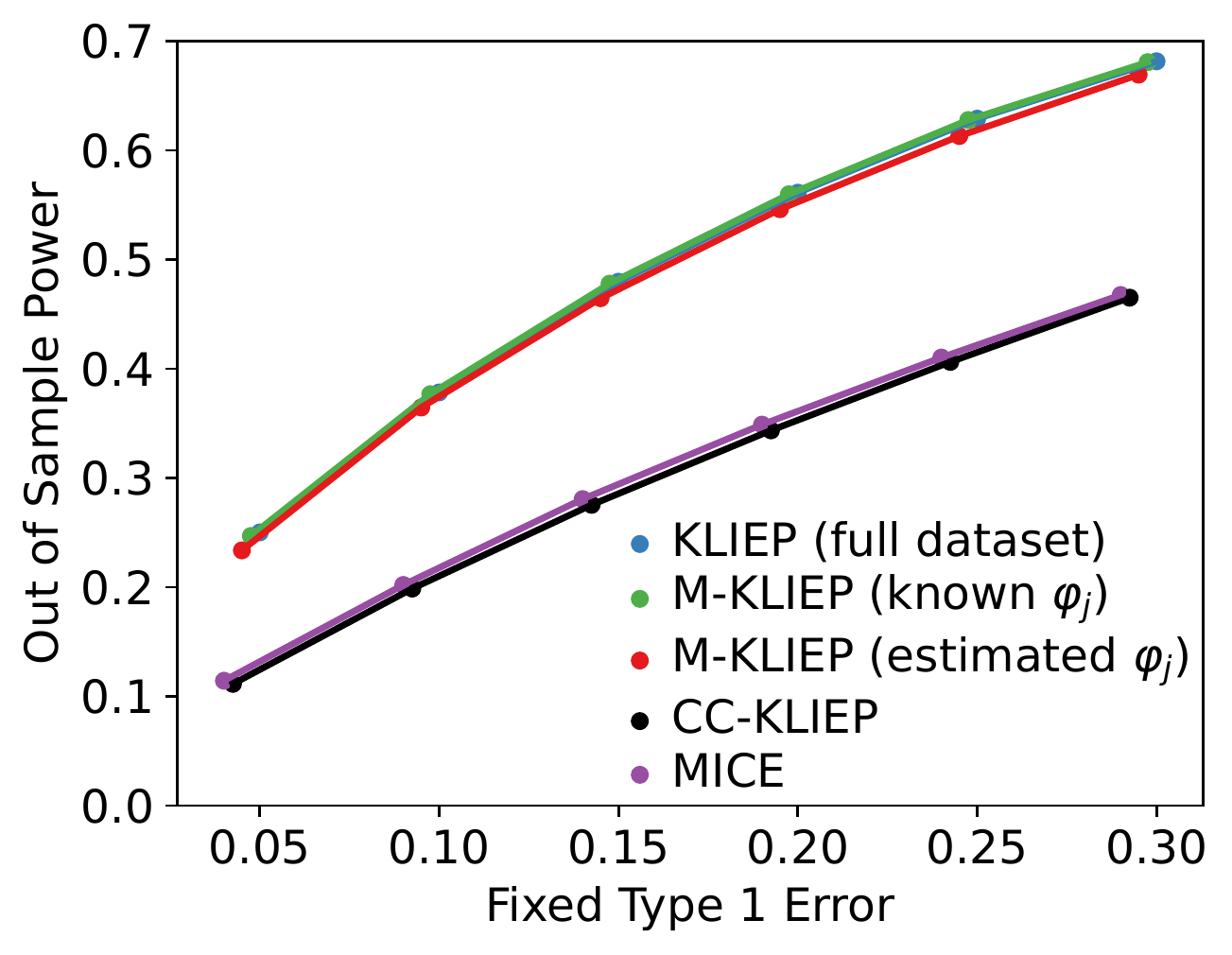}
        \caption{Weather Dataset.}
        \label{fig:np_RWE_weather_varyalpha_mice}
     \end{subfigure}
     \caption{Out of sample power with pseudo 95\% C.I.s for various different target Type I errors with iterative imputation approach included.}
     \label{fig:np_RWE_vary_alpha_mice}
    \hfill
\end{figure*}

\begin{figure*}[t]
     \centering
     \begin{subfigure}[b]{0.3\textwidth}
        \centering
        \includegraphics[width=\textwidth]{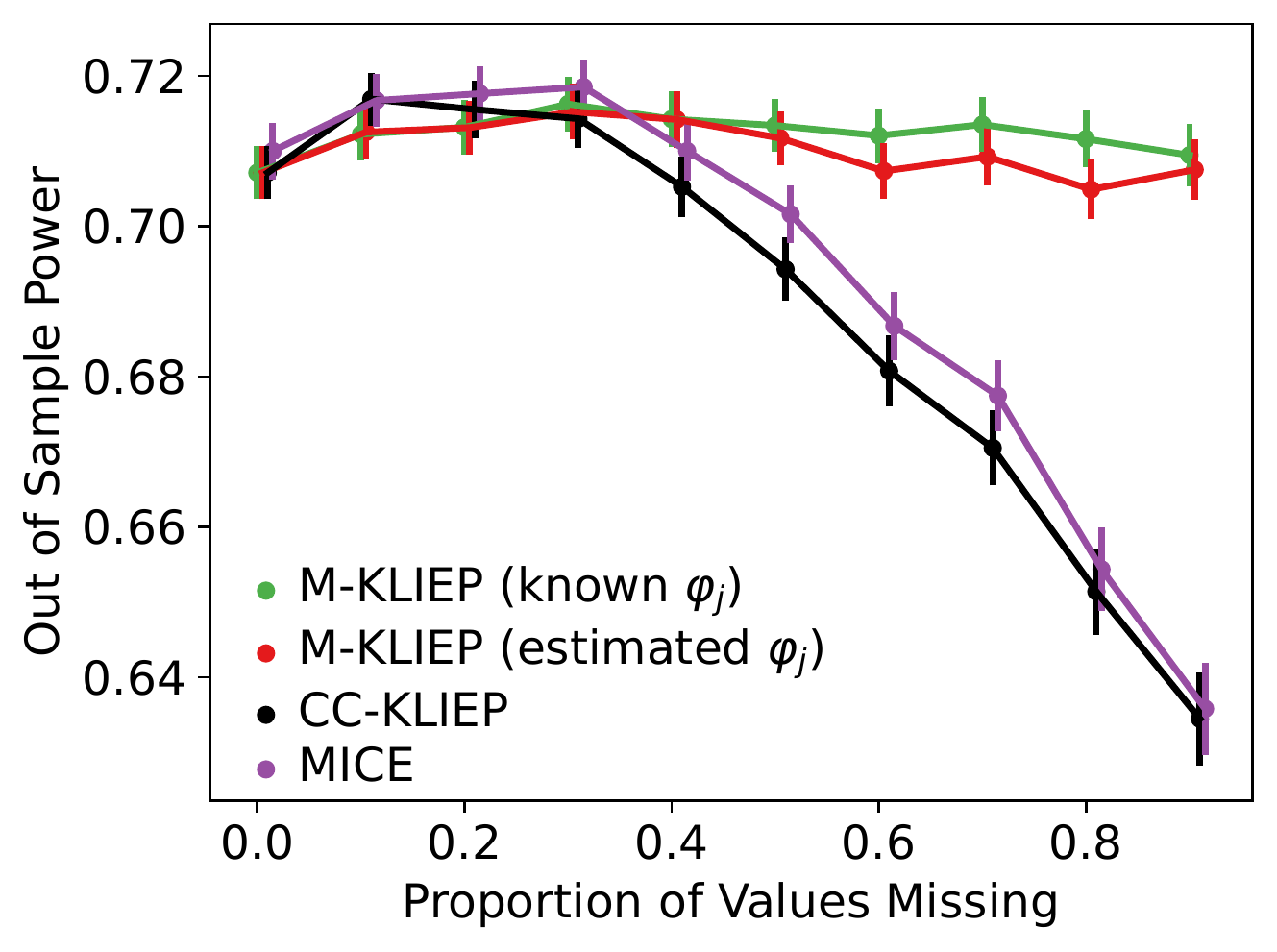}
        \caption{CTG Dataset.}
        \label{fig:np_RWE_CTG_varymiss_mice}
     \end{subfigure}
     \hfill
     \begin{subfigure}[b]{0.3\textwidth}
        \centering
        \includegraphics[width=\textwidth]{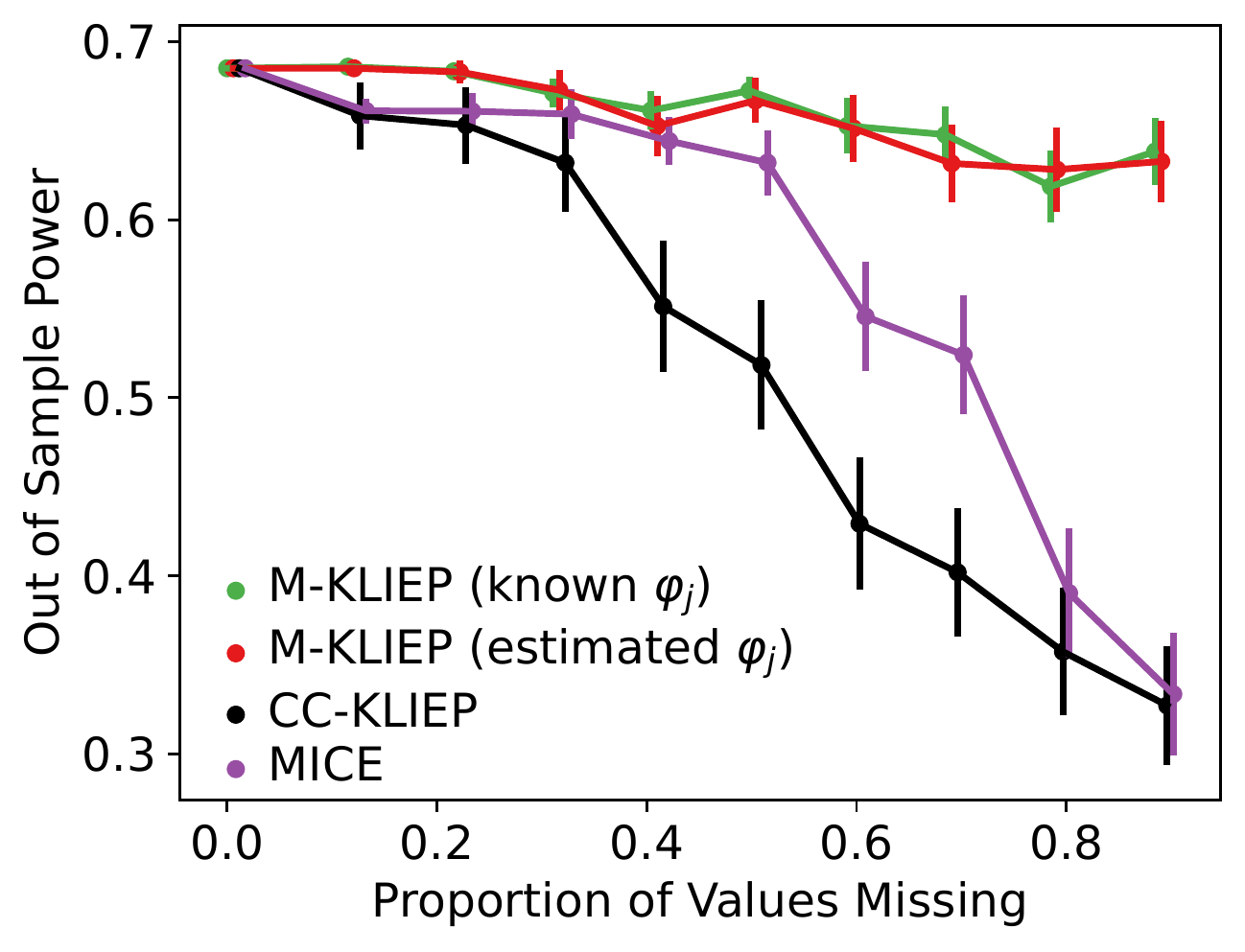}
        \caption{Fire Dataset.}
        \label{fig:np_RWE_smoke_varymiss_mice}
     \end{subfigure}
     \hfill
      \begin{subfigure}[b]{0.3\textwidth}
        \centering
        \includegraphics[width=1\textwidth]{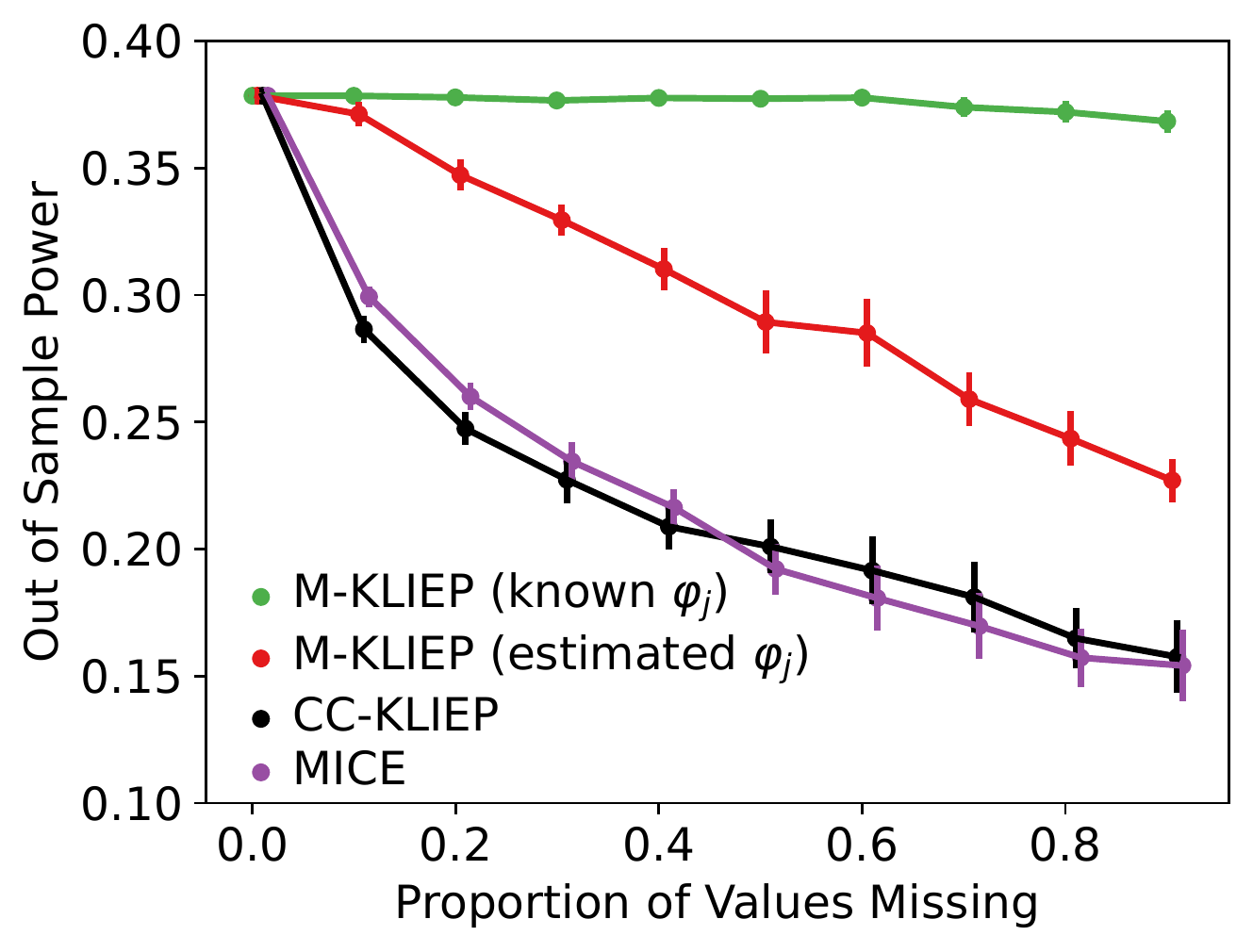}
        \caption{Weather Dataset}
        \label{fig:np_RWE_weather_varymiss_mice}
     \end{subfigure}
     \caption{Out of sample power with pseudo 95\% C.I.s for various $\varphi$ and varying missing proportions with iterative imputation approach included.}
     \label{fig:np_RWE_varymiss_mice}
\end{figure*}

\section{SOCIETAL IMPACT}\label{app:soc_imp}
We now briefly discuss the societal impact of our work. As we have shown, ignoring MNAR data can lead to degradation in performance. Perhaps worse than this however is that, without taking account of our MNAR structure, use of the original NP classification procedure won't guarantee our desired Type I error with high probability. This could lead to classifiers which perform far worse on our error controlled class than estimated. The impact of this could be serious in the case of say disease detection where we believe our classifier to be detecting a far higher proportion of diseased individuals than it truly is.
The impact of MNAR data on analysis has been explored before in more general settings \citep{Rutkowski2011,Padgett2014,Goldberg2021} with the key takeaway always being that ignoring the structure of the missingness can negatively impact your results and lead to false inference.
\end{document}